\documentclass[12pt]{article}
\usepackage{amsmath,amssymb,mathtools,algpseudocode,algorithm,indentfirst,amsthm}
\usepackage[x11names,table]{xcolor}
\usepackage{etoolbox,graphicx}
\usepackage[toc,page]{appendix}
\usepackage{subcaption}
\usepackage[english]{babel}
\usepackage{verbatim}
\usepackage[margin=3cm]{geometry}
\usepackage{url}
\usepackage{bbm}
\usepackage{algorithm}
\usepackage{algorithmicx}
\usepackage{accents}
\usepackage{ulem}
\usepackage{hyperref}
\usepackage{authblk}

\DeclareMathOperator*{\argmin}{argmin}

\newcommand{\at}[2][]{#1\bigg|_{#2}}
\newcommand{\dprod}[2]{\left\langle #1,#2\right\rangle}
\newcommand{\norm}[1]{\left\lVert #1\right\rVert}

\newcommand{\pDeriv}[2]{\frac{\partial #1}{\partial #2}}
\newcommand{\expPow}[2]{e^{-\lVert #1 - #2 \rVert^2/\epsilon}}
\newcommand{\expect}[1]{\mathbb{E}\left[#1\right]}
\newcommand{\dbtilde}[1]{\accentset{\approx}{#1}}
\newcommand{\re}[1]{\text{Re}\left\{#1\right\}}
\newcommand{\im}[1]{\text{Im}\left\{#1\right\}}

\algnewcommand{\Initialize}[1]{
  \State \textbf{Initialize:}\:#1
}

\newtheorem{theorem}{Theorem}

\newtheorem{lemma}[theorem]{Lemma}

\theoremstyle{definition}
\newtheorem{remark}{Remark}
\newtheorem{definition}[theorem]{Definition}
\numberwithin{equation}{section}
\AfterEndEnvironment{theorem}{\noindent\ignorespaces}
\AfterEndEnvironment{conjecture}{\noindent\ignorespaces}
\AfterEndEnvironment{definition}{\noindent\ignorespaces}

\newcommand{\man}{
	\mathcal{M}
}
\newcommand{\subman}{
	\mathcal{N}
}
\newcommand{\Hspace}{
	\mathcal{H}
}
\newcommand{\I}{
	\mathcal{I}
}

\newcommand{\CD}{
	\mathbb{C}^\mathcal{D}
}

\newcommand{\U}{
	\text{U}
}
\begin{document}

\title{The $G$-invariant graph Laplacian}

\author[1]{Eitan Rosen}
\author[2]{Paulina Hoyos}
\author[3]{Xiuyuan Cheng}
\author[2]{Joe Kileel}
\author[1]{Yoel Shkolnisky}

\affil[1]{Department of Applied Mathematics, Tel Aviv University}
\affil[2]{Department of Mathematics, The University of Texas at Austin}
\affil[3]{Department of Mathematics, Duke University}

%\author{Eitan Rosen, Paulina Hoyos, Xiuyuan Cheng, \\ Joe Kileel, Yoel Shkolnisky} %\corref{ER1}
%\affiliation[ER]{organization={Department of Applied Mathematics, Tel-Aviv University},
%	city={Tel-Aviv},
%	country={Israel}}

%\author{Paulina Hoyos}
%\affiliation[PH]{organization={Department of Mathematics, The University of Texas at Austin},
%	city={Austin},
%	state={Texas},
%	country = {USA}}

%\author{Xiuyuan Cheng}
%\affiliation[XC]{organization={Department of Mathematics, Duke University},
%city = {Durham},
%state = {North Carolina},
%country = {USA}}

%\author{Joe Kileel}
%\affiliation[JK]{organization={Department of Mathematics, The University of Texas at Austin},
%	city={Austin},
%	state={Texas},
%	country =  {USA}}

%\author{Yoel Shkolnisky}
%\affiliation[YS]{organization={Department of Applied Mathematics, Tel-Aviv University},
%	city={Tel-Aviv},
%	country={Israel}}

%\begin{keyword}
%	manifold learning\sep group invariance\sep graph Laplacian\sep manifold denoising
%\end{keyword}

%\cortext[ER1]{Please address manuscript correspondence to Eitan Rosen, eitanrose@gmail.com, (972) 585485141, School of Mathematical Sciences, Tel Aviv University, P.O. box 39040, Ramat-Aviv, Tel-Aviv 6997801, Israel.}
\date{}

%\journal{Applied and Computational Harmonic Analysis}
\maketitle
\begin{abstract}
     Graph Laplacian based algorithms for data lying on a manifold have been proven effective for tasks such as dimensionality reduction, clustering, and denoising. In this work, we consider data sets whose data points lie on a manifold that is closed under the action of a known unitary matrix Lie group~$G$. We propose to construct the graph Laplacian by incorporating the distances between all the pairs of points generated by the action of~$G$ on the data set. We deem the latter construction the~``$G$-invariant Graph Laplacian'' ($G$-GL). 
	 We show that the $G$-GL converges to the Laplace-Beltrami operator on the data manifold, while enjoying a significantly improved convergence rate compared to the standard graph Laplacian which only utilizes the distances between the points in the given data set. Furthermore, we show that the $G$-GL admits a set of eigenfunctions that have the form of certain products between the group elements and eigenvectors of certain matrices, which can be estimated from the data efficiently using FFT-type algorithms. We demonstrate our construction and its advantages on the problem of filtering data on a noisy manifold closed under the action of the special unitary group~$SU(2)$.
\end{abstract}

\section{Introduction}\label{secIntro}
A popular modeling assumption in data analysis is that the observed data lie on a low dimensional manifold~$\man$  that is embedded in high dimensional Euclidean space. When~$\man$ is a linear subspace, it can be identified by using principal component analysis (PCA). However, most often~$\man$ is non linear. A leading approach for analyzing data with a nonlinear manifold structure is to encode the data by using a graph, whose vertices are the data points, and whose edge weights encode the similarities between pairs of points. These similarities can be used to form a matrix known as the graph Laplacian, and its eigenvectors and eigenvalues are used for tasks such as dimensionality reduction, clustering, and denoising (\cite{diffMaps,lapMap,diffMapsSigProc}). While the term graph Laplacian has been given different definitions in different contexts \cite{chung,godsil}, in this paper we adopt the definition and notation in~\cite{belkin2003laplacian} and~\cite{diffMaps}.

Formally, let $\left\{ x_1,\ldots,x_N \right\}$ be a set of points that reside on a compact and smooth~$d$-dimensional manifold~$\man$ embedded in~$\mathbb{C}^n$. We form the matrix~$W\in  \mathbb{R}^{N\times N}$ with~$W_{ij} = K(x_i,x_j)$, where~$K$ is a positive semi-definite kernel function. 
The graph Laplacian is then defined as the matrix $L\in \mathbb{R}^{N\times N}$ given by
\begin{equation}\label{intro:classicalGLDef}
	L = D-W,\quad D_{ii} = \sum_{j=1}^{N}W_{ij},
\end{equation}
where $D$ is the $N\times N$ diagonal matrix with the $i$-th' element on the diagonal given by~$D_{ii}$ in~\eqref{intro:classicalGLDef}. Various choices for the kernel~$K$ have been utilized in the literature \cite{belkin2003laplacian,singer2012vector}. In this work, we make the popular choice of~$K$ being the Gaussian kernel function, due to its favorable analytical properties. In this case,
\begin{equation}\label{intro:gaussKerDef}
	W_{ij} = K(x_i,x_j) = e^{-\norm{x_i-x_j}^2/\epsilon}, \quad i,j\in\left\{1,\ldots,N\right\},
\end{equation}
where~$\epsilon$ is a bandwidth to be determined from the data. 

A particularly important matrix related to~$L$ of~\eqref{intro:classicalGLDef} is the random-walk normalized graph Laplacian, defined as
\begin{equation}\label{intro:classicalNormGLDef}
	\tilde{L}= D^{-1}L=I-P, \quad P = D^{-1}W.
\end{equation}
The matrix~$P$ is row-stochastic, and thus, may be viewed as the transition probability matrix of a random walk over the data points (which gives~$\tilde{L}$ its name).  
The latter view was adopted in the seminal work~\cite{diffMaps}, where the eigenvectors of~$P$ (which are identical to those of~$\tilde{L}$) and its eigenvalues are used to construct ``diffusion maps'', a successful machine learning framework for dimensionality reduction and clustering of manifold data. 
Furthermore, in~\cite{lapMapConv} it was shown that if the data points are sampled uniformly from a manifold~$\man$, then~$\tilde{L}$ converges to the Laplace-Beltrami operator~~$\Delta_\man$ on $\man$ as $\epsilon\rightarrow 0$ and $N\rightarrow \infty$.  Formally, it was shown that for a sufficiently smooth~$f$, with high probability
\begin{equation}\label{intro:classicalConv}
	\frac{4}{\epsilon}\sum_{j=1}^N \tilde{L}_{ij}f(x_j) = \Delta_\man f (x_i)+O\left(\frac{1}{N^{1/2}\epsilon^{1/2+d/4}}\right)+O(\epsilon). 
\end{equation}
The latter result has important theoretical and practical consequences. First, we observe that the convergence rate of the graph Laplacian to~$\Delta_\man$ depends on the intrinsic dimension $d$ of $\man$ and not on the ambient high dimension~$n$ of the data points, mitigating the ``curse of dimensionality''~\cite{Chen2009Curse}. Second, it is known that the eigenfunctions of~$\Delta_\man$ provide a basis for the space~$L^2(\man)$ of square-integrable functions on~$\man$ \cite{rosenberg}. For example, when~$\man$ is the circle~$S^1$, the eigenfunctions of~$\Delta_{S^1}$ are given by the Fourier modes~$\left\{ e^{im\theta}\right\}$. Recent results~\cite{CHENG2022132} show that the eigenvectors of~$\tilde{L}$ converge to the eigenfunctions of~$\Delta_\man$. This implies that the eigenvectors of the graph Laplacian constructed from a data set sampled from~$S^1$ are discrete approximations to the Fourier modes, giving rise to classical discrete Fourier analysis. 
Analogously, the eigenvectors of the graph Laplacian constructed by using a sample from a general compact manifold~$\man$ can be employed for a data-driven discrete Fourier analysis on~$\man$ \cite{diffMapsSigProc}.

In various scenarios, the data set under consideration is closed under the action of a group, namely, there is a known group $G$ such that if $x_{i}$ is a point in our data set, then for each $A\in G$ the point $A\cdot x_{i}$ resulting from the action of~$A$ on~$x_i$ is a valid data point (which is not necessarily in the data set, but may be added to it). Such data sets are called~``$G$-invariant''.  For example, in electron-microscopy imaging, a method to determine the 3D structure of a molecule from its 2D images acquired by an electron microscope~\cite{Frank}, all the images lie on a manifold of dimension~3 (diffeomorphic to the 3D rotations group~$SO(3)$). The planar rotation of any such image is a valid image that may have been acquired by the microscope. Thus, the manifold of images is closed under the action of the rotations group $G=SO(2)$. 
	
In~\cite{steerMaps}, it was shown how to construct the graph Laplacian from all given images and all their infinitely many in-plane rotations. This construction is deemed ``the steerable graph Laplacian''. A key result of~\cite{steerMaps} is that the steerable graph Laplacian converges to the Laplace-Beltrami operator on~$\man$ (in this case a~$SO(2)$-invariant compact manifold) faster than the graph Laplacian~\eqref{intro:classicalNormGLDef}. Specifically, it was shown that the steerable graph Laplacian approximates $\Delta_{\man}$ with an error that is given asymptotically by
\begin{equation}\label{intro:convSpeedUp}
	O\left(\frac{1}{N^{1/2}\epsilon^{1/2+(d-1)/4}}\right)+O(\epsilon).
\end{equation} 
The latter error converges to zero at a rate that depends on $d-1$, and so converges to zero faster than the corresponding error term~\eqref{intro:classicalConv} of the standard graph Laplacian~\eqref{intro:classicalNormGLDef}, whose convergence rate depends on~$d$. This improved convergence rate is attributed to the following two facts. The first is that all infinitely many in-plane rotations of each given image are known. The second is that the action of the rotations group on the image manifold~$\man$ accounts for one dimension of~$\man$ (since planar rotations are parametrized by a single angle in~$[0,2\pi)$). 
Combining these two facts implies that the error depends only on~$d-1$ dimensions.
Furthermore, it was shown in~\cite{steerMaps} that the eigenfunctions of the steerable graph Laplacian are tensor products between certain $N$-dimensional vectors and complex exponentials. This special form of the eigenfunctions gives rise to efficient algorithms for their computation. These eigenfunctions are used in~\cite{steerMaps} for filtering noisy functions on $\man$ using a Fourier-like scheme, and are shown to result in an improved error bound compared to the bound achieved by employing the eigenvectors of the standard graph Laplacian~\cite{diffMapsSigProc}. 

This paper is the first part of a two part work presenting a graph Laplacian based framework for the analysis of~$G$-invariant data sets. 
In this paper, we extend the results of~\cite{steerMaps} (which focuses on image manifolds closed under planar rotations) to the setting where the given data points lie on an arbitrary compact manifold~$\man$ of dimension~$d$, closed under the action of an arbitrary compact unitary matrix Lie group~$G$. An example of such a data set is a collection of subtomograms (volumes) in cryo-electron tomography~\cite{HYLTON2021102959,doerr2017cryo}, which due to the experimental setup are arbitrarily rotated in space, that is,~$G=SO(3)$. The results in the current paper also lay the foundations for Part~II~\cite{GDM}, where we develop low-dimensional embeddings of the~$G$-invariant data (which were not proposed in~\cite{steerMaps}) of two types. The first type is a $G$-invariant embedding, which means that any two points which are related by the action of an element of~$G$ are embedded into the same point. In the context of machine learning, this embedding may be used to organize the data into clusters where the points in each cluster are related by the action of a group element (for example, images which are rotations of one another). The second type is a $G$-equivariant embedding, which means that the embeddings of two points which are related by the action of an element of $G$, are themselves related by the action of the same element. Such embeddings may be applied, for instance, to align images which are rotations of one another. 

The contributions of this paper are as follows. 
First, we construct the $G$-invariant graph Laplacian ($G$-GL), which is conceptually the standard graph Laplacian~\eqref{intro:classicalNormGLDef} constructed using a data set consisting of the given data points as well as all (infinitely many) points generated by applying $G$ to the given points. Second, we show that if $d_G$ is the dimension of $G$, then the $G$-GL approximates the Laplace-Beltrami operator on the manifold with an error given asymptotically by
\begin{equation}\label{intro:convSpeedUpG}
	O\left(\frac{1}{N^{1/2}\epsilon^{1/2+(d-d_G)/4}}\right)+O(\epsilon),
\end{equation}  
analogously to \eqref{intro:classicalConv} and \eqref{intro:convSpeedUp}. The result~\eqref{intro:convSpeedUpG} is of great practical importance, as the improved convergence rate implies that significantly less data is required in order to achieve a prescribed accuracy, compared to the standard graph Laplacian.
Third, we derive the eigenfunctions of the $G$-invariant graph Laplacian and show that they admit the form of products between certain vectors and the elements of the irreducible unitary representations of~$G$. Furthermore, we show that this form of the eigenfunctions enables their efficient computation while avoiding explicitly augmenting the input data
(that is, by adding the points~$A\cdot x$ for every point~$x$ in the data set, and all~$A\in G$). 
We then demonstrate the utility of these eigenfunctions in filtering a noisy data set sampled from the four-dimensional unit sphere. 
We comment that different proofs for some of the theoretical results in this paper can be found in~\cite{pauJoe}. The proof strategies in both papers differ in that here we explicitly construct a $G$-invariant local parametrization of the data manifold, whereas~\cite{pauJoe} uses a less concrete approach by passing to the abstract quotient manifold. The advantage of the approach taken here is that while~\cite{pauJoe} uses advanced machinery from fiber bundle theory, here we employ mostly basic instruments from manifold calculus in Euclidean spaces, which is accessible to a wider audience. 

This paper is organized as follows. In Section~\ref{SecRelatedWork}, we review some related work on group invariance and compare it with our approach. In Section~\ref{SecPreliminaries}, we discuss the structure induced on the data manifold by the group action, and introduce some basic machinery from representation theory used in this work. In Section \ref{SecGInv}, we introduce the $G$-invariant graph Laplacian and present its key properties. In Section~\ref{secDenoise}, we demonstrate how to use the eigenfunctions of the $G$-invariant graph Laplacian to filter noisy data sets. In Section~\ref{numericsSec} we describe the details of the numerical computation of the~$G$-GL, and discuss computational complexity. Lastly, in Section~\ref{SecSummary}, we summarize our results and discuss future work. 

\section{Related work}\label{SecRelatedWork}
Other works dealing with group invariance typically focus on rotation invariance, especially in image processing algorithms~\cite{eller2017rotation,zimmer2008rotationally,ji2009moment,singer2012vector,singer2011viewing,zhao2014rotationally}. There are four main approaches in the literature towards rotation invariance. The first approach is based on the steerable PCA~\cite{FFBsPCA2015,landa2017steerable}, which computes the PCA of a set of images and all their infinitely many rotations, namely, finds the linear subspace which best spans a set of images and all their rotations. In a sense, our work is a generalization of this approach to nonlinear manifolds and to general compact matrix Lie groups (and not just rotations). The second approach towards rotation invariance is defining a rotationally-invariant distance for measuring pairwise similarities and constructing graph Laplacians using this distance \cite{singer2012vector}. Unfortunately, it is often not obvious what invariant distance is most appropriate for the task at hand, and how to compute it efficiently. Furthermore, in general,
the limiting operator resulting from such a construction is either unknown, or is not the Laplace-Beltrami operator~\cite{manLearnArbitraryNorm}, in which case, its properties are not well understood. 
In our approach, on the other hand, we consider not only the distance between best matching rotations of image pairs (nor any other type of a rotationally-invariant distance), but rather the standard (Euclidean) distance between all rotations of all pairs of images. We show that all these pointwise Euclidean distances can be computed efficiently by using FFT-type algorithms (when available), and that the resulting operator converges to the Laplace-Beltrami operator on the data manifold. This enables us to preserve the geometry of the underlying manifold (in contrast to various rotation-invariant distances) while making the resulting operator (the $G$-invariant graph Laplacian) invariant to the action of the group on our data set. Moreover, our approach is applicable not only to rotations, but rather to any compact matrix Lie group~$G$.  The third approach to group invariance is based on CNNs~\cite{cnnSO3,cnnSE3,cnnSO2} that produce group equivariant features (for low dimensional rotation groups) by convolving the data with steerable basis functions in each layer. However, this approach lacks solid theory, and in particular, provides no error bounds and no means for analyzing the properties of the resulting tools. Furthermore, unlike CNNs, our approach is applied directly to unlabeled data. 
The fourth approach, also commonly based on CCN's, is to augment a given data set~$X=\left\{x_1,\ldots,x_N\right\}$, by adding to it all the points of the form~$A\cdot x_i$ for some finite set of elements~$A\in G$ \cite{Cohen2016GroupEC,cycSymmCNN,rotInNeuro,Marcos2016LearningRI,bestPracticeCNN}. This approach suffers from several shortcomings. First, since we have chosen a finite set of elements~$A_1,\ldots,A_K\in G$, the augmented data set is only approximately invariant to the action of~$G$. Second, the augmented data set is larger than the original data set by a factor of $K$, which poses computational challenges. Third, if the data are noisy, this approach introduces correlations in the noise of different data points. In contrast, in our approach, we derive a numerically efficient construction of a~$G$-invariant operator that is equivalent to constructing the standard graph Laplacian in~\eqref{intro:classicalGLDef} from all the (infinitely many) points generated by the action of~$G$ on the points in~$X$, without explicitly augmenting~$X$.

We note the work~\cite{ZhangProduct}, which although does not deal with group invariance, makes an important contribution by deriving an algorithm for manifold factorization of product manifolds. We can consider the algorithm in~\cite{ZhangProduct} as a form of invariant learning, as the goal of the algorithm there is to learn submanifolds that are independent of the other submanifolds comprising the product, and thus, can be used to learn the submanifolds which are generated by the action of the group, and factor them out. However, as we later explain, in our setting, any sufficiently small neighborhood of the data manifold~$\man$ is isomorphic to a product of manifolds, but~$\man$ itself need not be a product of manifolds. In that sense, the setting in~\cite{ZhangProduct} is much more restrictive.

Finally, a special attention should be given to~\cite{fan2019unsupervised}. Similarly to the works mentioned above, this work also defines a group invariant distance by looking at a single group element that best ``aligns'' a given pair of points. In that sense, its approach is fundamentally different from what we propose here. Yet, this is the first work we know of that addresses the group invariance problem for arbitrary Lie groups.

\section{Preliminaries} \label{SecPreliminaries}

\subsection{Manifolds under actions of matrix Lie groups}\label{secLieGroupAction}
In this section, we describe our model for data sets closed under the action of a matrix Lie group. In particular, we define matrix Lie groups and their action on the data set. 

\begin{definition}\label{secLieGroupAction:LieGroupDef}
	A matrix Lie group is a smooth (that is, differentiable) manifold $G$, whose points form a group of matrices. 
\end{definition}
For example, consider the group $SU(2)$ of $2\times2$ unitary matrices with determinant~1. Each matrix $A\in SU(2)$ can be written using Euler angles as
\begin{equation}\label{secLieGroupAction:fundIUR}
	A(\alpha,\beta,\gamma) = 
	\begin{pmatrix}
		\cos{\frac{\beta}{2}}e^{i(\alpha+\gamma)/2} & \sin{\frac{\beta}{2}}e^{i(\alpha-\gamma)/2}\\ 
		i\sin{\frac{\beta}{2}}e^{-i(\alpha-\gamma)/2} & \cos{\frac{\beta}{2}}e^{-i(\alpha+\gamma)/2}
	\end{pmatrix},
\end{equation}
where $\alpha\in [0,2\pi), \beta \in [0,\pi)$ and $\gamma\in [-2\pi,2\pi)$. Using the fact that 
the sum of squares of the entries of $A(\alpha,\beta,\gamma)$ equals one, it is easily inferred that~$SU(2)$ is diffeomorphic to the three-dimensional unit sphere $S^3$. Other important examples for matrix Lie groups include the group of  three-dimensional rotation matrices~$SO(3)$, and the $n$-dimensional torus $\mathbb{T}^n$, which is simply the group of diagonal $n\times n$ unitary matrices. 

\begin{definition}\label{secLieGroupAction:gActDef}
	The action of a group $G$ of $n\times n$ matrices on a subset $S\subseteq \mathbb{C}^{n}$ is the map $\text{'}\cdot\text{'}:G\times S\rightarrow S$, defined for each $A\in G$ and $x\in S$ by matrix multiplication on the left~$A\cdot x$. 
	We say that a set~$S$ is closed under the action of a group~$G$, or simply~$G$-invariant, if~$A\cdot x\in S$ for all~$x\in S$ and~$A\in G$.
\end{definition}
In this work, we assume that we are given a data set~$X=\left\{x_1,\ldots,x_N\right\}$ sampled from a smooth, compact, and $G$-invariant manifold~$\man$ without boundary, embedded in~$\mathbb{C}^{n}$, where~$G$ is a unitary matrix Lie group. In particular, the $G$-invariance implies that $A\cdot x \in \man$ for all $x\in\man$ and $A\in G$. 
An additional useful characterization of~$G$-invariant manifolds is derived from the following definition. 
\begin{definition}\label{secLieGroupAction:orbitDef}
	For a fixed point~$x\in \man$, the orbit generated by the action of~$G$ on~$x$ is defined as the set
	\begin{equation}
		G\cdot x \coloneq \left\{A\cdot x \; : \; A\in G,\; x\in \man\right\}. 
	\end{equation}
\end{definition}
Thus, a manifold~$\man$ is~$G$-invariant if~$G\cdot x\subset\man$ for all~$x\in \man$, that is, $\man$ contains all the orbits of the action of~$G$ on its points. In particular, this implies that the set 
\begin{equation}\label{secLieGroupAction:GXDef}
	G\cdot X \coloneq \left\{A\cdot x_i \; :\; A\in G,\; x_i\in X\right\} = \bigcup_{i=1}^N G\cdot x_i,
\end{equation}
of points generated by the action of~$G$ on the data set~$X$ is a~$G$-invariant subset in~$\man$. 
In Section~\ref{SecGInv}, we construct the central object in our framework, namely, the~$G$-invariant graph Laplacian ($G$-GL), which is a graph Laplacian constructed by using not only the points in~$X$ but rather all the points in~$G\cdot X$.

Finally, we will assume that the Lie group~$G$ is also compact. In the rest of this section, we give a short introduction to the theory of harmonic analysis on compact Lie groups, which is essential for the construction of the~$G$-GL. 

\subsection{Haar integration}
The theory of harmonic analysis on matrix Lie groups requires integrating functions over these groups. This is known as ``Haar integration'' since it is performed with respect to the Haar measure, which we now define. 
\begin{definition}
	A Haar measure over a Lie group $G$ is a finite valued, non-negative function~$\eta(\cdot)$ over all (Borel) subsets $S\subseteq G$, such that 
	\begin{equation}\label{secLieGroupAction:leftInvar}
		\eta(A\cdot S) = \eta(S)\quad \text{ for all} \quad A\in G.
	\end{equation}
\end{definition}
By Haar's theorem (see .e.g~\cite{harmAnalysisFolland}), for every compact matrix Lie group there exists a Haar measure which is unique up to a multiplicative constant. In this work, we choose (without loss of generality) the unique measure~$\eta$ such that
\begin{equation}\label{secLieGroupAction:probMeasure}
	\eta(G) =1,
\end{equation}
and henceforth refer to this~$\eta$ as ``the Haar measure over~$G$''.
Essentially, the function $\eta(\cdot)$ measures the volume of subsets of the manifold $G$. Specifically, property~\eqref{secLieGroupAction:probMeasure} makes $\eta(\cdot)$  a probability measure over $G$. Furthermore,  property~\eqref{secLieGroupAction:leftInvar}, known as 'left invariance', means that multiplication by a matrix $A$ from the left maps the set $S\subseteq G$ to another subset of $G$ of the same measure, implying that~$\eta(\cdot)$ is uniform over $G$.
In the context of integration, property~\eqref{secLieGroupAction:leftInvar} implies that the Haar integral is left-invariant, namely, for any $B\in G$ we have that  
\begin{equation}\label{secLieGroupAction:haarIntLeftInvar}
	\int_G f(BA)d\eta(A) = \int_G f(C)d\eta(B^*C) = \int_G f(C)d\eta(C)= \int_G f(A)d\eta(A),
\end{equation} 
where we substituted~$C=BA$ in the first equality, and used~\eqref{secLieGroupAction:leftInvar} in the second equality. 

As an example of a Haar integration, the integral of a function $f$ over $G=SU(2)$ can be computed in terms of Euler angles by (see~\cite{nonComHarmAnalys})
\begin{equation}\label{secLieGroupAction:SU2IntEuler}
	\int_{SU(2)}f(A)d\eta(A) =\frac{1}{16\pi^2}\int_0^{2\pi}\int_0^{\pi}\int_{-2\pi}^{2\pi}f(A(\alpha,\beta,\gamma))\sin\beta d\alpha d\beta d\gamma,
\end{equation}
where $A(\alpha,\beta,\gamma)$ is defined in~\eqref{secLieGroupAction:fundIUR}. In this case, the volume element~$d\eta(A)$ induced by the Haar measure is just $d\alpha d\beta d\gamma$ multiplied by $\frac{\sin\beta}{16\pi^2}$, which is the absolute value of the Jacobian determinant of the parametrization of $SU(2)$ by Euler angles.  

\subsection{Harmonic analysis on compact matrix Lie groups}\label{secHarmAnalysis}
The framework we develop below in Section \ref{SecGInv} employs series expansions of functions over compact matrix Lie groups. The expansion of a function $f:G\rightarrow \mathbb{C}$ is obtained in terms of the elements of certain matrix-valued functions, known as the irreducible unitary representations of $G$, which we now define. 

\begin{definition}\label{harmAnalysisOnG:URdef}
An $n$-dimensional unitary representation of a group $G$ is a unitary matrix-valued function $U(\cdot)$ from $G$ to the group U(n) of $n\times n$ unitary matrices, such that  
	\begin{equation}\label{harmAnalysisOnG:hMorphProp}
		U(A\cdot B) = U(A)\cdot U(B),
	\end{equation}
\end{definition}
and the identity element in~$G$ is mapped to the identity element in~U(n). The homomorphism property \eqref{harmAnalysisOnG:hMorphProp}, implies that the set $\left\{U(A)\right\}_{A\in G}$ is also a matrix Lie group. In particular, the latter implies that each element of the matrix valued function~$U(\cdot)$ is a smooth function over $G$. 

\begin{definition}\label{harmAnalysisOnG:IURdef}
A group representation $U(\cdot)$ is called reducible if there exists a unitary matrix~$P$ such that $P\cdot U(A)\cdot P^{-1}$ is block diagonal for all $A\in G$. A group representation is called irreducible if it is not reducible. We abbreviate irreducible unitray representation as IUR.
\end{definition}
By the Peter-Weyl theorem~\cite{lieGroups}, there exists a countable family $\left\{U^{\ell}\right\}$ of finite dimensional IURs of~$G$, such that the collection~$\left\{U^{\ell}_{ij}(\cdot)\right\}$ of all the elements of all these IURs forms an orthogonal basis for $L^2(G)$.
This implies that any smooth function $f:G\rightarrow \mathbb{C}$ can be expanded in a series of the elements of the IURs of~$G$. 
For example, the IURs of $SU(2)$ in \eqref{secLieGroupAction:fundIUR} are given by a sequence of matrices $\left\{U^{\ell}\right\}$, where~$\ell=0,1/2,1,3/2,\ldots$, and $U^{\ell}(A)$ is a $(2\ell+1)\times (2\ell+1)$ dimensional matrix for each $A\in G$ (see e.g. \cite{nonComHarmAnalys}). In fact, the matrices in~\eqref{secLieGroupAction:fundIUR} correspond to the IUR of~$SU(2)$ with~$\ell=1/2$. 

The series expansion of a function $f:G\rightarrow \mathbb{C}$ is then given by 
\begin{equation}\label{harmAnalysisOnG:SO3Fourier}
	f(A) = \sum_{\ell\in \I_G}d_\ell\cdot\sum_{m,n=1}^{d_\ell} \hat{f}^\ell_{mn}U_{mn}^{\ell}(A),\quad \hat{f}^\ell_{mn} = \int_G f(B)\overline{U^{\ell}_{mn}(B)}d\eta(B),
\end{equation}
where~$\I_G$ is a countable set that enumerates the IURs of~$G$, $d_\ell$ is the dimension of the~$\ell$-th IUR, and $\eta(\cdot)$ is the Haar measure on~$G$. The latter can also be written in the form
\begin{equation}\label{secLieGroupAction:SU2FourierMatCoeff}
	f(A)=\sum_{\ell\in \I_G}  d_\ell\cdot \text{trace}\left(\hat{f}^\ell\cdot  U^\ell(A)\right),
\end{equation}
where $\hat{f}^\ell$ is the $d_\ell\times d_\ell$ matrix given by 
\begin{equation}\label{sec2:fHatDef}
	\hat{f}^\ell= \int_{G} f(A)\overline{U^\ell(A)}d\eta(A),
\end{equation}
for all $\ell\in \I_G$. 

\begin{remark}
The group $SO(2)$ of two-dimensional rotations is a one dimensional matrix Lie group, whose IURs are given by the Fourier modes $\left\{e^{im\theta}\right\}_{m=-\infty}^\infty$. Thus, the series expansion of an $SO(2)$-valued function in terms of the IURs of $SO(2)$ is nothing but the classical Fourier series. In this sense, the expansion~\eqref{secLieGroupAction:SU2FourierMatCoeff} can be viewed as generalized Fourier series over~$G$, with the Fourier modes replaced by the IURs~$\left\{U^\ell\right\}$, and with coefficients given by the matrices $\left\{\hat{f}^\ell\right\}$ of~\eqref{sec2:fHatDef}.
\end{remark}

\section{The $G$-invariant graph Laplacian}\label{SecGInv}
In this section, we construct the~$G$-invariant graph Laplacian ($G$-GL) - a generalization of the standard graph Laplacian~\eqref{intro:classicalGLDef} for data sets sampled from a~$G$-invariant manifold~$\man$. We then compute the~$G$-GL's eigendecomposition, and show that a proper normalization of the~$G$-GL converges to the Laplace-Beltrami operator on~$\man$ significantly faster than~\eqref{intro:classicalGLDef}.

Let $X = \left\{x_1,\ldots, x_N\right\}$ be a data set sampled from a~$G$-invariant (see Definition~\ref{secLieGroupAction:gActDef}) compact manifold~$\man\subset \mathbb{C}^n$.
Our goal is to construct the graph Laplacian by using all the points in the~$G$-invariant set $G\cdot X$ in~\eqref{secLieGroupAction:GXDef}. As we will see shortly, our construction results in an operator (rather than a matrix) over a certain Hilbert space, which we now define. 

\begin{definition}\label{Ginvsec:HspaceDef}
	Given a data set $X=\left\{1,\ldots,N\right\}$, let~$\Gamma$ be the set of pairs
	\begin{equation}\label{Ginvsec:GammaSetDef}
		\Gamma \coloneq \left\{1,\ldots,N\right\}\times G  = \left\{(i,A)\; :\; i\in\left\{1,\ldots,N\right\},\; A\in G\right\}, 
	\end{equation}
	where each pair~$(i,A)$ corresponds to the point $A\cdot x_i\in G\cdot X$. 
	We define the Hilbert space $\Hspace = L^2(\Gamma)$ as the space of functions of the form~$f(i,A) = f_i(A)$, where~$f_i\in L^2(G)$ for all~$i\in \left\{1,\ldots,N\right\}$, endowed with the inner product 
	\begin{equation}\label{GinvSec:InnerProdDef}
		\left \langle f,g \right \rangle_{\mathcal{H}} = \sum_{i=1}^N\int_{G}f_i(A)\overline{g_i(A)}d\eta(A),
	\end{equation} 
	where $\eta(\cdot)$ is the Haar measure on $G$. 
\end{definition}
Now, let~$D$ be an~$N\times N$ diagonal matrix. We define the action of~$D$ on a function~$f\in \Hspace$ by
\begin{equation}\label{GinvDef:DMatAction} 
	\left\{Df\right\}(i,A) =D_{ii}\cdot f_i(A), \quad A\in G,
\end{equation}
where~$D_{ii}$ is the $i$'th element on the diagonal of~$D$. Equipped with Definition~\ref{Ginvsec:HspaceDef} and~\eqref{GinvDef:DMatAction}, we are now ready to define the~$G$-GL. 
\begin{definition}
	Let~$W:\Hspace\rightarrow \Hspace$ be the operator acting on functions~$f\in \Hspace$ by
	\begin{equation}\label{GinvDef:Wdef}
		\left\{Wf\right\}(i,A) = \sum_{j=1}^{N}\int_{G} W_{ij}(A,B)f_j(B)d\eta(B), \quad W_{ij} = e^{-\norm{A\cdot x_i-B\cdot x_j}^2/\epsilon},
	\end{equation}
	and let~$D$ be the~$N\times N$ diagonal matrix defined by
	\begin{equation}\label{GinvDef:Ddef}
		D=\operatorname{diag}\left(D_{11},\ldots,D_{NN}\right),\quad 	D_{ii} = \sum_{j=1}^N\int_{G}W_{ij}(I,C)d\eta(C),
	\end{equation}
	where~$I$ is the identity element in~$G$. The $G$-invariant graph Laplacian ($G$-GL) is defined as the operator~$L:\Hspace \rightarrow \Hspace$ given by
	\begin{equation}\label{GinvDef:Ldef}
		L =  D-W.
	\end{equation}
\end{definition}
Note that by~\eqref{GinvDef:Wdef}, we have that $W_{ij}(A,B) = W_{ji}(B,A)$ for all~$i,j\in\left\{1,\ldots,N\right\}$ and~$A,B\in G$, which implies that the operator~$W$ is symmetric. Combining the latter with~\eqref{GinvDef:Ldef} implies the same for~$L$. The following result asserts that~$L$ is a positive semi-definite operator. 
\begin{lemma}\label{GinvDef:quadFormLemma}
	The~$G$-GL admits the positive semi-definite quadratic form
	\begin{equation}\label{GinvSec:quadForm}
		\dprod{f}{Lf}_\Hspace=\frac{1}{2}\sum_{i,j=1}^{N}\int_G\int_G W_{ij}(A,B) \left|f_i(A)-f_j(B)\right|^2 d\eta(A)d\eta(B).
	\end{equation}
\end{lemma}
The proof of~Lemma~\ref{GinvDef:quadFormLemma} is given in~Appendix~\ref{quadFormLemmaPrf}.
The form~\eqref{GinvSec:quadForm} is analogous to the quadratic form of the standard graph Laplacian~\cite{manReg}. Thus, this form is important on its own right since it can be used as a smoothness regularization term in various machine learning algorithms where the objective function is assumed to have been sampled over a compact~$G$-invariant manifold. This idea was first proposed in~\cite{manReg}, and rigorously justified in~\cite{Cheng2020ConvergenceOG}. Intuitively, the quantity~$\dprod{f}{Lf}_\Hspace$ puts large penalties on the differences~$\left|f_i(A)-f_j(B)\right|$ when $W_{ij}$ is large, that is, when there exist~$A,B\in G$ such that the points~$A\cdot x_i$ and~$B\cdot x_j$ are close.
Thus, the quantity~$\dprod{f}{Lf}_\Hspace$ can be viewed as imposing a notion of smoothness on functions over the domain~$\Gamma$ in~\eqref{Ginvsec:GammaSetDef}. 

Analogously to the results in \cite{steerMaps,convRate}, below we show that the normalization 
\begin{equation}\label{GinvDef:normGLapDef}
	\tilde{L} = D^{-1}L = I-D^{-1}L,
\end{equation}
of~$L$ in~\eqref{GinvDef:Ldef} converges to the Laplace-Beltrami operator on~$\man$. 
While other useful normalizations of~$L$ are possible, in the current work, we mainly focus on~\eqref{GinvDef:normGLapDef}. Thus, we henceforth refer to~\eqref{GinvDef:normGLapDef} as the normalized~$G$-GL.

As mentioned above, unlike previous works~\cite{lapMap,singer2012vector}, our construction results in an operator over a Hilbert space rather than a matrix (compare with~\eqref{intro:classicalGLDef}). This is a direct consequence of the continuous nature of the set~$\Gamma$ in~\eqref{Ginvsec:GammaSetDef}, being a product between a discrete set and a Lie group~$G$, on account of~$G$ being a smooth manifold by Definition~\ref{secLieGroupAction:LieGroupDef}.
As we will see next, the continuity of~$G$ also implies that the~$G$-GL admits an infinitely-countable basis of eigenfunctions for the space~$\Hspace$ (see Definition~\ref{Ginvsec:HspaceDef}) instead of the finite set of eigenvectors of the graph Laplacian matrix in~\eqref{intro:classicalGLDef}. In particular, the eigenfunctions of the~$G$-GL can be evaluated for any~$A\in G$.

\subsection{Eigendecomposition of the $G$-GL}
We now derive the eigendecompostions of the $G$-GL~\eqref{GinvDef:Ldef}, and its normalized version~\eqref{GinvDef:normGLapDef}.
Let~$\man\subset \mathbb{C}^n$ be a~$G$-invariant compact and smooth manifold, without a boundary, where~$G$ is a compact Lie group of unitary~$n\times n$ matrices. Let $X = \left\{x_1,\ldots,x_N\right\}\subset \man$ be a data set sampled from~$\man$. 
By~\eqref{GinvDef:Wdef}, and since~$G$ is unitary, for each~$i,j\in\left\{1,\ldots,N\right\}$ we have that
\begin{equation}\label{GinvSec:affinityKernelDef}
	W_{ij}(A,B) = W_{ij}(I,A^*B), \quad A,B\in G.
\end{equation}
That is, each function~$W_{ij}$ in~\eqref{GinvDef:Wdef} only depends on the quotient~$A^*B$.
Thus, by using~\eqref{secLieGroupAction:SU2FourierMatCoeff}, we can expand the function $W_{ij}(I,A^*B)\in \mathcal{H}$ in the Fourier series
\begin{equation}\label{sec1:fourierSO3}
	W_{ij}(I,A^*B)=\sum_{\ell\in \I_G } d_\ell\cdot \text{trace}\left(\hat{W}^\ell_{ij}U^\ell(A^*B)\right),
\end{equation}
where by~\eqref{sec2:fHatDef} and~\eqref{secLieGroupAction:haarIntLeftInvar}, $\hat{W}_{ij}^\ell$ is the $d_\ell\times d_\ell$ matrix given by
\begin{equation}\label{sec2:hat{W}Def}
	\hat{W}_{ij}^\ell= \int_G W_{ij}(I,A^*B)\overline{U^\ell(A^*B)}d\eta(B)=  \int_G W_{ij}(I,A)\overline{U^\ell(A)}d\eta(A),
\end{equation}
for each $\ell\in\I_G$. 

Clearly, the $G$-GL is completely characterized by the set of matrices~$\hat{W}^\ell_{ij}$ of \eqref{sec2:hat{W}Def}, since for any $i$ and $j$, the kernel function $W_{ij}(I,A^*B)$ can be recovered from them. 
The following theorem characterizes the eigendecomposition of $L$ of~\eqref{GinvDef:Ldef} in terms of certain products between the columns of the IURs $U^\ell$, and the eigenvectors of the block matrices \begin{equation}\label{sec2:blockFourierMat}
	\hat{W}^\ell = 	\begin{pmatrix}
		\hat{W}^\ell_{11} & \hat{W}^\ell_{12}& ... & \hat{W}^\ell_{1N}\\
		\vdots & \ddots &  & \vdots\\
		\vdots & & \ddots   & \vdots\\
		\hat{W}^\ell_{N1} & \hat{W}^\ell_{N2}& ... & \hat{W}^\ell_{NN}		
	\end{pmatrix}, 
	\quad \ell\in \I_G,
\end{equation}
of dimension $Nd_\ell\times Nd_\ell$ whose $ij$-th block of size~$d_\ell\times d_\ell$ is~$\hat{W}^\ell_{ij}$ of~\eqref{sec2:hat{W}Def}.
To derive the eigendecomposition, we introduce the following notation. 
For any vector $v \in \mathbb{C}^{Nd_\ell}$ and any $j\in\{1,\ldots,N\}$, we denote by
\begin{equation}\label{sec1:parVecNotation}
	e^j(v) = (v((j-1)d_\ell+1),\ldots,v(jd_\ell))\in \mathbb{C}^{d_\ell},
\end{equation}
the elements $(j-1)\cdot d_\ell+1$ up to $j\cdot d_\ell$ of $v$ stacked in a $d_\ell$-dimensional row vector.
\begin{theorem}\label{GInvLapProp:Thrm1}
	For each $\ell\in \I_G$, let $D^\ell$ be the $Nd_\ell\times Nd_\ell$ block-diagonal matrix whose $i$-th block of size $d_\ell\times d_\ell$ on the diagonal is given by the product of the scalar~$D_{ii}$ in~\eqref{GinvDef:Ddef} with the~$d_\ell\times d_\ell$ identity matrix. 
	Then, the $G$-invariant graph Laplacian~$L$ in~$\eqref{GinvDef:Ldef}$ admits the following:
	\begin{enumerate}
		\item A sequence of non-negative eigenvalues $\{\lambda_{1,\ell},\ldots,\lambda_{Nd_\ell,\ell}\}_{\ell\in \I_G}$, where~$\lambda_{n,\ell}$ is the~$n$-th eigenvalue of the matrix~$D^\ell-\hat{W}^\ell$. 
	  \item  A sequence $\{\Phi_{\ell,1,1},\ldots,\Phi_{\ell,d_\ell,Nd_\ell}\}_{\ell\in \I_G}$ of eigenfunctions, which are orthogonal and complete in $\mathcal{H}$ and are given by 
	  \begin{equation}\label{GInvLapProp:EigenFuncForm}
	  	\Phi_{\ell,m,n}(i,A) = 
	   e^i(v_{n,\ell})\cdot
	  	U^\ell_{\cdot,m}(A^*),
	  \end{equation}
	  where $v_{n,\ell}$ is the eigenvector of $D^\ell-\hat{W}^\ell$ which corresponds to its eigenvalue~$\lambda_{n,\ell}$. Furthermore, for each $n\in \{1,\ldots, Nd_\ell\}$ and $\ell\in\I_G$, the eigenfunctions $\{\Phi_{\ell,1,n},\ldots,\Phi_{\ell,d_\ell,n}\}$ correspond to the eigenvalue $\lambda_{n,\ell}$ of the $G$-invariant graph Laplacian. 
	\end{enumerate}
\end{theorem}
The proof of Theorem~\ref{GInvLapProp:Thrm1} is given in Appendix~\ref{eigDecomPrf}.
A nearly identical theorem (Theorem~\ref{GInvLapProp:Thrm1Norm}) characterizing the eigendecomposition of the normalized $G$-GL in~\eqref{GinvDef:normGLapDef} is given below in Appendix~\ref{secEigDecompNGGL}. Theorem~\ref{GInvLapProp:Thrm1Norm} states that for the normalized~$G$-GL we only need to replace the eigenvectors $\left\{v_{n,\ell}\right\}_{n,\ell}$ above with the eigenvectors $\left\{\tilde{v}_{n,\ell}\right\}_{n\ell}$ of the sequence of matrices 
\begin{equation}\label{GinvProp:normFourierMat}
	S^\ell = I-(D^\ell)^{-1}\hat{W}^\ell, \quad \ell \in \I_G,
\end{equation}
with the only difference that the resulting eigenfunctions 
\begin{equation}\label{GInvLapProp:normEigenFunc}
	\{\tilde{\Phi}_{\ell,1,1},\ldots,\tilde{\Phi}_{\ell,d_\ell,Nd_\ell}\}_{\ell\in \I_G},
\end{equation}
are no longer orthogonal due to the fact that the matrices in \eqref{GinvProp:normFourierMat} are generally not Hermitian.

The form of the eigenfunctions in~\eqref{GInvLapProp:EigenFuncForm} is of practical importance for numerical computations, as it implies that the 
eigendecomposition of the $G$-GL can be obtained by diagonalizing the sequence of matrices $\hat{W}^\ell$ of~\eqref{sec2:blockFourierMat}. Furthermore, for groups which are common in applications (e.g. $SO(3)$) all the elements of the Fourier matrices $\left\{\hat{W}^\ell\right\}$ can be computed efficiently by employing generalized FFT algorithms~\cite{nonComHarmAnalys,fastNFTSO3Potts}. We provide the details of such a computational procedure for the case $G=SU(2)$ in Appendix~\ref{appnedix:numerIntegration} below. 

\subsection{Convergence of the $G$-invariant graph Laplacian}\label{secGGLConv}
We now show that the normalized $G$-invariant graph Laplacian \eqref{GinvDef:normGLapDef} converges to the Laplace-Beltrami operator on $\mathcal{M}$. Furthermore, we show that the convergence has an improved rate which scales with $d-d_G$ instead of $d$, where $d_G$ is the dimension of the group $G$.  
\begin{theorem}\label{sec2:ConvThrmUnnormalized}
	Let $\man$ be a smooth $d$-dimensional compact manifold without boundary, closed under the action of a matrix Lie group~$G$. Let $\left\{x_1,\ldots,x_N\right\}\in \mathcal{M}$ be i.i.d with the uniform probability density function $p(x) = 1/\text{Vol}(\mathcal{M})$, and suppose that $A\cdot x_i \neq x_i$ for all $A\in G$ with probability one. Let $f:\mathcal{M}\rightarrow\mathbb{R}$ be a smooth function, and define $g\in\mathcal{H}$ so that $g(i,A) = f(A\cdot x_i)$. Then, with high probability, we have that 
	\begin{equation}\label{convSec:errFormula}
		\frac{4}{\epsilon}\left\{\tilde{L}g\right\}(i,A) = \Delta_{\mathcal{M}}f(A\cdot x_i)+ O\left(\frac{1}{N^{1/2}\epsilon^{1/2+(d-d_G)/4}}\right) +O(\epsilon).
	\end{equation}
\end{theorem}
The proof of Theorem~\ref{sec2:ConvThrmUnnormalized} is given in~Appendix~\ref{convTheoremPrf}.
We point out that the requirement that $Ax_i\neq x_i$ with probability one ensures that the orbits generated by the action of~$G$ (see Definition~\ref{secLieGroupAction:orbitDef}) are diffeomorphic to $G$. This eliminates pathological cases in which 
the convergence analysis of the $G$-GL may become over-complicated or even superfluous, while still accounting for a broad class of data manifolds. 

Inspecting \eqref{convSec:errFormula}, we observe that as $N\rightarrow \infty$, the $G$-GL estimates $\Delta_\man$ with a bias error of~$O(\epsilon)$ given by the third term on the r.h.s. The second term on the r.h.s accounts for the variance of the estimator when the sample size~$N$ is finite. Thus, we conclude that the $G$-GL reduces the variance error compared to that of the standard GL in~\eqref{intro:classicalConv}, proportionally to the dimension $d_G$ of~$G$. 
The improvement in the variance error~\eqref{convSec:errFormula} in comparison to~\eqref{intro:classicalConv} can be explained as follows. 
In the proof of Theorem~\ref{sec2:ConvThrmUnnormalized}, we show that any sufficiently small neighborhood~$\man'\subset \man$ can be written as a disjoint union of orbits generated by~$G$. In fact, we show that there exists a set of $d$ coordinates for~$\man'$ such that given a point~$x\in \man'$, the first~$d-d_G$ coordinates specify the orbit in which~$x$ resides, while 
the last~$d_G$ coordinates indicate the position of~$x$ on that orbit. In other words, these last~$d_G$ coordinates are the ``directions'' in which~$G$ acts on~$\man$.   
The construction of the $G$-GL incorporates all the points in the set~$G\cdot X$ in~\eqref{secLieGroupAction:GXDef}, namely, those generated by following the directions in~$\man$ in which~$G$ acts on the data set~$X\subset \man$. Thus, the variance error of approximating $\Delta_\man$ by the $G$-GL stems entirely from sampling the remaining~$d-d_G$ directions in~$\man$, resulting in the reduced variance error in~\eqref{convSec:errFormula}.

\begin{remark}
	In Theorem~\ref{sec2:ConvThrmUnnormalized}, we have assumed that the sampling density~$p(x)$ is uniform over~$\man$. In~Appendix~\ref{secNonUniformDistProof}, we show that in the case where $p(x)$ is non-uniform, the operator~$\tilde{L}$ in~\eqref{GinvDef:normGLapDef} converges to the Fokker-Planck operator $\tilde{\Delta}_\man$, given for every smooth function~$f:\man \rightarrow \mathbb{R}$ by
	\begin{equation}
		\tilde{\Delta}_{\man} = \Delta_\man f -2 \frac{\dprod{\nabla_\man f(x)}{\nabla_\man \tilde{p}(x)}}{\tilde{p}(x)}, 
	\end{equation}
where~$\tilde{p}$ is the probability density given by
\begin{equation}\
	\tilde{p}(x) = \int_G p(A\cdot x)d\eta(A).
\end{equation}
Furthermore, we show that there exists a normalization~$\bar{L}$ of~$L$ in~\eqref{GinvDef:Ldef} (different from~$\tilde{L}$ in~\eqref{GinvDef:normGLapDef}) that still converges to~$\Delta_\man$.
\end{remark}

The practical importance of Theorem \ref{sec2:ConvThrmUnnormalized} is that we expect the eigenvalues and eigenfunctions of the $G$-GL to approximate those of the Laplace-Beltrami operator $\Delta_\man$ better than the standard normalized graph Laplacian~\eqref{intro:classicalNormGLDef}. We support this conjecture by simulations in the following section. 

\subsection{Numerical examples}\label{sec:toyExample}
At this point, we wish to demonstrate the improved convergence rate \eqref{convSec:errFormula} with some numerical examples. In the following simulation, we let the group $G=SU(2)$ (of $2\times 2$ unitary matrices with determinant one) act on a data set sampled from the four-dimensional sphere $S^4$, as follows. First, we sample a set of~$N$ points~$\left\{p_1,\ldots,p_N\right\}\in S^4$ and embed them in the Euclidean space $\mathbb{C}^3$ via the map
\begin{equation}\label{secConvSim:embed}
	x_i\left(p_{i,1},\ldots,p_{i,5}\right)= (p_{i,1}+ip_{i,2},p_{i,3}+ip_{i,4},p_{i,5}),\quad p_{i,1}^2+\cdots +p_{i,5}^2 =1,
\end{equation}
where we denote by~$p_{i,j}$ the $j$-th coordinate in~$p_i$. 
Then, we let the group $SU(2)$ act on each embedded point $x_i$ of \eqref{secConvSim:embed} via the multiplication 
\begin{equation}\label{secConvSim:toyExampleGroupAction}
	\begin{pmatrix}
	A & \\
	& 1 
\end{pmatrix}\cdot x_i,\quad  A\in SU(2),
\end{equation}
where $SU(2)$ was defined explicitly in \eqref{secLieGroupAction:fundIUR}.
We then apply the $SU(2)$-invariant graph Laplacian to the test function
\begin{equation}\label{secConvSim:testFunc}
	f(x_i) = \text{Re}(x_{i,1})+\text{Im}(x_{i,1}) = p_{i,1}+p_{i,2}, 
\end{equation}
at the point $x_0 = (1/2+i/2,1/2+i/2,0)$, where we denote by~$x_{i,j}$ the $j$-th coordinate of~$x_i$.
It can be shown that the coordinate functions $h_j(p_i)=p_{i,j}$ on~$S^4$ are eigenfunctions of the Laplace-Beltrami operator~$\Delta_{S^4}$ corresponding to the eigenvalue~$\lambda = -4$ (see~\cite{harmFunTheoryAxler}). Thus, we have that $\Delta_{S^4}f = -4\cdot (p_{i,1}+p_{i,2})$ and $\Delta_{S^4}f(x_0) = -4$. To demonstrate the convergence and variance error of~\eqref{convSec:errFormula}, we uniformly sample~$N=5000$ points~$p_i\in S^4$, and generate the data set $X=\{x_1,\ldots,x_N\}$ by using~\eqref{secConvSim:embed}. We then approximate~$\Delta_{S^4}f(x_0)$ by applying~$\tilde{L}$, the normalized $SU(2)$-GL, to the function $g(i,A) = f(A\cdot x_i)$ for~$A\in SU(2)$ 
by
\begin{figure}
	\centering
	\subfloat[]  	
	{
		\includegraphics[width=0.46\textwidth]{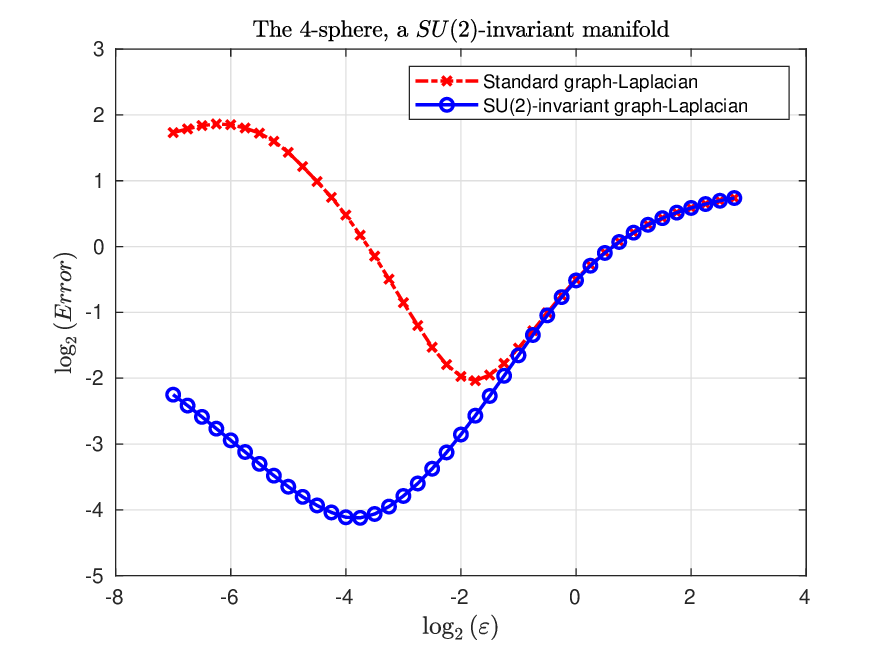}
	}
	\subfloat[]    
	{ 
		\includegraphics[width=0.46\textwidth]{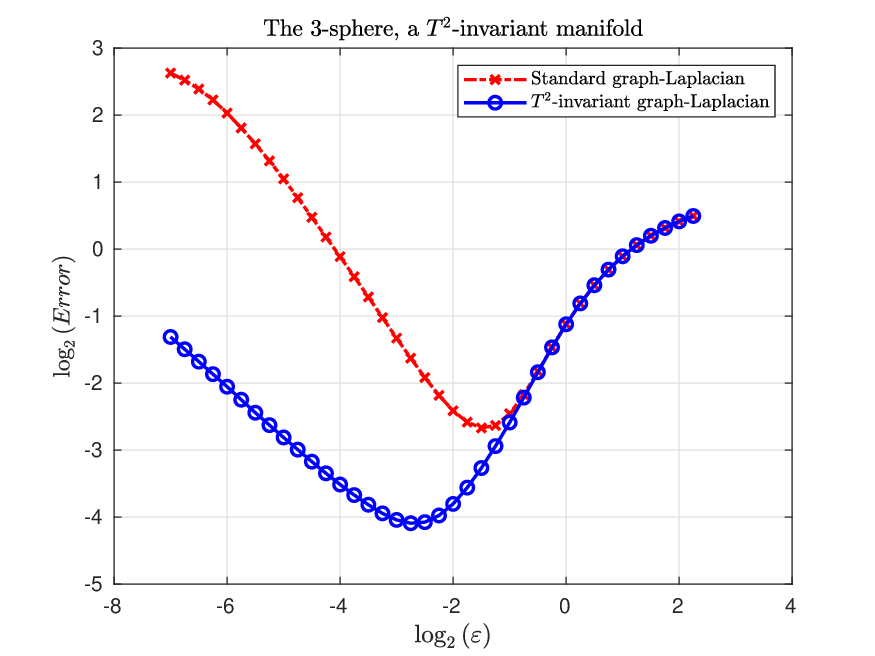}
	}
	\caption{Improved convergence rates of the ${SU}(2)$-invariant GL and the $\mathbb{T}^2$ invariant GL.} 
	\label{fig1}
\end{figure} 
\begin{align}\label{secConvSim:ParametrizedG}
	\frac{4}{\epsilon}&\left\{\tilde{L}g \right\}(0,I) = \frac{4}{\epsilon} \left[f(I\cdot x_0) - \frac{\sum_{j=1}^N\int_{SU(2)}W_{0,j}(I,A)f(A\cdot x_j)d\eta(A)}{\sum_{j=1}^N\int_{SU(2)}W_{0,j}(I,A)d\eta(A)}\right]. 
\end{align} 
The quantity~\eqref{secConvSim:ParametrizedG} can be approximated efficiently by using the parametrization~\eqref{secLieGroupAction:fundIUR} of~$SU(2)$ by Euler angles, together with Gauss-Legendre quadratures to approximate the integrals over~$SU(2)$.

We observe that for large values of~$\epsilon$, the error~\eqref{convSec:errFormula} is dominated by the term~$O(\epsilon)$, while
for small values of $\epsilon$, the error is dominated by the middle term on the r.h.s of~\eqref{convSec:errFormula}, which accounts for the sampling variance of the approximation. Thus, we refer to the error for small values of~$\epsilon$ as the 'variance dominated region' of the error. 

The results of this experiment are depicted in Figure~\ref{fig1}(a), where we plot the log-error of approximation of $\Delta_{S^4}f(x_0)$ by~\eqref{secConvSim:ParametrizedG} against different values of $\log(\epsilon)$.
The slope of the log-error in the variance dominated region is -1.4122 for the normalized standard graph Laplacian (abbreviated standard-GL) and -0.7048 for~the normalized~$SU(2)$-GL, supporting the classical result~\eqref{intro:classicalConv} for the normalized standard-GL, and~\eqref{convSec:errFormula} for~the normalized~$SU(2)$-GL, which predict slopes of~-1.5 and~-0.75 respectively, when substituting $d=4$ and $d_G=3$. 

As another example, we simulated the action of the torus group~$\mathbb{T}^2$, defined as the group of all diagonal~$2\times 2$ unitary matrices, on the unit 3-sphere~$S^3$. In a similar fashion to our first example, we embed the sampled data points~$\left\{p_1,\ldots,p_N\right\}\in S^3$ into~$\mathbb{C}^2$ via the map 
\begin{equation}\label{secConvSim:S3embed}
	x_i\left(p_{i,1},\ldots,p_{i,4}\right)= (p_{i,1}+ip_{i,2},p_{i,3}+ip_{i,4}),\quad p_{i,1}^2+\cdots +p_{i,4}^2 =1,
\end{equation}
and let~$\mathbb{T}^2$ act on each point~$p_i$ by matrix multiplication. We then repeat the steps of our first simulation, computing the~$\mathbb{T}^2$-GL by using~$N=5000$ samples from~$S^3$ and applying it to the function~$g(i,A) = f(A\cdot x_i)$ for $f$ in~\eqref{secConvSim:testFunc} (defined over~$S^3$), at the point~$x_0 = (1/2+i/2,1/2+i/2)$. Using the fact the coordinate functions~$h_j(p_i) = p_{i,j}$ on~$S^3$ are eigenfunctions of~$\Delta_{S^3}$ corresponding to the eigenvalue~$\lambda=-3$ (see~\cite{harmFunTheoryAxler}), we obtain that~$\Delta_{S^3}f = -3\cdot f$. 
The plot of the logs of the approximation errors of~$\Delta_{S^3}f(x_0)$ by the normalized~$\mathbb{T}^2$-GL and the normalized standard-GL against different values of~$\log(\epsilon)$ show the same qualitative picture as in the first simulation. In particular, the slope of the log-error in the variance dominated region is~$-1.2171$ for the normalized standard-GL, and~$-0.7454$ for the $\mathbb{T}^2$-GL, supporting the results ~\eqref{intro:classicalConv} and~\eqref{convSec:errFormula}, which predict slopes of~-1.25 and~-0.75, respectively, when~$d=3$ and~$d_G=2$ (since~$\mathbb{T}^2$ is a two-dimensional group).

\begin{figure}
	\centering
	\subfloat[]  	
	{
		\includegraphics[width=0.45\textwidth]{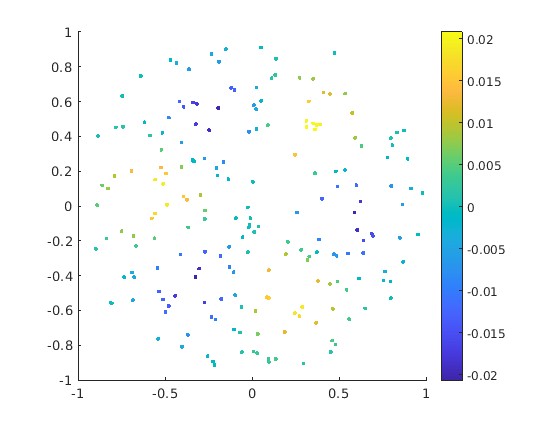}
	}\label{fig:eigVecOnX}
	\subfloat[] 	
	{
		\includegraphics[width=0.45\textwidth]{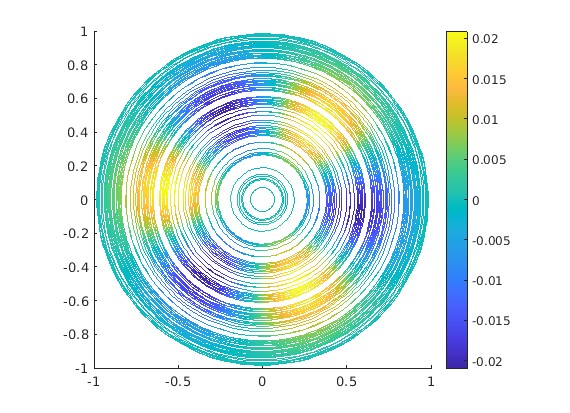}
	}\label{fig:eigVecOnGX}
	\centering
	\subfloat[]    
{ 
	\includegraphics[width=0.6\textwidth]{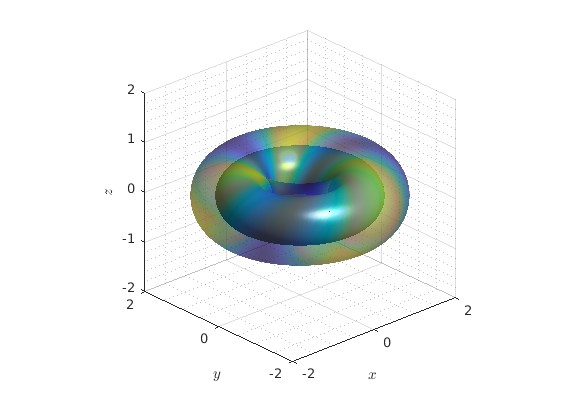}
}
	\caption[]
	{The real part of the eigenfunction~$\Phi_{(2,4),1,1}$. Figure~(a) shows the values at points in the data set~$X$ which were projected to the~$xy$ plane. Figure~(b) shows the values at circles generated by the action of~$\mathbb{T}^2$ on~$S^3$. Figure~(c) shows the nested tori obtained via stereographic projection onto~$\mathbb{R}^3$ of two of the orbits in~$S^3$ generated by the action of~$\mathbb{T}^2$.} \label{fig:sphere eigenfunctions}
\end{figure}

\begin{figure}
	\centering
	\subfloat  	
	{
		\includegraphics[width=0.95\textwidth]{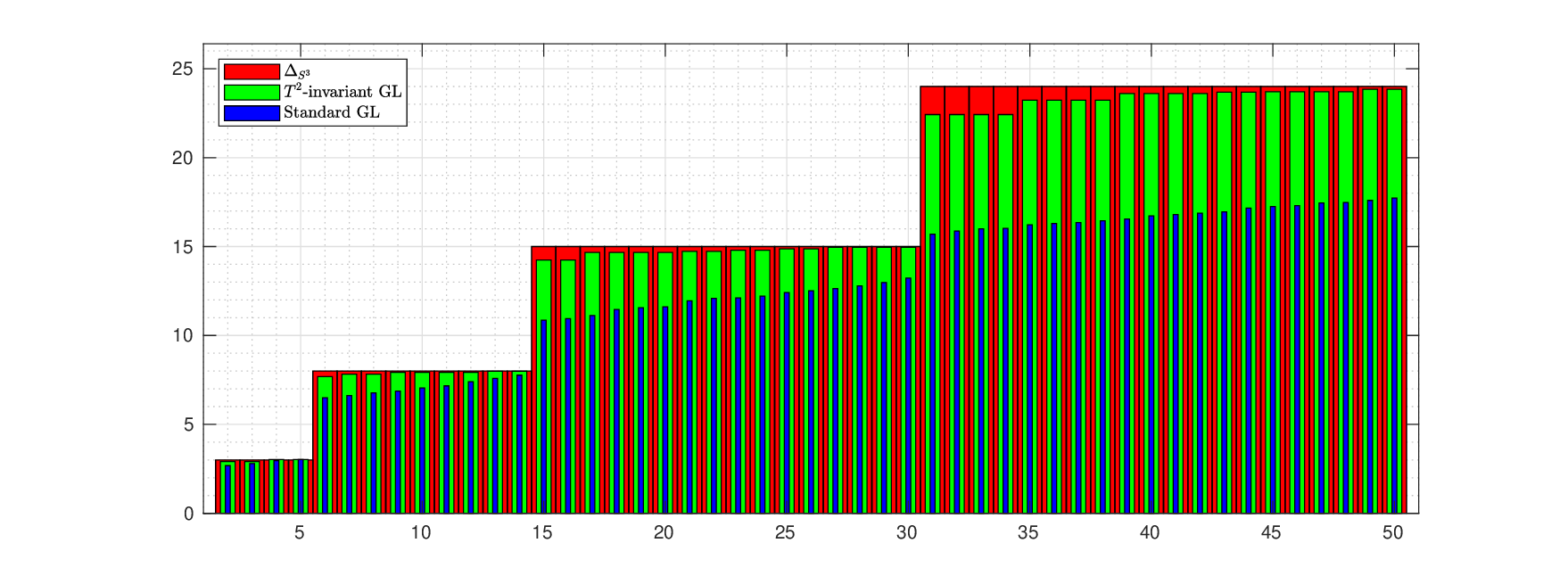}
	}
	\caption[sphere eigenvalues]
	{The~50 smallest eigenvalues of the normalized $\mathbb{T}^2$-GL, scaled by~$4/\epsilon$ (green), the normalized standard-GL, also scaled by~$4/\epsilon$ (blue), and~50 smallest eigenvalues of~$\Delta_{S^3}$ (red). Both graph Laplacians were computed by using the same~$N=5000$ data points, where~$\epsilon=2^{-7}$ for the normalized~$\mathbb{T}^2$-GL, and~$\epsilon=2^{-3}$ for the normalized standard-GL. } \label{fig:sphere eigenvalues} 
\end{figure}
We also computed the~$50$ smallest eigenvalues of the normalized $\mathbb{T}^2$-GL on $S^3$, scaled by~$4/\epsilon$ in accordance with~\eqref{convSec:errFormula}, and the~$50$ smallest eigenvalues of the normalized standard-GL, also scaled by~$4/\epsilon$ (see~\eqref{intro:classicalConv}). We used (the same)~$N=5000$ points for the construction of both graph Laplacians, with bandwidth parameter values of~$\epsilon = 2^{-7}$ for the normalized~$\mathbb{T}^2$-GL, and~$\epsilon=2^{-3}$ for the standard graph Laplacian. The values of~$\epsilon$ were chosen to minimize the mean absolute error of approximating the eigenvalues of~$\Delta_{S^3}$ by those of each graph Laplacian.  
The results are illustrated in Figure~\ref{fig:sphere eigenvalues}. The red bars depict the eigenvalues of~$\Delta_{S^3}$ which are given by the~$5$ unique values~$0,3,8,15$ and~$24$ with respective multiplicities~$1,4,9,16$ and~$25$ (see e.g. \cite{harmFunTheoryAxler}). The green and blue bars depict
the eigenvalues of the normalized~$\mathbb{T}^2$-GL, and those of the normalized standard-GL, respectively. 
While for both graph Laplacians the multiplicities are in agreement with those of~$\Delta_{S^3}$, it is clear that the eigenvalues of the normalized~$\mathbb{T}^2$-GL better approximate those of~$\Delta_{S^3}$ than those of the normalized standard-GL. 

Lastly, we illustrate how constructing the normalized~$\mathbb{T}^2$-GL by using all the points in~$\mathbb{T}^2\cdot X\subset S^3$ is manifested in the eigenfunctions~\eqref{GInvLapProp:EigenFuncForm}. The IUR's of~$\mathbb{T}^2$ (see Definitions~\ref{harmAnalysisOnG:URdef} and~\ref{harmAnalysisOnG:IURdef}) are all one-dimensional, and are given by the set of products of Fourier modes~$\{e^{il_1\theta}\cdot e^{il_2\varphi}\}$, which can be conveniently enumerated by the set~$\I_{\mathbb{T}^2} = \left\{(\ell_1,\ell_2)\;:\; \ell_1,\ell_2\in \mathbb{Z} \right\}$. Thus, Theorem~\ref{GInvLapProp:Thrm1} implies that the eigenfunctions~$\Phi_{\ell,m,n}$ in~\eqref{GInvLapProp:EigenFuncForm} take the form of a Kronecker product between an~$N$-dimensional vector and a bivariate function~$e^{il_1\theta}\cdot e^{il_2\varphi}$. 
To visualize the eigenfunctions, we first map the points in~$\mathbb{T}^2\cdot X$ to~$\mathbb{R}^3$ by using the stereographic projection from~$S^3\subset \mathbb{R}^4$. It can be shown that each orbit~$\mathbb{T}^2\cdot x_i\subset S^3$ gets projected to a torus in~$\mathbb{R}^3$ (a ``bagel-shaped" surface), and furthermore, that the image of~$S^3$ under this projection is a union of nested tori that fill all of~$\mathbb{R}^3$. Figure~\ref{fig:sphere eigenfunctions}(c) depicts two of these tori (one nested inside the other), generated by the action of~$\mathbb{T}^2$ on a pair of data points in~$X$, colored according to the values of~$\text{Re}\left\{\Phi_{(2,4),1,1}\right\}$, the real part of the function~$\Phi_{(2,4),1,1}$ (i.e. $\ell = (2,4))$. In Figure~\ref{fig:sphere eigenfunctions}(a), we show the values of~$\text{Re}\left\{\Phi_{(2,4),1,1}\right\}$ at the points of the stereographic projection of~$X\subset S^3$, which were projected to the~$xy$-plane in~$\mathbb{R}^3$, and in Figure~\ref{fig:sphere eigenfunctions}(b), we show the values of~$\text{Re}\left\{\Phi_{(2,4),1,1}\right\}$ at intersection of the~$xy$-plane with all the tori generated by the action of~$\mathbb{T}^2$ on those points, which happens at planar circles. In particular, each circle in Figure~\ref{fig:sphere eigenfunctions}(b) is generated by the action of~$\mathbb{T}^2$ on a point in Figure~\ref{fig:sphere eigenfunctions}(a), which illustrates how the eigenfucntions account for the group action.  

\section{Denoising $G$-invariant data sets}\label{secDenoise}
We now demonstrate how to apply Theorem \ref{GInvLapProp:Thrm1} to denoise a data set sampled from an $SU(2)$-invariant manifold. In the following simulations, we generate noisy samples from the $4$-sphere $S^4$ according to the following model. For a scalar $\sigma>0$, we define the $\sigma$-tubular neighborhood of $S^4$ by 
\begin{equation}
	S^4_\sigma = \left\{x \; :\; \min_{y\in S^4}\norm{x-y}<\sigma\right\}. 
\end{equation}
The set $S^4_\sigma$ is simply a spherical shell of width $2\sigma$ in $\mathbb{R}^5$. A noisy sample of $S^4$ is generated by drawing points uniformly from $S^4_\sigma$ for some fixed $\sigma >0$. Thus, the parameter $\sigma$ controls the amount of noise in the data set. 
We generate a data set~$X=\left\{x_1,\ldots,x_N\right\}$ by drawing~$N$ points $\left\{p_1,\ldots,p_N\right\}\in S^4_\sigma$, and then mapping each point~$p_i$ to a point $x_i\in\mathbb{C}^3$ by using the map~\eqref{secConvSim:embed}.

To apply our framework to denoise the data set $X$, we consider the action of the group $SU(2)$ on $X$ defined in~\eqref{secConvSim:toyExampleGroupAction}.
Using the notation in~\eqref{secConvSim:embed} and~\eqref{secConvSim:toyExampleGroupAction}, we define the functions
\begin{equation}\label{secDenoise:crdFuncDef}
	F_1(i,A) = \left(U(A)\cdot x_i\right)_1, \quad F_2(i,A) = \left(U(A)\cdot x_i\right)_2, \quad F_3(i,A)= x_{i,3},
\end{equation}
for all $A\in SU(2)$ and $i\in \left\{1,\ldots,N\right\}$, where $\left(\cdot\right)_1$ and $\left(\cdot\right)_2$ denote the first and second elements of a vector in $\mathbb{C}^3$. 
Clearly, we have that $	F_1,F_2$, and~$F_3$ are all elements of the Hilbert space $\mathcal{H} = L^2\left\{\left\{1,\ldots,N\right\}\times SU(2)\right\}$. 
For each $k\in \left\{1,2,3\right\}$, the function $F_k(i,\cdot):G\rightarrow \man$ is the $k$-th' coordinate of the points in the orbit $G\cdot x_i$, and thus $F_k$ is the $k$-th coordinate function of the points in $G\cdot X$. Denote by $S^4_{\mathbb{C}}$ the embedding of $S^4$ in $\mathbb{C}^3$ by the map in~\eqref{secConvSim:embed}. 
Thus, the function $F_k$ attains the values of the $k$-th coordinate of $S^4_{\mathbb{C}}$ sampled at the points in $G\cdot X$. 

We now denoise the data set $X$ as follows. First, we construct the normalized normalized $SU(2)$-GL by using the points in the data set~$X$, and compute its eigenfunctions $\left\{\Phi_{\ell,m,n}\right\}$ given by \eqref{GInvLapProp:EigenFuncForm}, as described by Theorem \ref{GInvLapProp:Thrm1}. We choose the bandwidth parameter~$\epsilon$ so as to make the matrices~$\hat{W}^\ell$ in~\eqref{sec2:blockFourierMat} sparse. Specifically, for a data set of~$N=5000$ points, we first subsample~$50$ points and sort the elements in each of the rows of~$\hat{W}^{(0)}$ (which is real valued) corresponding to those points in descending order. The bandwidth is then chosen such that the values of the sorted elements in each row decay exponentially fast, and such that the index of the elbow of the scree plot  of values in each row (defined as the first point where the derivative equals~$\approx -1$) is~$<250$ (which is~$5\%$ of the values). 
We then expand each of the functions of~\eqref{secDenoise:crdFuncDef} in terms of the eigenfunctions $\left\{\Phi_{\ell,m,n}\right\}$, and truncate the expansion.
A standard approach is to retain the terms in the expansion that correspond to eigenvalues $\lambda_{n,\ell}$ above some threshold value. However, we truncate the expansion using the following observation. The $4$-sphere can be completely recovered using the five eigenfunctions that correspond to the second leading eigenvalue of the Laplacian operator $\Delta_{S^4}$. This is simply due to the fact that the coordinate functions $h_1,\ldots,h_5$ defined for each~$p_i\in S^4$ by~$h_j(p_i) = p_{i,j}$, span the eigenspace that corresponds to the second smallest eigenvalue of~$\Delta_{S^4}$~\cite{harmFunTheoryAxler}. 
Thus, we expect that the functions in~\eqref{secDenoise:crdFuncDef} should be well approximated by the space spanned by the eigenfunctions corresponding to the five smallest eigenvalues of the normalized $SU(2)$-GL after excluding the smallest eigenvalue. This suggests retaining only the terms corresponding to the latter eigenfunctions in the expansion of each coordinate function in~\eqref{secDenoise:crdFuncDef}. Finally, for each $i\in\left\{1,2,3\right\}$ let $\tilde{F_i}\in\mathbb{C}^N$ denote the vector of values of the truncated expansion (just described) of the function $F_i$ of \eqref{secDenoise:crdFuncDef} at the points~$(j,I)$ for all $j\in\left\{1,\ldots,N\right\}$. The denoised data points $\tilde{x}_1,\ldots,\tilde{x}_N$ are then given by
\begin{equation}\normalsize
	\begin{pmatrix}
		-\:\tilde{x}_1\;-\\-\;\tilde{x}_2\;- \\ \vdots \\ -\;\tilde{x}_N\;-
	\end{pmatrix}= 
\begin{pmatrix}
	|&|&|&|&|\\
	\\
	\text{Re}\left\{\tilde{F}_1\right\}& \text{Im}\left\{\tilde{F}_1\right\}& \text{Re}\left\{\tilde{F}_2\right\}& \text{Im}\left\{\tilde{F}_2\right\}& \text{Re}\left\{\tilde{F}_3\right\}\\\\
	|&|&|&|&|
\end{pmatrix}.
\end{equation}

The denoising results of $N=5000$ points sampled from $S_\sigma^4$ for various values of~$\sigma$ are presented in Figure~\ref{fig:errorStatTable}. Defining the error of approximation of each noisy point~$x_i$ as the distance
\begin{equation}
	d_i = \min_{y\in S^4}\norm{x_i-y},
\end{equation}
for each value of $\sigma$, we report the mean squared error (MSE) of the approximation obtained by preforming our proposed denoising procedure using the normalized $SU(2)$-GL. For comparison, we also report the MSE for the same data sets denoised by the eigenvectors of the normalized standard~GL. Denoising using the normalized standard~GL is implemented by viewing each column $H_i$ of the matrix 
\begin{equation}
		\begin{pmatrix}
		| &  & |\\	
		H_1 & \cdots & H_5 \\
		| &  &  |\\
	\end{pmatrix} = 
	\begin{pmatrix}
		-\; x_1\;-\\-\; x_2\;- \\ \vdots \\ -\; x_N \;- 
	\end{pmatrix}
\end{equation}
formed by stacking the data points in rows, as a sample of a coordinate function on~$\man$, and projecting~$H_i$ on the eigenvectors that correspond to the five smallest eigenvalues of the standard GL, after excluding the smallest one.
We observe that for moderate noise levels $\sigma = 0.1,0.2$, denoising using the normalized $SU(2)$-GL outperforms denoising using the normalized standard-GL by an order of magnitude, recovering the 4-sphere with high accuracy. 
\begin{table}
		\centering
		\small
		\begin{tabular}{|c|c|c|c|c|}
			\hline
			$\sigma$	& noisy data MSE & standard GL denoised data MSE & $SU(2)$-GL denoised data MSE\\ \hline
			0.1	& 3.3E-03 &		9.3E-04 &	5.04E-05 \\ \hline
			0.2 &	1.33E-02 &	3.11E-03 &	3.30E-04 \\ \hline
			0.4 &	5.33E-02 &	1.745E-02 &	1.6E-02 \\ \hline
		\end{tabular}
		\caption{MSE of noisy data before and after denoising.}
		\label{fig:errorStatTable}	
\end{table}

\section{Implementation details and computational complexity}\label{numericsSec}
In this section, we describe a numerical procedure to compute the eigendecomposition of the~$G$-invariant graph Laplacian in the case where~$G=SU(2)$. We point out that almost all of our analysis can be readily generalized to the case where~$G$ is an arbitrary compact matrix Lie group, and we restrict ourselves to the case~$G=SU(2)$ whose representation theory is well understood, for the sake of clarity and concreteness. In particular, the important case where~$G=SO(3)$ is nearly identical to that of~$G=SU(2)$ since the IUR's of $SO(3)$ are a subset of those of~$SU(2)$.

With the exception of $SO(2)$ and the 2-dimensional torus $\mathbb{T}^2$, the dimension of a matrix Lie group is $\geq3$. Thus, even for a low-dimensional group such as $SU(2)$, the integrals in~\eqref{sec2:hat{W}Def}, required to construct the matrices \eqref{sec2:blockFourierMat}, need to be evaluated by  triple sums. Such sums are computationally expensive even for moderate values of $N$.  Fortunately, for groups such as $SU(2)$ (and the closely related $SO(3)$) there exist generalized FFT algorithms that compute the Fourier coefficients efficiently \cite{fastNFTSO3Potts}. 

The general approach for numerical integration over $SU(2)$ hinges upon the fact 
that the elements of its IURs can be parameterized by Euler angles, and written in a separable form as a product of factors, each of which depends on a single angle. The integrals are then evaluated using quadrature formulas that are computed using FFT-type algorithms applied to each factor seperately, requiring $O(\tilde{K}\log^2 \tilde{K})$ operations where $\tilde{K}$ is a prescribed sampling resolution over the group. We give a detailed exposition of an $SU(2)$-FFT in~Appendix~\ref{appnedix:numerIntegration} below. 

We now continue to describe analyze the complexity of computing the eigendecomposition presented in Theorem~\ref{GInvLapProp:Thrm1} for the case where~$G=SU(2)$ acts on a data set $\left\{x_1,\ldots,x_N\right\}\in\man \subset \mathbb{C}^\mathcal{D}$ by matrix multiplication. 
The first step of the algorithm requires computing the affinities~$W_{ij}$ in~\eqref{GinvDef:Wdef} at~$O(\tilde{K})$ sampling points, and in particular, the Euclidean pairwise distances inside each exponent.
In practice, the matrices $A\in G$ are usually block-diagonal where each block is an IUR of~$SU(2)$ (see e.g. \cite{momAbinito,steerMaps}). Formally, we write 
\begin{equation}\label{secDenoise:UmapDef}
	A = \text{diag}(U^{\ell_1}(A),\ldots, U^{\ell_{S}}(A)), \quad \ell_j\in \I_{\man},\quad j=1\ldots,S,
\end{equation}
where $U^{\ell_j}$ is the $\ell_j$-th dimensional IUR of $SU(2)$, and $\I_{\man}$ is the set of IURs that appear as blocks on the diagonal of $A$, such that $\ell_j\leq \ell_{j+1}$. Note that some of the IURs may appear more than once on the diagonal. 
Accordingly, we can now index the coordinates of a point $x_i$ in the data set to match the indices of the rows of the IURs in the blocks of $A$, by 
\begin{equation}\label{secDenoise:xCoordinates}
	x_i = \left(x_{i,(\ell_j,m)}\right), \quad \quad m=-\ell_j,\ldots,\ell_j, \quad \ell_j\in\I_{\man}.
\end{equation}
That is, the indexing \eqref{secDenoise:xCoordinates} partitions $x_i$ into $\#\left\{\I_\man\right\}$ tuples of length~$(2\ell_j+1)$ such that the action of $SU(2)$ on~$x\in\man$ can be written as
\begin{equation}\label{secDenoise:CoeffAction}
	(A\cdot x_i)_{l_j,m} = \sum_{r=-\ell_j}^{\ell_j} U_{m,r}^{\ell_j}(A)\cdot x_{i,(\ell_j,r)},
\end{equation}
for each $\ell_j$ and $m$ in \eqref{secDenoise:xCoordinates}. Altogether, in matrix form, we have that 
\begin{equation}\label{secDenoise:xCoordinatesMatForm}
	A\cdot x_i = \begin{pmatrix}
		U^{\ell_1}(A)&&&&\\ &\ddots&&&\\ &&\ddots&&\\ &&&\ddots&\\ &&&& U^{\ell_{S}}(A)
	\end{pmatrix}\cdot 
\begin{pmatrix}
	x_{i,-\ell_1}\\ \vdots \\ x_{i,\ell_1}\\ \vdots \\x_{i,-\ell_S}\\ \vdots \\x_{i,\ell_S}
\end{pmatrix}.
\end{equation}

To compute the matrices in~\eqref{sec2:hat{W}Def}, we must first evaluate the Euclidean distances 
\begin{equation}\label{su2Example:distsAllOribts}
	\norm{x_i-A\cdot x_j},\quad  A\in G, \quad i,j\in\left\{1,\ldots,N\right\}.
\end{equation}
Expanding the squared norm function, we have
\begin{align}\label{su2Example:distTermByTerm}
	\lVert x_i - A\cdot x_j \rVert^2 = \norm{x_i}^2+ \norm{x_j}^2-2\text{Re}\left\{\dprod{x_i}{ A\cdot x_j}\right\}.
\end{align}
Then, expanding the inner product in the third term on the right hand side of~\eqref{su2Example:distTermByTerm}, we get 
\begin{align}\label{secDenoise:distThirdTerm}
	\dprod{x_i}{A\cdot x_j} &= \sum_{\ell\in \I_{\man}}\sum_{m=-\ell}^{\ell}\sum_{\left\{k: \ell_k=\ell\right\}}x_{i,(\ell_k,m)}\sum_{r=-\ell}^{\ell}U^\ell_{mr}(A)\cdot x_{j,(\ell_k,r)}\\ \nonumber
	& = \sum_{\ell\in \I_{\man}}\sum_{m,r=-\ell}^{\ell}\sum_{\left\{k: \ell_k=\ell\right\}}c_{(i,j),(\ell,m,r)}\cdot U^\ell_{mr}(A), 
\end{align}
where we denote
\begin{equation}\label{su2Example:cellmj}
	c_{(i,j),(\ell,m,r)} = \sum_{\left\{k: \ell_k=\ell\right\}}x_{i,(\ell_k,m)}\cdot x_{j,(\ell_k,r)}.
\end{equation}

Given an integration parameter~$K$, we compute~\eqref{secDenoise:distThirdTerm} and subsequently~\eqref{su2Example:distTermByTerm} for all matrices~$A_{k1,k_2,k_3}$ defined by using~\eqref{secLieGroupAction:fundIUR} and~\eqref{secDenoise:UmapDef} as 
\begin{equation}\label{secDenoise:SU2DiscreteDef}
	A_{k_1,k_2,k_3}  \coloneq \text{diag}(U^{\ell_1}(A(\pi k_1/K,\pi k_2/K,\pi k_3/K),\ldots, U^{\ell_{S}}(A(\pi k_1/K,\pi k_2/K,\pi k_3/K))),
\end{equation}
where~$k_1=0,\ldots,K-1$, and~$k_2 = 0,\ldots,2K-1$, and~$k_3 = -2K,\ldots,2K-1$.
Once we have computed the coefficients $c_{(i,j),(\ell,m,r)}$, the third term in \eqref{su2Example:distTermByTerm} can be computed for all $A_{k_1,k_2,k_3}$ with $O(K^3\log^2 K)$ operations by using a generalized FFT algorithm for $SU(2)$ (see~Appendix~\ref{appnedix:numerIntegration}). 
Now, since the $\ell$-th IUR consists of $(2\ell+1)^2$ elements, the number of coefficients $c_{(i,j),(\ell,m,r)}$ that need to be computed for a fixed pair~$i$ and~$j$ amounts to
\begin{equation}\label{su2Example:cijComplex}
	\sum_{\ell\in\I_{\man}}\sum_{k:\ell_k=\ell} (2\ell_k+1)^2.
\end{equation}
By \eqref{secDenoise:xCoordinates} and \eqref{secDenoise:xCoordinatesMatForm}, we have
\begin{equation}\label{su2Example:cijComplexBound}
	\sum_{\ell\in\I_{\man}}\sum_{k:\ell_k=\ell} (2\ell_k+1) = n,
\end{equation}
where~$n$ is the dimension of the points $x_i$. Since \eqref{su2Example:cijComplex} is bounded from above by the square of \eqref{su2Example:cijComplexBound}, we have that \eqref{su2Example:cijComplex} is $O(n^2)$. 
Finally, once we have computed the squared distances~\eqref{su2Example:distTermByTerm}, we use Algorithm~\ref{alg:integrateSU(2)} to compute the elements of the matrices~$S^{\ell}$ in~\eqref{GinvProp:normFourierMat}, and compute their eigenvectors and eigenvalues. The entire procedure is described in~Algorithm~\ref{alg:Evaluating the steerable manifold harmonics}. 

\begin{algorithm}
	\caption{Evaluating the $G$-invariant manifold harmonics}\label{alg:Evaluating the steerable manifold harmonics}
	\begin{algorithmic}[1]
		\Statex{\textbf{Input:} A data set of $N$ points $\left\{x_1,\ldots,x_N\right\}\subset\CD$, integration parameter $K$, and bandwidth parameter~$\epsilon$}.
		\State For every $i,j\in\left\{1,\ldots,N\right\}$, apply Algorithm~\ref{alg:integrateSU(2)} with integration parameter~$K$, in conjunction with \eqref{su2Example:distTermByTerm} and \eqref{secDenoise:distThirdTerm} to compute the affinities
		\begin{equation}\label{algSteerHarm:WijDiscrete}
			{W}_{ij}(I,A_{k_1,k_2,k_3}) = \exp{\left\lbrace-{\left\Vert x_{i} - A_{k_1,k_2,k_3}\cdot x_{j}  \right\Vert^2 }{/\epsilon}\right\rbrace},
		\end{equation}
		where $A_{k_1,k_2,k_3}$ is defined in~\eqref{secDenoise:SU2DiscreteDef}.
		\State For every $i,j\in\left\{1,\ldots,N\right\}$ and $\ell\in \left\{0,\ldots,K-1\right\}$, apply Algorithm~\ref{alg:integrateSU(2)} to evaluate the generalized Fourier coefficient matrices $\hat{W}_{ij}^{\ell}$ of \eqref{sec2:hat{W}Def}.  
		\State For every $\ell\in\left\{0,\ldots,K-1\right\}$ form the matrix
		\begin{equation}\label{algSteerHarm:matSeq}
			\tilde{S}_\ell = I-\left(D^\ell\right)^{-1}\hat{W}^{\ell},
		\end{equation}
		from \eqref{GinvProp:normFourierMat}, and return its eigenvectors $\left\{\tilde{v}_{n,\ell}\right\}_{n=1}^N$ and eigenvalues $\left\{\tilde{\lambda}_{n,\ell}\right\}_{n=1}^N$.
	\end{algorithmic}
\end{algorithm}

We now summarize the computational complexity of Algorithm \ref{alg:Evaluating the steerable manifold harmonics}. Given that we evaluate the $SU(2)$ Fourier series over $O(K)$ points for each Euler angle, the sampling resolution of $SU(2)$ amounts to $O(K^3)$ points. Denoting $\tilde{K}=O(K^3)$, computing the distances in \eqref{algSteerHarm:WijDiscrete} requires $O(N^2\tilde{K}\log^2 \tilde{K} + N^2n^2)$ operations, out of which $O(N^2n^2)$ operations are required to compute the coefficients $c_{(i,j),(\ell,m,r)}$, and $O(N^2\tilde{K}\log^2 \tilde{K})$ operations to compute \eqref{secDenoise:distThirdTerm} using a fast polynomial transform based $SU(2)$-FFT. 
Forming the generalized Fourier coefficients matrices $\hat{W}^\ell$ of~\eqref{sec2:hat{W}Def} when using a $SU(2)$-FFT requires $O(N^2\tilde{K}\log^2 \tilde{K})$ operations. Forming the sequence of matrices~\eqref{algSteerHarm:matSeq} in the last step of Algorithm~\ref{alg:Evaluating the steerable manifold harmonics} requires $O(N^2K)$ operations, and evaluating the eigenvalues and eigenfunctions of~\eqref{algSteerHarm:matSeq} requires additional~$O(N^3K)$ operations. Thus, the computational complexity of Algorithm~\ref{alg:Evaluating the steerable manifold harmonics} amounts to $O(N^3K+N^2n^2+N^2\tilde{K}\log^2\tilde{K})$ operations in total. 

\section{Summary and future work}\label{SecSummary}
In this work, we extended the graph Laplacian to data sets that are closed under the action of a matrix Lie group. To that end, we introduced the $G$-invariant graph Laplacian (the $G$-GL), that incorporates the group action into its construction, by considering the pairwise distances between all points generated by applying the group action to the given data set. We have shown that the $G$-GL converges to the Laplace-Beltrami operator $\Delta_\man$, at a rate accelerated proportionally to the dimension of the group. This accelerated rate implies that it is advantageous to employ the $G$-GL for graph Laplacian based methods \cite{diffMaps,lapMap,manReg} whenever the data set is equipped with a known group action, since faster convergence implies that significantly less data is required for a prescribed accuracy. We also derived the eigendecomposition of the $G$-GL, showing that its eigenfunctions have a separable form, where the dependence on the group is expressed analytically using the irreducible unitary representations of the group. 
We then demonstrated how the $G$-GL can be employed to denoise a noisy sample from the 4-sphere $S^4$ by using a discrete Fourier analysis type algorithm, with the Fourier modes replaced by the eigenfunctions of the~$G$-GL. 

As of future research, an important direction is to investigate the spectral convergence (see \cite{CHENG2022132}) of the $G$-GL, that is, the convergence of its eigenvectors and eigenvalues to those of~$\Delta_{\man}$. Another, could be to further develop applications of the $G$-GL, e.g., in electron-microscopy imaging \cite{Frank}.

\section*{Acknowledgements}
PH and JK were supported in part by NSF Award DMS-2309782 and start-up grants provided by the College of Natural Sciences and Oden Institute for Computational Engineering and Sciences at the University of Texas at Austin. XC was supported in part by NSF-BSF award 2019733. ER and YS were supported by NSF-BSF award 2019733 and by the European Research Council (ERC) under the European Union's Horizon 2020 research and innovation programme (grant agreement 723991 - CRYOMATH). YS was supported also by the NIH/NIGMS Award R01GM136780-01.

\appendix
\section{The FFT over $SU(2)$}\label{appnedix:numerIntegration}
We now describe how to efficiently compute the Fourier series of a function defined over the group $SU(2)$ whose elements are given by~\eqref{secLieGroupAction:fundIUR}.
The explicit form of the series is given in~\eqref{harmAnalysisOnG:SO3Fourier}, with the IURs of~$SU(2)$ enumerated by the set of non-negative half integers~$\I_{SU(2)}=\left\{0,1/2,1,3/2,\ldots\right\}$, and $d_\ell = 2\ell+1$ for all~$\ell\in \I_{SU(2)}$. 

Recall that the Fourier series of a function over a Lie group $G$ is given by the elements of its IURs. 
The elements of the IURs of $SU(2)$ are given by (see \cite{specialFunc})
\begin{equation}\label{secSU2Quad:SU2IURs}
\begin{gathered}
	U^\ell_{mn}(\alpha,\beta,\gamma) = e^{-im\alpha}P_{mn}^\ell(\cos \beta)e^{-in\gamma}, \\ 
 m,n\in \left\{-\ell,\cdots,\ell\right\},\quad  \ell = 0,\frac{1}{2},1,\frac{3}{2},\ldots,
 \end{gathered}
\end{equation}
with $P_{mn}^\ell(\cos \beta)$ given by
\begin{equation}
	P_{mn}^\ell(\cos \beta) = \bigg[ \frac{(\ell-m)!(\ell+m)!}{(\ell-n)!(\ell+n)!} \bigg]^{\frac{1}{2}}\sin^{m-n}\left(\frac{\beta}{2}\right)\cos^{m+n}\left(\frac{\beta}{2}\right)P_{\ell-m}^{(m-n,m+n)}(\cos\beta),
\end{equation}
where  $P^{(a,b)}_r$ are the Jacobi polynomials (see \cite{specialFunc}).
Now, by~\eqref{secLieGroupAction:SU2IntEuler} and~\eqref{harmAnalysisOnG:SO3Fourier}, the generalized Fourier coefficients of a function~$f:SU(2)\rightarrow \mathbb{C}$ are given by 
\begin{equation}\label{eq:SU2 Fourier coefficient}
\hat{f}_{mn}^\ell=	\frac{1}{16\pi^2}\int_0^{2\pi}\int_0^{\pi}\int_{-2\pi}^{2\pi}f(B(\alpha,\beta,\gamma))\overline {U^\ell_{mn}(\alpha,\beta,\gamma)}\sin\beta d\alpha d\beta d\gamma.
\end{equation}
These coefficients can be approximated rapidly to an arbitrary accuracy as we now describe. 
Note that the functions $U^\ell_{mn}(\alpha,\beta,\gamma)$ in $\eqref{secSU2Quad:SU2IURs}$ separate into a product of factors each depending on a single angle $\alpha$, $\beta$, or $\gamma$. Thus, \eqref{eq:SU2 Fourier coefficient} can be computed by integrating over each of the angles successively, as we now show. 

Set a bandlimit $L\geq 0$ depending on the required accuracy, and an integration parameter~$K>2L$. We begin by evaluating the integrals
\begin{equation}\label{secSU2Quad:fn}
	\tilde{f}_{m}(\beta,\gamma) = \int_{0}^{2\pi}f(B(\alpha,\beta,\gamma))e^{im\alpha}d\alpha, 
\end{equation}
of $f$ multiplied by the conjugate of the factor depending on $\alpha$ in \eqref{secSU2Quad:SU2IURs}
for each $m\in \left\{-L,\ldots,L\right\}$, all~$\gamma \in \left\{ -2\pi,-2\pi+2\pi/K,\ldots,2\pi(K-1)/K  \right\}$, and all~$\beta\in\left\{\arccos(y_k)\right\}_{k=1}^M$, where~$y_k$ are the Gauss-Legendre quadrature nodes for some~$M=O(K)$ (the reason for this choice of $\beta$s will become apparent shortly), by 
\begin{equation}
	\tilde{f}_{m}(\beta,\gamma) \approx \tilde{f}_{\left[m\right]}(\beta,\gamma) =  \frac{2\pi}{K}\sum_{k=0}^{K-1} f \left ( B\left ( 2\pi k/K,\beta,\gamma \right ) \right )e^{2\pi i m k/K}.
\end{equation} 
Using $O(K^2)$ applications of the classical FFT, the entire computation is accomplished by using $O(K^3\log K)$ operations. Next, we evaluate the integrals 
\begin{equation}\label{secSU2Quad:fnm}
	\dbtilde{f}_{mn}(\beta) = \int_{-2\pi}^{2\pi} \tilde{f}_{m}(\beta,\gamma)e^{i n \gamma}d\gamma=\int_{-2\pi}^{2\pi}\int_{0}^{2\pi} f(\alpha,\beta,\gamma)e^{im\alpha}e^{in\gamma}d\alpha \, d\gamma,
\end{equation}
of $f$ multiplied by the conjugate of the factors in~\eqref{secSU2Quad:SU2IURs} that depend on $\alpha$ and $\gamma$, for each $n \in \left\{-L,\ldots,L\right\}$, and all values of $\beta$ and $m$ used in the previous computation, by
\begin{equation}
	\dbtilde{f}_{mn}(\beta) \approx \dbtilde{f}_{\left[m\right]\left[n\right]}(\beta) =  \frac{4\pi}{K}\sum_{k=0}^{K-1} \tilde{f}_{\left[m\right]}(\beta,2\pi k/K)e^{2\pi i n k/K}.
\end{equation}
Using $O(K^2)$ applications of the FFT, the latter computation amounts to a total of $O(K^3\log K)$ operations.
\begin{algorithm}
	\caption{$SU(2)$-FFT}\label{alg:integrateSU(2)}
	\begin{algorithmic}[1]
		\Statex{\textbf{Input:}}
		\begin{enumerate}
			\item Integration parameter $K$.
			\item Function $f:SU(2)\rightarrow \mathbb{C}$. 
			\item Precomputed weights $\left\{w_0,\ldots,w_{M}\right\}$ and nodes~$\left\{y_1,\ldots, y_M\right\}$ for Gauss-Legendre quadrature.
		\end{enumerate}			 
		\For{$\ell\in\left\{0,\ldots,L \right\}$}
		\For{$m\in\left\{-\ell,\ldots,\ell \right\} $}
		\For{$n\in\left\{-\ell,\ldots,\ell \right\}$}	
		
		\For{$\beta\in \left\{\arccos(y_0),\ldots,\arccos(y_{M})\right\}$}	
		\For{$\gamma\in \left\{-2\pi ,-2\pi+\frac{2\pi}{K}\ldots,\frac{2\pi (K-1)}{K}\right\}$}			
		\begin{equation}
			\tilde{f}_{[m]}(\beta,\gamma) = \frac{2\pi}{K}\sum_{k=0}^{K-1} f\left(\frac{2\pi k}{K},\beta,\gamma\right)e^{i2\pi m k/K}, 
		\end{equation}
		\EndFor 
			\begin{equation}
			\dbtilde{f}_{\left[m\right]\left[n\right]}(\beta) = \frac{4\pi}{K}\sum_{k=-K}^{K-1} \tilde{f}_{[m]}\left(\beta,\frac{2\pi k}{K}\right)e^{i2\pi n k/K}
		\end{equation}
		\EndFor 
		\begin{equation}
			\hat{f}_{\left[m\right]\left[n\right]}^{\left[\ell\right]} = \frac{1}{16\pi^2}\sum_{k=0}^{M} w_{k} \cdot \dbtilde{f}_{\left[m\right]\left[n\right]}\left(\arccos(y_k)\right)P^{\ell}_{mn}\left( y_k\right)
		\end{equation}
		\EndFor
		\EndFor		
		\EndFor
		\Ensure{The generalized Fourier coefficients $\hat{f}_{\left[m\right]\left[n\right]}^{\left[\ell\right]}$}.
	\end{algorithmic}
\end{algorithm}
Lastly, we evaluate 
\begin{align}\label{secSU2Quad:fHatnm}
	\hat{f}_{mn}^\ell &= \frac{1}{16\pi^2}\int_{-2\pi}^{2\pi}\int_0^{\pi}\int_{0}^{2\pi}f(B(\alpha,\beta,\gamma))\overline {U^\ell_{mn}(\alpha,\beta,\gamma)}\sin\beta \, d\alpha \, d\beta \, d\gamma, \nonumber \\ 
	&= \frac{1}{16\pi^2}\int_0^\pi 	\dbtilde{f}_{mn}(\beta)P_{mn}^\ell(\cos\beta)\sin\beta \, d\beta 
	=  \frac{1}{16\pi^2}\int_{-1}^1 	\dbtilde{f}_{mn}(\arccos(y))P_{mn}^\ell(y) \, dy
\end{align}
for each $\ell \in \left\{0,\ldots,L\right\}$, and all $m$ and $n$ from the previous computation, by
\begin{equation}\label{secSU2Quad:fhatnm}
	\hat{f}_{mn}^\ell \approx \hat{f}_{\left[m\right]\left[n\right]}^{\left[\ell\right]} = \frac{1}{16\pi^2}\sum_{k=1}^{M} w_k\cdot \dbtilde{f}_{\left[m\right]\left[n\right]}(\arccos(y_k))P_{mn}^\ell(y_{k}),
\end{equation}
using Gauss-Legendre quadrature with precomputed weights $w_1,\ldots,w_{M}$. The latter computation is accomplished using $O(K^3)$ direct evaluations of size $O(K)$, amounting to a total complexity of $O(K^4)$ operations. 

We point out that \eqref{secSU2Quad:fHatnm} can also be computed using $O(K^3\log^2K)$ operations by applying fast polynomial transforms (see e.g. \cite{fastPTPotts}), bringing the overall complexity of the entire algorithm to $O(K^3\log^2 K)$ operations. However, after some experimentation, we found that while the direct computation of~\eqref{secSU2Quad:fhatnm} is asymptotically more expensive, in practice, utilizing GPUs to evaluate it is substantially faster than the available $O(K^3\log^2K)$ algorithms. Unfortunately, utilizing GPUs does not easily lend itself for speeding up fast polynomial transform algorithms, due to their iterative nature. The entire procedure of evaluating the integrals in~\eqref{eq:SU2 Fourier coefficient} is outlined in Algorithm~\ref{alg:integrateSU(2)}.

Lastly, we note that the method described above can be applied to $SO(3)$ by restricting all computations to the integer valued IURs of $SU(2)$, and the angle~$\gamma$ to~$[0,2\pi)$. 

\section{Proof of Lemma~\ref{GinvDef:quadFormLemma}}\label{quadFormLemmaPrf}
For any~$f\in\Hspace$, expanding~\eqref{GinvDef:Ldef} by using~\eqref{GinvDef:Wdef} and~\eqref{GinvDef:Ddef}, we obtain that
\begin{equation}\label{GinvDef:Ldef2}
	\left\{Lf\right\}(i,A) =D_{ii}\cdot f_i(A)-\sum_{i=1}^{N}\int_{G} W_{ij}(A,B)f_j(B)d\eta(B), \quad (i,A)\in \Gamma,
\end{equation}
which implies that
\begin{equation}\label{quadFormLemmaPrf:eq1}
	\dprod{f}{Lf}_{\Hspace} = \sum_{i=1}^N D_{ii}\cdot \int_G \left|f_i(A)\right|^2d\eta(A)-\sum_{i,j=1}^N\int_G \int_G W_{ij}\overline{f_i}(A)f_j(B) d\eta(A)d\eta(B). 
\end{equation}
Next, by using the left-invariance property~\eqref{secLieGroupAction:haarIntLeftInvar} of~$\eta$ and \eqref{GinvSec:affinityKernelDef}, for any~$i,j\in\left\{1,\ldots,N\right\}$ and~$A\in G$ we have that
\begin{equation}
	\int_G W_{ij}(I,C)d\eta(C) = \int_G W_{ij}(I,C)d\eta(AC) = \int_G W_{ij}(I,A^*B)d\eta(B) = \int_G W_{ij}(A,B)d\eta(B), 
\end{equation}
where we made the change of variables~$B=AC$ in the second equality. 
Thus, using that~$W_{ij}(A,B) = W_{ji}(B,A)$ (by the definition of $W_{ij}$ in~\eqref{GinvDef:Wdef}), we can write the first expression on the r.h.s of~\eqref{quadFormLemmaPrf:eq1} as
\begin{align}\label{quadFormLemmaPrf:eq2}
	\sum_{i=1}^N D_{ii}\cdot \int_G \left|f_i(A)\right|^2d\eta(A) &= \sum_{i,j=1}^N\int_G\int_G W_{ij}(A,B)\left|f_i(A)\right|^2d\eta(A) d\eta (B) \nonumber \\
	&= \sum_{i,j=1}^N\int_G\int_G W_{ij}(A,B)\left|f_j(B)\right|^2d\eta(A) d\eta (B).	
\end{align}
Plugging~\eqref{quadFormLemmaPrf:eq2} into~\eqref{quadFormLemmaPrf:eq1} we obtain that
\begin{align}
	\dprod{f}{Lf}_{\Hspace} &= \frac{1}{2}\sum_{i,j=1}^N \int_G\int_G W_{ij}(A,B)\Big[\left|f_i(A)\right|^2+\left|f_j(B)\right|^2-f_i(A)\overline{f_j}(B)-\overline{f_i(A)}f_j(B)\Big] d\eta(A)d\eta(B)\nonumber \\
	&=  \frac{1}{2}\sum_{i,j=1}^N \int_G\int_G W_{ij}(A,B)\left|f_i(A)-f_j(B)\right|^2 d\eta(A)d\eta(B).
\end{align}

\section{Proof of Theorem \ref{GInvLapProp:Thrm1}}\label{eigDecomPrf}
	For any $\Psi\in \mathcal{H}$, by plugging \eqref{sec1:fourierSO3} into \eqref{GinvDef:Wdef}, and using~\eqref{harmAnalysisOnG:SO3Fourier} we have for any~$i\in\{1,\ldots,N\}$ that
	\begin{align}
		\left\{W\Psi\right\}(i,A)= \sum_{j=1}^N\int_G \sum_{\ell\in\I_G} d_\ell\cdot\sum_{m,n=1}^{d_\ell}\left(\hat{W}^\ell_{ij}\right)_{mn}U^\ell_{mn}\left(A^*B\right) \Psi_j(B)d\eta(B),
	\end{align}
	where we denote by $(\hat{W}^\ell_{ij})_{mn}$ and $U^\ell_{mn}(\cdot)$ the $(m,n)^{\text{th}}$ entries of $\hat{W}^\ell_{ij}$ and $U^\ell(\cdot)$, respectively, and $\I_G$ enumerates the IURs of~$G$. Next, by using the homomorphism property of group representations~\eqref{harmAnalysisOnG:hMorphProp}, we have 
	\begin{align}\label{eigdecompProof:afterDecoupling}
		\left\{W\Psi\right\}(i,A)= \sum_{j=1}^N \sum_{\ell\in \I_G} d_\ell\cdot\sum_{m,n=1}^{d_\ell}\left(\hat{W}^\ell_{ij}\right)_{mn}\sum_{r=1}^{d_\ell}U^\ell_{mr}\left(A^*\right)\int_GU^\ell_{rn}\left(B\right) \Psi_j(B)d\eta(B).
	\end{align}
	We now show that for a given $q\in \I_G$, $p\in\left\{1,\ldots,d_\ell\right\}$ and $s\in\{1,\ldots,N\}$ the function~$\Phi_{q,p,s}$ of \eqref{GInvLapProp:EigenFuncForm} is an eigenfunction of~$L$. 	

	Extending the notation of \eqref{sec1:parVecNotation}, for any $v\in \mathbb{C}^{N\ell}$ and all $j\in\{1,\ldots,N\}$, we denote the $d_\ell$ entries of the vector $e^j(v)\in \mathbb{C}^{d_\ell}$ by
	\begin{equation}\label{eigdecompProof:ejellNotation}
		e^j(v)=\left(e^j_{1}(v),\ldots,e^j_{d_\ell}(v)\right).
	\end{equation}		
	Now, the homomorphism property~\eqref{harmAnalysisOnG:hMorphProp} implies that $U^q_{\cdot,p}\left(A^*\right) =\overline{U^q_{p,\cdot}\left(A\right)}$. Thus, plugging~$\Psi=\Phi_{q,p,s}$ into \eqref{eigdecompProof:afterDecoupling}, and using the notation in \eqref{eigdecompProof:ejellNotation}, and that 
	 \begin{equation}\label{eigdecompProof:PhiDef}
	 	\Psi_j(A)= \Phi_{q,p,s}(j,A) = e^j\left(v_s\right)U^q_{\cdot,p}(A^*),
	 \end{equation}
 	the expression for $\left\{W \Phi_{q,p,s}\right\}(i,A)$ is given by
	\begin{align}\label{eigdecompProof:afterPluggingPhi}
		&\sum_{j=1}^N \sum_{\ell\in\I_G} d_\ell\cdot\sum_{m,n=1}^{d_\ell}\left(\hat{W}^\ell_{ij}\right)_{mn}\sum_{r=1}^{d_\ell}U^\ell_{mr}\left(A^*\right)\int_GU^\ell_{rn}\left(B\right)\sum_{k=1}^{d_\ell} e^j_k\left(v_s\right)\overline{U^q_{p,k}\left(B\right)}d\eta(B) \\ \nonumber
		&= \sum_{j=1}^N \sum_{\ell\in\I_G} d_\ell\cdot\sum_{m,n=1}^{d_\ell}\left(\hat{W}^\ell_{ij}\right)_{mn}\sum_{r=1}^{d_\ell}\sum_{k=1}^{d_\ell}U^\ell_{mr}\left(A^*\right) e^j_k\left(v_s\right)\int_GU^\ell_{rn}\left(B\right)\overline{U^q_{p,k}\left(B\right)}d\eta(B)
	\end{align}
	By Schur's orthogonality relations (see e.g. \cite{nonComHarmAnalys}), we have
	\begin{equation}
		\int_GU^\ell_{rn}\left(B\right)\overline{U^q_{p,k}\left(B\right)}d\eta(B) = d_q^{-1}\delta_{rp}\delta_{nk}\delta_{\ell q},
	\end{equation}
	by which we get that the expression in \eqref{eigdecompProof:afterPluggingPhi} for $\left\{W\Phi_{q,p,s}\right\}(i,A)$ becomes
	\begin{align}\label{eigdecompProof:evecResultForWMidCalc}
		\left\{W\Phi_{q,p,s}\right\}(i,A) &= \sum_{j=1}^N \sum_{m,n=1}^{d_q}\left(\hat{W}^q_{ij}\right)_{mn}U^q_{mp}\left(A^*\right) e^j_n\left(v_s\right)\\ \nonumber
		& = \sum_{m=1}^{d_q} U^q_{mp}\left(A^*\right) \sum_{j=1}^N\sum_{n=1}^{d_q}\left(\hat{W}^q_{ij}\right)_{mn}e^j_n\left(v_s\right). 
	\end{align}
	Next, we notice that 
	\begin{equation}
		\sum_{j=1}^N\sum_{n=1}^{d_q}\left(\hat{W}^q_{ij}\right)_{mn}e^j_n\left(v_s\right) = \left(\hat{W}^q\right)_{(i-1)d_q+m,\cdot}\cdot v_s,
	\end{equation}
	where $\hat{W}^q$ is the block matrix of Fourier coefficients matrices of $q^{\text{th}}$ order that was defined in \eqref{sec2:blockFourierMat}, and $(\hat{W}^q)_{((i-1)d_q+m,\cdot)}$ is the $m^{\text{th}}$ row of the $d_q\times Nd_q$ matrix consisting of blocks $(i,1),(i,2),\ldots,(i,N)$ of $\hat{W}^q$. Thus, we get that
	\begin{equation}\label{eigdecompProof:evecResultForW}
		\left\{W\Phi_{q,p,s}\right\}(i,A)= \sum_{m=1}^{d_q} U^q_{mp}\left(A^*\right)\left(\hat{W}^q\right)_{(i-1)d_q+m,\cdot}\cdot v_s.
	\end{equation}
	Next, we notice that by the definition of $D^\ell$ in statement of the theorem, we have that 
	\begin{equation}\label{eigdecompProof:diagOfDl}
		D^\ell_{(i-1)d_\ell+m,(i-1)d_\ell+m}= D_{ii}, \quad m\in\{1,\ldots d_\ell\}, \quad i\in\{1,\ldots,N\}.
	\end{equation}
	That is, the $(m,m)$-th element of the $(i,i)$-th block of the $Nd_\ell \times Nd_\ell$ matrix $D^\ell$ is given by $D_{ii}$ defined in $\eqref{GinvDef:Ddef}$, for all $m\in \{1,\ldots, d_\ell\}$. 
	Thus, by \eqref{GinvDef:Ldef}, \eqref{eigdecompProof:evecResultForW} and~\eqref{eigdecompProof:diagOfDl}, we have
	\begin{align}\label{eigdecompProof:evProp1}
		L\left\{\Phi_{q,p,s}\right\}(i,A) &= D_{ii}\Phi_{q,p,s}(i,A) - \sum_{m=1}^{d_q} U^q_{mp}\left(A^*\right)\left(\hat{W}^q\right)_{(i-1)d_q+m,\cdot}\cdot v_s \nonumber \\ 
		&= D_{ii}e^i\left(v_s\right)U^q_{\cdot,p}(A^*) - \sum_{m=1}^{d_q} U^q_{mp}\left(A^*\right)\left(\hat{W}^q\right)_{(i-1)d_q+m,\cdot}\cdot v_s\nonumber \\ 
		&= \sum_{m=1}^{d_q}U_{m,p}^q(A^*)D^q_{(i-1)d_q+m,(i-1)d_q+m}e^i_m(v_s) \nonumber \\
        &\qquad\qquad\qquad\qquad\qquad -\sum_{m=1}^{d_q} U^q_{mp}\left(A^*\right)\left(\hat{W}^q\right)_{(i-1)d_q+m,\cdot}\cdot v_s \nonumber 	\\
		&=\sum_{m=1}^{d_q}U_{m,p}^q(A^*)\left(D^q_{(i-1)d_q+m,\cdot}\cdot v_s-\left(\hat{W}^q\right)_{(i-1)d_q+m,\cdot}\cdot v_s\right)\nonumber \\ 
		&=\sum_{m=1}^{d_q}U_{m,p}^q(A^*)\left(D^q_{(i-1)d_q+m,\cdot}-\left(\hat{W}^q\right)_{(i-1)d_q+m,\cdot} \right)v_s.
	\end{align}
	Thus, since $v_s$ is an eigenvector of  $D^q-\hat{W}^q$ corresponding an eigenvalue $\lambda$, we get 
	\begin{align}\label{eigdecompProof:evProp2}
		L\left\{\Phi_{q,p,s}\right\}(i,A) &= \sum_{m=1}^{d_q}U_{m,p}^q(A^*)\lambda e^i_m(v_s) = \lambda \sum_{m=1}^{d_q}e^i_m(v_s)U_{m,p}^q(A^*) \\ \nonumber 
		&= \lambda e^i(v_s)U^q_{\cdot,p}(A^*)= \lambda \Phi_{q,p,s}(i,A),
	\end{align}
	showing that the function $\Phi_{q,p,s}$
	is an eigenfunction of~$L$ in \eqref{GinvDef:Ldef}. 
	
	Next, we show that the eigenfucntions in \eqref{GInvLapProp:EigenFuncForm} are orthogonal. Indeed, we have that
	\begin{align}\label{eigdecompProof:Orthogonality1}
		\langle\Phi_{m_1,k_1,\ell_1},\Phi_{m_2,k_2,\ell_2}\rangle_{\mathcal{H}}&=\sum_{j=1}^N\int_G\Phi_{m_1,k_1,\ell_1}(j,A)\Phi_{m_2,k_2,\ell_2}^*(j,A)d\eta(A) = \\ \nonumber
		&=\sum_{j=1}^N \int_Ge^j(v_{m_1,\ell_1})U_{\cdot,k_1}^{\ell_1}(A) (U_{\cdot,k_2}^{\ell_2}(A))^*(e^j(v_{m_2,\ell_2}))^*d\eta(A) \\ \nonumber
		&=\sum_{j=1}^N e^j(v_{m_1,\ell_1})\left(\int_GU_{\cdot,k_1}^{\ell_1}(A) \left(\overline{U_{\cdot,k_2}^{\ell_2}(A)}\right)^TdA\right)(e^j(v_{m_2,\ell_2}))^*.
	\end{align}
	The outer product of rows $U_{\cdot,k_1}^{\ell_1}(A) \left(\overline{U_{\cdot,k_2}^{\ell_2}}\right)^T(A)$ is a $d_{\ell_1}\times d_{\ell_2}$ matrix of products between elements of the IURs of $G$, and by Schur's orthogonality relations, we have that 
	\begin{equation}
		d_\ell\cdot\int_GU_{\cdot,k_1}^{\ell_1}(A) \overline{U_{k_2,\cdot}^{\ell_2}}(A)d\eta(A) =
		\begin{cases}
			I_{d_\ell\times d_\ell} & \ell_1=\ell_2=\ell, k_1=k_2=k, \\
			0 & \text{otherwise}.
		\end{cases} 
	\end{equation}
	Thus, when $\ell_1=\ell_2=\ell$ and  $k_1=k_2=k$, we are left with
	\begin{align}
		\langle\Phi_{m_1,k_1,\ell_1},\Phi_{m_2,k_2,\ell_2}\rangle_{\mathcal{H}} &=d_\ell\cdot\sum_{j=1}^N e^j(v_{m_1,\ell})(e^j(v_{m_2,\ell}))^*\\ \nonumber 
		&= d_\ell\cdot\langle v_{m_1,\ell},v_{m_2,\ell}\rangle_{\mathbb{C}^{Nd_\ell}}= 
		\begin{cases}
			d_\ell  & m_1 = m_2 = m,\\
			0 & m_1\neq m_2, 
		\end{cases}
	\end{align}
	which shows that $\Phi_{m_1,k_1,\ell_1}$ and $\Phi_{m_2,k_2,\ell_2}$ are orthogonal. 
	
	To show that the eigenfunctions in \eqref{GInvLapProp:EigenFuncForm} form a basis for $\Hspace$, we first assert that the matrices~$\hat{W}^{(\ell)}$ in \eqref{sec2:hat{W}Def} are hermitian. For the latter we require the following result (see p.82 in~\cite{ChirikjianStochasticLieGroups}).
	\begin{lemma}\label{conjInvarLemma}
		Let $G$ be a compact unitary matrix Lie group. Then, we have that 
		\begin{equation}
			\int_{G} f(A^*)d\eta(A) = \int_{G} f(A)d\eta(A), 
		\end{equation}
		where $\eta$ is the Haar measure over $G$. 
	\end{lemma}
	Now, by~\eqref{sec2:hat{W}Def} and~\eqref{GinvDef:Wdef}, we have
	\begin{align*}
		\hat{W}^{(\ell)}_{ji}&=\int_GW_{ji}(I,A)\overline{U^\ell(A)}d\eta(A)=\int_GW_{ij}(I,A^*)
		\overline{\left(U^{\ell}(A^*)\right)^T}d\eta(A)\\ 
		&=\overline{\left(\int_GW_{ij}(I,A^*)U^\ell(A^*)d\eta(A)\right)^T}=\overline{\left(\int_GW_{ij}(I,A)U^\ell(A)d\eta(A)\right)^T}=\left(\hat{W}^{(\ell)}_{ij}\right)^*,
	\end{align*}
	where in passing to the second equality we used the homomorphism property  \eqref{harmAnalysisOnG:hMorphProp} that implies
	\begin{equation*}
		I = U^\ell(AA^*) =  U^\ell(A) \cdot  U^\ell(A^*),
	\end{equation*}  
	hence $U^\ell(A^*) = \left(U^\ell(A)\right)^*$, and Lemma \ref{conjInvarLemma} in passing to the third equality. 

	Now, for $f\in L^2\left(\left\{1,\ldots,N\right\}\times G\right)$ and a fixed $i\in\{1,\ldots,N\}$, we observe that since~$f(i,A)\in L^2(G)$ then we also have ~$\overline{f(i,A)}\in L^2(G)$, and thus we can expand~$\overline{f(i,A)}$ as
	\begin{align}\label{completenessPrf:expInSU(2)Conj}
		\overline{f(i,A)}&=\sum_{\ell\in \I_G} d_\ell\cdot\sum_{m,m'=1}^{d_\ell}\alpha^i_{\ell,m,m'}U^\ell_{m,m'}(A),
	\end{align}
	from which we get
	\begin{align}\label{completenessPrf:expInSU(2)}
		f(i,A)&=\sum_{\ell\in \I_G} d_\ell\cdot\sum_{m,m'=1}^{d_\ell}\tilde{\alpha}^i_{\ell,m,m'}\overline{U^\ell_{m,m'}(A)},
	\end{align}
	where $\tilde{\alpha}^i_{\ell,m,m'} = \overline{\alpha^i_{\ell,m,m'}}$ for all $\ell\in \I$ and $m,m'\in \left\{1,\ldots,d_\ell\right\}$. 
	Fix $\ell$ and $m$. The matrix $W^{(\ell)}$ is hermitian, and thus admits an orthonormal basis of eigenvectors $\left\{v_{n,\ell}\right\}$, $n=1,2,\ldots,d_\ell N$. Thus, we can expand
	\begin{equation*}
		\left(\tilde{\alpha}^1_{\ell,m,1},\ldots,\tilde{\alpha}^1_{\ell,m,d_\ell},\ldots,\tilde{\alpha}^N_{\ell,m,1},\ldots,\tilde{\alpha}^N_{\ell,m,d_\ell}\right)^T=\sum_{n=1}^{Nd_\ell}\beta_{\ell,m,n}v_{n,\ell}, 
	\end{equation*}
	from which we have that
	\begin{equation}\label{completenessPrf:alphaTildCoeff}
		\tilde{\alpha}_{l,m,m'}^i=\sum_{n=1}^{Nd_\ell}\beta_{\ell,m,n}e^i_{m'}(v_{n,\ell}).
	\end{equation}
	Plugging~\eqref{completenessPrf:alphaTildCoeff} back in to \eqref{completenessPrf:expInSU(2)}, we have
	\begin{align*}
		f(i,R)&=d_\ell\cdot\sum_{\ell\in \I}\sum_{m,m'=1}^{d_\ell}\sum _{n=1}^{Nd_\ell}\beta_{\ell,m,n}e^i_{m'}(v_{n,\ell})U_{m',m}^\ell(A^*)\\
		&= d_\ell\cdot\sum_{\ell\in \I}\sum_{m=1}^{d_\ell}\sum _{n=1}^{Nd_\ell}\beta_{\ell,m,n}\sum_{m'=-\ell}^\ell e^i_{m'}(v_{n,\ell})U_{m',m}^\ell(A^*)\\
		&= d_\ell\cdot\sum_{\ell\in \I}\sum_{m=1}^{d_\ell}\sum _{n=1}^{Nd_\ell}\beta_{\ell,m,n}\Phi_{\ell,m,n}(i,A),
	\end{align*}
	which shows directly that any function $f\in L^2\left(\{1,\ldots,N\}\times G\right)$ can be expanded in a series of eigenfunctions of the $G$-GL. 
	
	Lastly, the fact that the eigenvalues of~$L$ are real and non-negative is a direct result of Lemma~\ref{GinvDef:quadFormLemma}, coupled with the fact that~$L$ is a symmetric operator, since we have that~$W_{ij}{(A,B)} = W_{ji}(B,A)$ for all~$A,B\in G$ and all~$i,j\in \left\{1\ldots,N\right\}$. 

\section{Proof of Theorem \ref{sec2:ConvThrmUnnormalized}} \label{convTheoremPrf}
The analysis that follows is a generalization of the proof of Theorem~2 in~\cite{steerMaps}.
The proof is divided into 4 parts, given in appendices \ref{secConvPrf},\ref{secConvRatePrf} \ref{secGinvParam} and \ref{secConvBigLemmaPrf}.
In part~1, we show that the $G$-invariant graph Laplacian converges to the Laplace-Beltrami operator on the data manifold~$\mathcal{M}$. 
In part~2, we derive the convergence rate (the variance term) of our operator, using the proof technique derived in~\cite{convRate} for the standard graph Laplacian. In parts~3 and~4, we provide proof for key results that are used in part~2. Appendices~\ref{secRealMAnifolds} and~\ref{secLieCoordinates} provide some differential geometry background needed in for \ref{secGinvParam}. 

\subsection{Convergence of the $G$-invariant graph Laplacian}\label{secConvPrf}
In this section, we show that for a fixed $\epsilon>0$ and as $N\rightarrow\infty$, the normalized $G$-invariant graph Laplacian approximates the Laplace-Beltrami operator on the data manifold $\mathcal{M}$ up to an $O(\epsilon)$ error at each data point $A\cdot x_i\in G\cdot X$. We will assume w.l.o.g that $A=I$, since all the analysis that follows can be carried out exactly in the same manner and with the same results when $A\neq I$. 

By~\eqref{GinvDef:DMatAction},\eqref{GinvDef:Ddef},\eqref{GinvDef:Ldef}, and \eqref{GinvDef:normGLapDef}, we can write
\begin{align}
	\frac{4}{\varepsilon} \left\{\tilde{L}g\right\}(i,I) &= \frac{4}{\varepsilon}\left[f(x_i) - \sum_{j=1}^N \int_G D^{-1}_{ii}{W}_{ij}(I,B)f(B\cdot x_j)d\eta(B) \right] \nonumber \\
	&= \frac{4}{\varepsilon}\left[f(x_i) -  \frac{\frac{1}{N}\sum_{j=1}^N \int_G \exp{\left\lbrace-{\left\Vert  x_i - B\cdot x_j\right\Vert^2}{/\varepsilon}\right\rbrace}f(B\cdot x_j) d\eta(B)}{\frac{1}{N} \sum_{j=1}^{N}  \int_G \exp{\left\lbrace-{\left\Vert x_i -B\cdot x_j\right\Vert^2}{/\varepsilon}\right\rbrace} d\eta(B)}\right]. \label{eq:steerable graph laplacian fraction}
\end{align}
We now derive the limit of~\eqref{eq:steerable graph laplacian fraction} for $N\rightarrow\infty$ and a fixed $\varepsilon>0$, showing that it is essentially the Laplace-Beltrami operator $\Delta_\man$ with an additional bias error term of~$O(\varepsilon)$. 
First, let us focus on the expression
\begin{equation}\label{convPrf:C1}
	C_{i,N}^1 \coloneq \frac{1}{N}\sum_{j=1}^N \int_G \exp{\left\lbrace-{\left\Vert x_i - B\cdot x_j\right\Vert^2}{/\varepsilon}\right\rbrace}f(B\cdot x_j) d\eta(B),
\end{equation}
which is the numerator of the second term of~\eqref{eq:steerable graph laplacian fraction} (inside the brackets).
Let us define 
\begin{equation}\label{convPrf: groupConvolution}
	H_i(x) \coloneq \int_G \exp{\left\lbrace-{\left\Vert x_i - B\cdot x\right\Vert^2}{/\varepsilon}\right\rbrace}f(B\cdot x)d\eta(B), \quad x\in \man. 
\end{equation}
Since $\left\lbrace x_i \right\rbrace$ are i.i.d samples from $\man$, the law of large numbers implies
\begin{align}
	\lim_{N\rightarrow\infty} C_{i,N}^1 &= \lim_{N\rightarrow\infty} \frac{1}{N}\sum_{j=1}^N H_i(x_j) = \lim_{N\rightarrow\infty} \frac{1}{N}\sum_{j\neq i,j=1}^N H_i(x_j) \label{convPrf:numerLim1}\\
	 &=\mathbb{E}{\left[ H_i(x) \right]} = \int_{\man} H_i(x) p(x)d\omega(x), \label{convPrf:numerLim}
\end{align}
where $p(x)$ is the sampling density of the data over $\man$, and $\omega(x)$ is the measure with respect to the Riemannian metric on $\mathcal{M}$ induced by the standard Euclidean inner product in~$\mathbb{C}^N$. 

Next, we recall that $G$ acts on points $x\in \mathcal{M}$ by multiplication by unitary matrices $A$.
Consider the map $U_A:\man\rightarrow \man$ defined by 
\begin{equation}\label{convPrf:U_ADef}
	U_A(x) = A\cdot x, \quad x\in\man. 
\end{equation}
The pushforward of $\omega$ by $U_A$ is the measure $U_A^*(\omega)(\cdot)$ over $\man$ defined by 
\begin{equation}
	U_A^*(\omega)(S) = \omega(U_A^{-1}(S)), 
\end{equation}
for all Lebesgue-measurable subsets $S\subseteq\man$. 
Since $U_A$ acts as an isometry over $\man$, and the metric tensor over $\man$  is invariant under isometries, we conclude that $\omega(x)$ is $G$-invariant. That is, for fixed $A\in G$ we have 
\begin{equation}\label{convPrf:pushforwardInvarianceProp}
	U_A^*(\omega)(S)=\omega(S), 
\end{equation}
for all Lebesgue-measurable subsets $S\subseteq\man$. 

Using the latter observation, and assuming that $p$ is uniform over $\man$ (and so $p(x) = 1/\operatorname{Vol}\{\mathcal M\}$)  we have
\begin{align}\label{convPrf:measureInv1}
	\int_{\man} H_i(x) p(x)d\omega(x) &= \frac{1}{\operatorname{Vol}\left\lbrace\man\right\rbrace}\int_{\mathcal{M}}\int_G \exp{\left\lbrace-{\left\Vert x_i - B\cdot x \right\Vert^2}{/\varepsilon}\right\rbrace}f(B\cdot x) d\omega(x)d\eta(A) \nonumber \\
	&= \frac{1}{\operatorname{Vol}\left\lbrace\man\right\rbrace}\int_G\int_{\mathcal{M}} \exp{\left\lbrace-{\left\Vert x_i - y \right\Vert^2}{/\varepsilon}\right\rbrace}f(y) dU_A^*(\omega)(y)d\eta(A) \nonumber \\
	&=\frac{1}{\operatorname{Vol}\left\lbrace\man\right\rbrace}\int_G\int_{\mathcal{M}} \exp{\left\lbrace-{\left\Vert x_i - y \right\Vert^2}{/\varepsilon}\right\rbrace}f(y) d\omega(y)d\eta(A) \nonumber \\
	&=	\frac{1}{\operatorname{Vol}\left\lbrace\man\right\rbrace} \int_{\mathcal{M}} \exp{\left\lbrace-{\left\Vert x_i - y \right\Vert^2}{/\varepsilon}\right\rbrace}f(y) d\omega(y), 
\end{align}
where in the second equality we applied the change of variables $y=B\cdot x = U_B(x)$, and in the fourth equality that $\operatorname{Vol}(G)=1$, by~\eqref{secLieGroupAction:probMeasure}. 

In a similar fashion, if we consider the denominator of the second term in~\eqref{eq:steerable graph laplacian fraction}
\begin{equation}\label{convPrf:C2}
	C_{i,N}^2 \coloneq \frac{1}{N}\sum_{j=1}^N \int_G \exp{\left\lbrace-{\left\Vert x_i - B\cdot x_j\right\Vert^2}{/\varepsilon}\right\rbrace} d\eta(B),
\end{equation}
and by repeating the calculations carried above for $C_{i,N}^1$ but with $f\equiv 1$, we get that
\begin{equation}\label{eq: measInvar2}
	\lim_{N\rightarrow\infty} C_{i,N}^2 = \frac{1}{\operatorname{Vol}\left\lbrace\man\right\rbrace} \int_{\man}  \exp{\left\lbrace-{\left\Vert x_i - x \right\Vert^2}{/\varepsilon}\right\rbrace} d\omega(x) = \expect{G_i(x)},
\end{equation}
where we defined 
\begin{equation}\label{convPrf:denomLim}
	G_i(x)  \coloneq \sum_{j=1}^{N}\int_G \expPow{x_i}{B\cdot x}d\eta(B), \quad x\in \man.
\end{equation}

Lastly, if we substitute~\eqref{convPrf:measureInv1} and~\eqref{eq: measInvar2} into~\eqref{eq:steerable graph laplacian fraction}, we have that
\begin{align}
	\lim_{N\rightarrow\infty} \frac{4}{\varepsilon} \left\{\tilde{L}g\right\}(i,I) &=  \frac{4}{\varepsilon}\left[f(x_i) -  \frac{\frac{1}{\operatorname{Vol}\left\lbrace\man\right\rbrace} \int_{\man}  \exp{\left\lbrace-{\left\Vert x_i - x\right\Vert^2}{/\varepsilon}\right\rbrace}f(x) d\omega(x)}{\frac{1}{\operatorname{Vol}\left\lbrace\man\right\rbrace} \int_{\man}  \exp{\left\lbrace-{\left\Vert x_i - x\right\Vert^2}{/\varepsilon}\right\rbrace} d\omega(x)}\right] \label{eq:fraction limit}
	\\ &= \Delta_\man f(x_i) + O(\varepsilon), \label{eq:bias error in proof}
\end{align}
where~\eqref{eq:bias error in proof} is justified in~\cite{convRate}.

\subsection{The convergence rate}\label{secConvRatePrf}
The variance error in the approximation of the Laplace-Beltrami operator by the $G$-GL (second term on the r.h.s of \eqref{convSec:errFormula}),
is attributed to the difference between the values of $C_{i,N}^1$ and $C^2_{i,N}$ for a finite $N$ and their limit when $N\rightarrow \infty$. 
To derive the variance error we employ the proof technique derived in \cite{convRate}. As in Section~\ref{secConvPrf}, we perform all our analysis in a neighbourhood of an arbitrary data point $A\cdot x_i$, assuming w.l.o.g that $A=I$. 

Using the definitions \eqref{convPrf: groupConvolution} and \eqref{convPrf:denomLim}, the normalized $G$-GL $\eqref{GinvDef:normGLapDef}$ applied to an arbitrary smooth function~$f$ on~$\man$, and evaluated at the fixed point~$(i,I)$ can be written as
\begin{equation}\label{convPrf:LHG}
	\tilde{L}\left\{f\right\}(i,I) = f(x_i) - \frac{\sum_{j=1}^{N}H_i(x_j)}{\sum_{j=1}^{N}G_i(x_j)}. 
\end{equation}
Following \cite{convRate}, we employ the Chernoff tail inequality to bound the probability of \eqref{convPrf:LHG} deviating from its mean (the limit of \eqref{convPrf:LHG} when~$N\rightarrow \infty$). 
We now derive a bound on the probability of the $\alpha$-error
\begin{equation}\label{convPrf:pPlus}
	p_+(N,\alpha) = Pr\left\{\frac{\sum_{j\neq i}^N H_i(x_j)}{\sum_{j\neq i }^{N} G_i(x_j)}-\frac{\mathbb{E}\left[H_i\right]}{\mathbb{E}\left[G_i\right]}>\alpha\right\},
\end{equation}
where we point out that excluding the diagonal terms $H_i(x_i)$ and $G_i(x_i)$ results in an even smaller error than the variance error itself, as was shown in \cite{steerMaps} and \cite{convRate}.
We also point out that a bound on the probability
\begin{equation}
		p_-(N,\alpha) = Pr\left\{\frac{\sum_{j\neq i}^N H_i(x_j)}{\sum_{j\neq i }^{N} G_i(x_j)}-\frac{\mathbb{E}\left[H_i\right]}{\mathbb{E}\left[G_i\right]}<-\alpha\right\},
\end{equation}
can be obtained by the same technique that we now apply to bound \eqref{convPrf:pPlus}.

Now, defining
\begin{equation}\label{eq:Ji}
	J_i(x_j) \coloneq \expect{G_i}H_i(x_j)-\expect{H_i}G_i(x_j)+\alpha \expect{G_i}\left( \expect{G_i}-G_i(x_j)\right),
\end{equation}
it was shown in \cite{convRate} that $p(N,\alpha)$ can be rewritten as 
\begin{equation}
	p_+(N,\alpha) = Pr\left\{\sum_{j\neq i}^{N}J_i(x_j)>\alpha(N-1)\left( \expect{G_i} \right)^2\right\}, 
\end{equation}
where $	J_i(x_j)$ are i.i.d random variables.  Using the Chernoff inequality we get
\begin{equation}\label{convPrf:pPlusChernBound}
	p_+(N,\alpha) \leq \exp \left\{ \frac{\alpha^2(N-1)^2 \left( \expect{G_i} \right)^4}{2(N-1)\expect{\left(J_i\right)^2}+O\left(\alpha\right)} \right\}. 
\end{equation}
From~\eqref{eq:Ji} we get by a direct calculation that
\begin{equation}\label{convPrf:JiExpression}
\begin{split}
	\expect{\left( J_i \right)^2} = \left(\expect{G_i}\right)^2\expect{\left( H_i\right)^2}-2\expect{G_i}\expect{H_i}\expect{H_iG_i}\\
 +\left(\expect{H_i}\right)^2\expect{\left(G_i\right)^2}+O\left(\alpha\right).
 \end{split}
\end{equation}
To evaluate all the moments in \eqref{convPrf:JiExpression}, and consequently the quantities~$\expect{\left( J_i \right)^2}$ and~$\expect{\left( G_i \right)^4}$ in~\eqref{convPrf:pPlusChernBound}, we use the following result from~\cite{convRate}, which will be the key instrument in the analysis that follows. 
 \begin{theorem}\label{convPrf: ClassicalConvRateThrm}
	Let $\mathcal{M}$ be a smooth and compact $d$-dimensional submanifold, and let $f:\mathcal{M}\rightarrow\mathcal{R}$ be a smooth function. Then, for any $y\in \mathcal{M}$ 
	\begin{equation}\label{convPrf:gaussIntegral}
		\left(\pi \epsilon\right)^{-d/2}\int_{\mathcal{M}} e^{-\norm{y-x}^2/\epsilon}f(x)dx = f(y)+\frac{\epsilon}{4}\bigg[E(y)f(y)+\Delta_{\mathcal{M}}f(y)\bigg]+O\left(\epsilon^2\right),
	\end{equation}
	where $E(y)$ is a scalar function of the curvature of $\mathcal{M}$ at $y$. 
\end{theorem}
This shows that the integral on the l.h.s of \eqref{convPrf:gaussIntegral} essentially operates as an evaluation functional of $f$ at the point $y$, up to an $O(\epsilon)$ error. 

Applying Theorem \ref{convPrf: ClassicalConvRateThrm} to the first order moments appearing in \eqref{convPrf:JiExpression}, we immediately obtain
\begin{equation}\label{convPrf:firstMomentApprxFun}
	\expect{H_i} = \frac{1}{\text{Vol}\left\{\man\right\}}\int_{\man}\expPow{x_i}{x}f(x)dx = \frac{1}{\text{Vol}\left(\man\right)}(\pi \epsilon)^{d/2}\big[f(x_i)+O\left(\epsilon\right)\big],
\end{equation}
and
\begin{equation}\label{convPrf:firstMomentApprxConst}
	\expect{G_i} = \frac{1}{\text{Vol}\left\{\man\right\}}\int_{\man}\expPow{x_i}{x}dx = \frac{1}{\text{Vol}\left(\man\right)}(\pi \epsilon)^{d/2}\big[1+O\left(\epsilon\right)\big].
\end{equation}
Thus, in order to evaluate $\expect{\left( J_i \right)^2}$ in \eqref{convPrf:JiExpression}, it remains to evaluate the second order moments $\expect{\left( H_i\right)^2},\expect{\left( G_i\right)^2}$ and $\expect{ H_iG_i}$, which we carry out in two steps in the following two sections. 

First, in Section \ref{secGinvParam} we construct a local parametrization of $\man$ in a sufficiently small neighborhood $\man'$ of $x_i$, such that each $x\in \man'$ is mapped to a unique pair~$(z,B)$, where all the~$z$ values reside on a~$(d-d_G)$-dimensional submanifold~$\mathcal{N}\subset\man'$, and $B\in G$. 
Next, in Section~\ref{secConvBigLemmaPrf}, we use the results of Section~\ref{secGinvParam} to reduce integration over $\man$ in the expressions for the second order moments in~\eqref{convPrf:JiExpression} to integration over~$\subman$, leading to the following lemma. 
\begin{lemma}\label{convPrf:scndMomentApproxLemma}
	There exist a smooth function $\mu(x)>0$ over $\subman$, and a smooth function $p_\subman(x)>0$ over $\man'$ such that
	\begin{align}
		&\expect{\left(H_i(x)\right)^2}= \frac{(\pi\epsilon)^{\left(d+d_G\right)/2}}{2^{(d-d_G)/2}}\bigg[ \frac{f^2(x_i)p_\subman(x_i)}{\mu^2(x_i)}+O(\epsilon)\bigg],\label{convPrf:scndMomentApproxLemmaHi2}\\ 
	&\expect{\left(G_i(x)\right)^2}= \frac{(\pi\epsilon)^{\left(d+d_G\right)/2}}{2^{(d-d_G)/2}}\bigg[ \frac{p_\subman(x_i)}{\mu^2(x_i)}+O(\epsilon)\bigg],\\ 
		&\expect{H_i(x)G_i(x)} = \frac{(\pi\epsilon)^{\left(d+d_G\right)/2}}{2^{(d-d_G)/2}}\bigg[ \frac{f(x_i)p_\subman(x_i)}{\mu^2(x_i)}+O(\epsilon)\bigg].
\end{align}
\end{lemma}
\noindent By Lemma \ref{convPrf:scndMomentApproxLemma}, \eqref{convPrf:firstMomentApprxConst} and \eqref{convPrf:firstMomentApprxFun}, we have
\begin{align}
	\expect{\left(J_i\right)^2}&=\frac{1}{\text{Vol}\left(\man\right)^2}\frac{(\pi\epsilon)^{\left(d+d_G\right)/2}}{2^{(d-d_G)/2}}\left(\pi\epsilon\right)^d\bigg[ \frac{f^2(x_i)p_\subman(x_i)}{\mu^2(x_i)}+O(\epsilon)\bigg] \\ \nonumber
	&-2\frac{1}{\text{Vol}\left(\man\right)^2}\frac{(\pi\epsilon)^{\left(d+d_G\right)/2}}{2^{(d-d_G)/2}}\left(\pi\epsilon\right)^d\bigg[ \frac{f^2(x_i)p_\subman(x_i)}{\mu^2(x_i)}+O(\epsilon)\bigg]\\ \nonumber
	&+\frac{1}{\text{Vol}\left(\man\right)^2}\frac{(\pi\epsilon)^{\left(d+d_G\right)/2}}{2^{(d-d_G)/2}}\left(\pi\epsilon\right)^d\bigg[ \frac{f^2(x_i)p_\subman(x_i)}{\mu^2(x_i)}+O(\epsilon)\bigg] + O\left(\alpha\right)\\ \nonumber
	&= \frac{1}{\text{Vol}\left(\man\right)^2}\frac{(\pi\epsilon)^{\left(d+d_G\right)/2}}{2^{(d-d_G)/2}}\left(\pi\epsilon\right)^d \cdot O(\epsilon) = O\left( \epsilon^{3d/2+(d_G+2)/2} \right) + O\left( \alpha \right).
\end{align}
We now obtain Theorem \ref{sec2:ConvThrmUnnormalized}, and in particular \eqref{convSec:errFormula}, by repeating the computations in equations $121-128$ in the proof of Lemma~$10$ in \cite{steerMaps}, with~$d_G=1$ replaced by an arbitrary group dimension~$d_G$.

\subsection{Real manifolds embedded in $\mathbb{C}^n$}\label{secRealMAnifolds}
Before we continue with the proof of Theorem \ref{sec2:ConvThrmUnnormalized}, we describe the way we view real manifolds embedded in a complex vector space. 

Firstly, we point out that we say that a $d$-dimensional manifold $\man\subset\mathbb{C}^n$ is real, in the sense that its charts are given by maps of the form $\Phi:U\rightarrow\mathbb{R}^{d}$, where $U$ is an open subset in $\man$. This is in contrast to complex manifolds that admit charts that map open subsets in the manifold to the unit disk in~$\mathbb{C}^n$. The crucial distinction between the two is that real manifolds admit a differentiable structure where the transition maps between charts are differentiable with respect to real variables, while complex manifolds admit transition maps that are holomorphic. 

Specifically , we can formulate our entire analysis in a real space by identifying $\mathbb{C}^n$ with $\mathbb{R}^{2n}$ via the map 
\begin{equation}\label{catReIm}
	z \mapsto \tilde{z} = \begin{pmatrix}
		\re{z}\\ \im{z}
	\end{pmatrix},\quad  z\in \mathbb{C}^n.
\end{equation}
If we equip~$\mathbb{C}^n$ with the real valued inner product given by 
\begin{equation}\label{realDotProd}
	\dprod{u}{v} = \re{\dprod{u}{v}_{\mathbb{C}^n}},
\end{equation}
the map \eqref{catReIm} becomes an isometry, since
\begin{equation}\label{isometryProp}
	\re{\dprod{z}{w}_{\mathbb{C}^n}} = \dprod{\tilde{z}}{\tilde{w}}, \quad z,w \in \mathbb{C}^n. 
\end{equation}

Now, let $\man \subset \mathbb{C}^n$ be an embedded $d$-dimensional submanifold, and let $\left(U,\Phi\right)$ be a local chart on $\man$, that is, $U\subset\man$ is an open subset, and $\Phi$ is a diffeomorphism that maps $U$ onto an open subset of $\mathbb{R}^d$, where we identify~$U$ with the set
\begin{equation*}
	\tilde{U}=\left\{\tilde{u}\; :\; u\in U\right\}
\end{equation*} 
using the map $\tilde{(\cdot)}$ in~\eqref{catReIm}.
The inverse map $\Phi^{-1}(u_1,\ldots,u_d)$ parametrizes the points~$x\in U$ as $x = x(u_1,\ldots,u_d) = \Phi^{-1}(u_1,\ldots,u_d)$. The Jacobian matrix of the latter parametrization is given in coordinates by
\begin{equation}
	J_u = \begin{pmatrix}
		\pDeriv{x_1}{u_1} & \dots& \pDeriv{x_1}{u_d}\\
		\vdots & \dots & \vdots \\
		\pDeriv{x_{n}}{u_1} & \dots& \pDeriv{x_n}{u_d}
	\end{pmatrix}. 
\end{equation}
Thus, denoting
\begin{equation}\label{coordGrad}
	\pDeriv{x_i}{u} = \left(\pDeriv{x_i}{u_1},\ldots,\pDeriv{x_i}{u_d}\right)^T,
\end{equation}
the metric tensor induced on $\man$ by the dot product \eqref{realDotProd} is given in local coordinates as
\begin{equation}
	\begin{pmatrix}
		\re{\dprod{\pDeriv{x_1}{u}}{\pDeriv{x_1}{u}}_{\mathbb{C}^n}} & \dots & \re{\dprod{\pDeriv{x_1}{u}}{\pDeriv{x_n}{u}}_{\mathbb{C}^n}}\\
		\vdots & \dots & \vdots\\
		\re{\dprod{\pDeriv{x_1}{u}}{\pDeriv{x_n}{u}}_{\mathbb{C}^n}} & \dots & \re{\dprod{\pDeriv{x_n}{u}}{\pDeriv{x_{n}}{u}}_{\mathbb{C}^n}}
	\end{pmatrix}=\re{J^*J}.
\end{equation}
The latter enables us to integrate smooth functions over open subsets $U\subset\man$ by 
\begin{equation}
	\int_U f(x)dx = \int_{\Phi^{-1}\left(U\right)} f(x(u))dV(u), 
\end{equation}
where $V(u)$ is the volume form on $\man$ given by 
\begin{equation}
	V(u) = \sqrt{\det\left\{\re{J^*J}\right\}},
\end{equation}
and 
\begin{equation}
	dV(u) = V(u)du = V(u)du_1\dots du_n.
\end{equation}

\subsection{Coordinate charts on Lie groups}\label{secLieCoordinates}
The next part of the proof of Theorem \ref{sec2:ConvThrmUnnormalized} also requires us to define coordinate charts on Lie groups. 
The standard coordinates on a Lie group~$G$ is given by the exponential map over the Lie-algebra of~$G$. In detail, the Lie-algebra $\mathfrak{g}$ of $G$ is the tangent space to $G$ at the identity $I_G$. By a theorem (due to Von-Neumann, see~\cite{lieGroupsHall}), there exists a sufficiently small neighborhood $\mathcal{N}_I\subset G$ of $I_G$, where for each $A\in G$ there exists a unique element $X\in \mathfrak{g}$ such that $\exp(X) = A$, where $\exp(\cdot)$ is the  matrix exponential. Thus, choosing a basis $\left\{X_1,...,X_{d_G}\right\}$ for $\mathfrak{g}$,
we can write each element $X\in\mathfrak{g}$ as a linear combination
\begin{equation}
	X = \sum_{i=1}^{d_G} u_i X_i, 
\end{equation}
inducing a coordinate chart for $\mathcal{N}_I\subset G$, such that the elements of $\mathcal{N}_I$ are given explicitly by the matrix valued map
\begin{equation}\label{convPrf: expMapForGnearI}
	A_{I}(u_1,\ldots,u_{d_G}) \coloneq \exp\left(\sum_{i=1}^{d_G} u_i X_i\right),\quad (u_1,\ldots,u_{d_G})\in U,
\end{equation}
where $U\subset\mathbb{R}^{d_G}$ is an open subset. 
Multiplying the elements of $\mathcal{N}_I$ by a fixed element $B\in G$ (either from the left or the right) translates  $\mathcal{N}_I$ to a neighborhood~$\mathcal{N}_B$ of~$B$.  
A chart for $\mathcal{N}_B$ is thus given by (also see \cite{tappMatrixGroups})
\begin{equation}\label{convPrf: expMapForG}
		A_I(u_1,\ldots,u_{d_G})\cdot B = \exp\left(\sum_{i=1}^{d_G} u_i X_i\right)\cdot B,\quad (u_1,\ldots,u_{d_G})\in U. 
\end{equation}	
Hence, an atlas of charts for $G$ can be obtained by choosing a finite cover of~$G$ (since~$G$ is compact) by such neighborhoods. 
Equipped with the map~\eqref{convPrf: expMapForG}, 
a chart for a neighborhood of $B\cdot x \in G\cdot x$ is obtained by multiplying~$x$ by~\eqref{convPrf: expMapForG} on the left, that is
\begin{equation}\label{convPrf:expCoordinatesForGExplicit}
	(u_1,\ldots,u_{d_G})\mapsto A_I(u_1,\ldots,u_{d_G})\cdot B\cdot x, \quad (u_1,\ldots,u_{d_G})\in \exp^{-1}(\mathfrak{g}).
\end{equation}
Since $G\cdot x$ is diffeomorphic to $G$, and thus compact, an atlas of charts for $G\cdot x$ is obtained by choosing a finite covering of~$G\cdot x$. 
In this work, to simplify notation, we refer to all charts in the atlas for~$G\cdot x$ by the notation $A(u_1,\ldots,u_{d_G})\cdot x$, where we define implicitly that
\begin{equation}\label{convPrf:expCoordinatesForG}
	A(u_1,\ldots,u_{d_G})\cdot x=A_I(u_1,\ldots,u_{d_G})\cdot B,
\end{equation}	
whenever we refer to points in a chart for a neighborhood of $B\cdot x$.

\subsection{G-invariant local parametrization of the data manifold}\label{secGinvParam}
In this section, we construct a parametrization of $\man$ in a local neighbourhood $\man'$ around the point $B\cdot x_i\in G\cdot x_i$, which takes the form of a product between $G$ and a certain $(d-d_G)$-dimensional submanifold in $\man$, and derive the integration volume form over $\man'$ in terms of the resulting local coordinates. The parametrization we construct is $G$-invariant in the sense that for every $x\in \man'$ we have $G\cdot x\subset \man'$. As in the previous sections, for the rest of the this section, we will assume w.l.o.g that~$B=I$.

To construct our parametrization, for each $x\in \man$ we consider the solution to minimization problem
\begin{equation}\label{convPrf: minimiaztion problem}
 \hat{A}(x) = \argmin_{A\in G} \lVert Ax-x_i\rVert, 
\end{equation}
and the value $	z(x) = \hat{A}(x)\cdot x$. In other words, for each~$x \in \man$ we solve for the element $\hat{A}(x)\in G$ such that $z(x)$ is the point on the orbit $G\cdot x$ closest to $x_i$.
Since~$G$ is compact, a solution for \eqref{convPrf: minimiaztion problem} exists. 
In Lemma~\ref{convPrf: uniqueProjLemma} below, we prove that there exists a certain neighborhood $\man'\subset \man$ of $x_i$, such that the solution of 
\eqref{convPrf: minimiaztion problem} is also unique for each~$x$ in this neighborhood.
Subsequently, we parameterize points~$x\in \man'$ by  
\begin{equation}\label{convPrf: parametrization}
	x(z,\hat{B})= \hat{B}\cdot z, \quad z\in \subman,
\end{equation}
where $\hat{B}(x) = (\hat{A}(x))^*$, and $\subman$ is the set of unique solutions of \eqref{convPrf: minimiaztion problem} for all $x\in\man'$. 
The proof of Lemma~\ref{convPrf: uniqueProjLemma} requires the notion of a~$\delta$-neighborhood of a manifold.
\begin{definition}\label{deltaNbdDef}
	Let $M\subset \mathbb{C}^n$ be a smooth compact embedded submanifold. Given a~$\delta>0$, the~$\delta$-neighborhood of~$M$ is defined as
	\begin{equation}
		M_{\delta} = \left\{y\in \mathbb{C}^n \; : \; \norm{x-y}<\delta \text{ for some } x\in M\right\}.
	\end{equation}
\end{definition}
Our proof also requires the following property of $\delta$-neighborhoods. 
\begin{theorem}\label{tubularThrm}
	There exists a~$\delta>0$ such that any~$x\in M_\delta$ has a unique closest point in~$M$. 
\end{theorem}
For a proof, see Theorem 6.24 and Proposition 6.25 in \cite{leeSmoothMan}. 
\begin{lemma}\label{convPrf: uniqueProjLemma}
	There exists a~$\delta>0$ such that the problem \eqref{convPrf: minimiaztion problem} has a unique solution for any~$x$ in a $\delta$-neighborhood of~$G\cdot x_i$. Furthermore, the $\delta$-neighborhood~$\left(G\cdot x_i\right)_\delta$ is~$G$-invariant.
\end{lemma}
\begin{proof}
	By assumption, with probability one we have that $Ax_i\neq x_i$ for all $A\in G$. Since~$G$ is a smooth manifold, the map $x_i\mapsto A\cdot x_i$, $A\in G$, is a smooth injective map onto the orbit $G\cdot x_i\subset\man$. This implies that $G\cdot x_i$ is a smooth $d_G$-dimensional compact embedded submanifold in $\mathcal{M}$, diffeomorphic to~$G$.
	By Theorem~\ref{tubularThrm}, there exists a~$\delta>0$ such that for any $x\in \left(G\cdot x_i\right)_\delta$ (see Definition \ref{deltaNbdDef}) there exists a unique solution to the problem 
	\begin{align}
		\min_{y\in G\cdot x_i}\norm{x- y} &= \min_{A\in G}\norm{x- A\cdot x_i} = \min_{A\in G}\norm{A^*x- x_i}\nonumber \\
		 &= \min_{A\in G}\norm{Ax- x_i} = \min_{z\in G\cdot x}\norm{z- x_i},
	\end{align}
	which shows that there exists a unique point~$z\in G\cdot x$ closest to~$x_i$, which is given by $z=\hat{A}(x)\cdot x$, where $\hat{A}(x)$ is the unique solution to~\eqref{convPrf: minimiaztion problem}. 
	
	Moreover, we observe that
	\begin{equation}\label{interOrbitDistInvar}
		d = \norm{z-x_i}=\norm{Bz-Bx_i},
	\end{equation}
	for all $B\in G$, which shows that any point on the orbit $G\cdot x$ has a point $y\in G\cdot x_i$ at the minimal distance~$d$. Thus, using that $d\leq \delta$ (since $z\in (G\cdot x_i)_\delta$) Definition~\ref{deltaNbdDef} implies that
	\begin{equation}\label{deltaNbdGinv}
		G\cdot x \subset \left(G\cdot x_i\right)_\delta,
	\end{equation}
	for all~$x\in \left(G\cdot x_i\right)_\delta$,
	which shows that the $\delta$-neighborhood~$\left(G\cdot x_i\right)_\delta$ is~$G$-invariant.  
\end{proof}

Now, let us denote
\begin{equation}\label{manNbd}
	\man' \coloneq \left(G\cdot x_i \right)_\delta \cap \man,	
\end{equation}
for $\delta > 0$ as in Lemma~\ref{convPrf: uniqueProjLemma}. The subset $\man'$ is an open subset of $\man$ and thus a submanifold in $\man$. Furthermore, Lemma~\ref{convPrf: uniqueProjLemma} also implies that the neighborhood~$\left(G\cdot x_i \right)_\delta$ is~$G$-invariant, and since $\man$ is also $G$-invariant than so is $\man'$.
Let us further denote by
\begin{equation}\label{subManDef}
	\mathcal{N} \coloneq \left\{z \;:\; \min_{A\in G} \lVert Ax-x_i\rVert = \norm{z-x_i}, \quad x\in \man'\right\},
\end{equation}
the set resulting from solving~\eqref{convPrf: minimiaztion problem} for $x\in \man'$.  
Using \eqref{manNbd} and \eqref{subManDef} we can write
\begin{equation}\label{convPrf: submanDef}
	\man' = G\cdot \subman = \left\{A\cdot z \; : \; A\in G, \quad z \in \subman\right\}. 
\end{equation}
We now show that $\subman$ is an embedded compact $(d-d_G)$-dimensional submanifold in $\mathcal{M}$. We do this by deriving an explicit solution for~\eqref{convPrf: minimiaztion problem} for all $x\in \man'$. 
Note that \eqref{convPrf: submanDef} is diffeomorphic to the product space $G\times \subman$,  which implies that $\man'$ is a~$d$~-dimensional submanifold in~$\man$. 

Now, denoting~$u=\left(u_1,\ldots,u_{d_G}\right)$, and differentiating the norm in \eqref{convPrf: minimiaztion problem} with respect to $u_k$ for each $k\in\left\{1,\ldots,d_G\right\}$, the solution $\hat{A}(x)$ for \eqref{convPrf: minimiaztion problem} is given by $A(u)$ that solves the set of $d_G$ equations
\begin{align}\label{convPrf:optDeriv1}
	\text{Re}\left\{\left\langle \frac{\partial A(u)\cdot x}{\partial u_k},A(u)\cdot x - x_i \right \rangle\right\} = 0 , \quad k = 1,\ldots, d_G,
\end{align}
which are equivalent to
\begin{align}\label{convPrf:optDeriv}
	\text{Re}\left\{\left\langle \frac{\partial A(u)\cdot z}{\partial u_k}\bigg|_{u=u^*}, z - x_i \right \rangle\right\}=0, \quad k = 1,\ldots, d_G,
\end{align}
where $A(u^*) = I_G$, since we defined $z$ to be the closest point in~$G\cdot x$ to~$x_i$. In particular, we have $z=\hat{A}(x)\cdot x$, where $\hat{A}$ is the unique solution to \eqref{convPrf: minimiaztion problem}. 
The expression on the l.h.s of the inner product in \eqref{convPrf:optDeriv} is a vector tangent to~$G\cdot z$ at~$z$, since it is the derivative of the map $u\rightarrow A(u)\cdot z$ (an explicit parametrization of the orbit~$G\cdot z$) at $u=u^*$, for which $A(u^*) = I_G$.
Thus, by our discussion in Section~\ref{secRealMAnifolds}, and in particular~\eqref{catReIm}-\eqref{isometryProp},  equation~\eqref{convPrf:optDeriv} simply implies that the closest point to $x_i$ on the orbit $G\cdot x$ is $z$ such that $z-x_i$ is perpendicular to the tangent space of $G\cdot x$ at $z$.
We may now rewrite \eqref{convPrf:optDeriv} as
\begin{equation}\label{convPrf:linConstraintQuadForm}
	\text{Re}\left\{\dprod{\frac{\partial A(u)}{\partial u_k}\at{u=u^*}\cdot z}{z}\right\} = \text{Re}\left\{z^* Y_k  z\right\}, \quad Y_k = \left(\frac{\partial A(u)}{\partial u_k}\at{u=u^*}\right)^*, \quad k=1,\ldots,d_G,
\end{equation}
where $Y_k$ resides in the tangent space to $G$ at $I$, given by the Lie algebra $\mathfrak{g}$ of~$G$. Now, the Lie-algebra $\mathfrak{u}(n)$ of $\U(n)$ is the space of all $n\times n$ skew-Hermitian matrices, and by a theorem (see~\cite{gallierDiffGeom}), if~$G$ is a Lie subgroup of~$\U(n)$, then~$\mathfrak{g}$ is a $d_G$-dimensional subspace of $\mathfrak{u}(n)$.
Using the fact that the diagonal entries of skew-Hermitian matrices are all purely imaginary, we have for any $Y\in \mathfrak{g}$
\begin{align}\label{convPrf:quadFormPureImag}
	z^*Yz &= i\sum_{j=1}^N \left|(Y)_{jj}\right| \left|z_j\right|^2 + \sum_{i<j}^N (Y)_{ij}\overline{z_i}z_j+\sum_{i>j}^N (Y)_{ij}\overline{z_i} z_j  \nonumber \\ 
	& = i\sum_{j=1}^N \left|(Y)_{jj}\right| \left|z_j\right|^2 + \sum_{i<j}^N Y_{ij}\overline{z_i}z_j-\sum_{i<j}^N (\overline{Y})_{ij}z_i\overline{z_j}\nonumber\\ 
	& = i\sum_{j=1}^N \left|(Y)_{jj}\right| \left|z_j\right|^2 -2i\cdot \text{Im}\left\{\  \sum_{i<j}^N (Y)_{ij}\overline{z_i}z_j\right\}, 
\end{align}
where in passing to the second equality we switched the roles of $i$ and $j$ in the third sum and used that $Y_{ij} = -\left(\overline{Y}\right)_{ji}$ since $Y$ is skew-Hermitian. 
Plugging~\eqref{convPrf:quadFormPureImag} into~\eqref{convPrf:linConstraintQuadForm}, we obtain 
\begin{equation}\label{convPrf:linConstraintQuadFormEqualZero}
 	\text{Re}\left\{\dprod{\frac{\partial A(u)}{\partial u_k}\at{u=u^*}\cdot z}{z}\right\}=0, \quad k=1,\ldots,d_G.
\end{equation}
Then, substituting \eqref{convPrf:linConstraintQuadFormEqualZero} into \eqref{convPrf:optDeriv}, we are left with
\begin{equation}\label{convPrf:linConstraints}
	\text{Re}\left\{\dprod{Y_k\cdot z}{x_i}\right\} =0 , \quad k=1,\ldots,d_G,
\end{equation}
by which we can write the set $\mathcal{N}$ in \eqref{subManDef} as 
\begin{equation}\label{convPrf:subManCharacter}
	\mathcal{N} = \left\{ z:	\text{Re}\left\{\dprod{Y_k\cdot z}{x_i}\right\} =0, \quad k=1,\ldots,d_G, \quad z\in\man' \right\}.
\end{equation}

We now observe that $\subman$ is the intersection of an open neighborhood of $\man$ with the subspace of $\mathbb{C}^n$ defined by the $d_G$ linear constraints in~\eqref{convPrf:linConstraints}. In the following lemma we show that~$\subman$ is a $d-d_G$-dimensional submanifold in $\man$.
\begin{lemma}
	The set~$\subman$ in \eqref{convPrf:subManCharacter} is a $d-d_G$-dimensional submanifold in $\man$. 
\end{lemma}
\begin{proof}
	In the following, we use the formulation of real manifolds in $\mathbb{C}^n$ presented in Section \ref{secRealMAnifolds}. In particular, by using the map $\tilde{(\cdot)}$ in \eqref{catReIm}, let us define
	\begin{equation}\label{manTilde}
		\tilde{\man} = \left\{\tilde{x}\; : \; x\in \man\right\}, \quad \tilde{x} = \left(\re{x},\im{x}\right)^T.
	\end{equation}
	 Clearly, the manifold $\tilde{\man}$ is diffeomorphic to $\man$, and by \eqref{realDotProd} and \eqref{isometryProp}, the map~$\tilde{(\cdot)}$ restricted to $\man$ is a Riemannian isometry, preserving the metric tensor of~$\man$.
 	Furthermore, defining 
 	\begin{equation}\label{convPrf:subManTagCharacter}
 		\tilde{\mathcal{N}} = \left\{ \tilde{z}\in \mathbb{R}^{2n}:	\text{Re}\left\{\dprod{Y_1\cdot u}{x_i}_{\mathbb{C}^{n}}\right\}=0, \quad k=1,\ldots,d_G, \quad z\in\man' \right\},
 	\end{equation}
 	we have that $z\in \subman \iff \tilde{z}\in \tilde{\mathcal{N}}$, that is, the map~$\tilde{(\cdot)}$ restricted to~$\subman$ is a bijection (and a isometry) onto~$\tilde{\mathcal{N}}$.
 	Thus, it suffices to show that~$\tilde{\mathcal{N}}$ is a $(d-d_G)$-dimensional submanifold in~$\tilde{\man}$, which we now do.  
 	
	The proof utilizes the implicit function theorem. 
	By a theorem (see proposition~5.16 in~\cite{leeSmoothMan}), there exists a neighborhood~$\tilde{U}$ of~$\tilde{x}_i$ in~$\tilde{\man}$, a diffeomorphism onto its image~$\Phi:\tilde{U}\rightarrow\mathbb{R}^{2n-d}$ , and $c\in \mathbb{R}^{2n-d}$ such that~$\tilde{U}$ can be parameterized as 
	\begin{equation}
		\Phi(\tilde{u}_1,\ldots,\tilde{u}_{2n}) = c, 
	\end{equation}
	where $(\tilde{u}_1,\ldots,\tilde{u}_{2n})=\tilde{u}\in\mathbb{R}^{2n}$ are coordinates for~$\tilde{U}$. 
	In other words, the neighborhood~$\tilde{U}\subset\tilde{\man}$ of $\tilde{x}_i$ is a level set of $\Phi$. Now, consider the set of equations 
	\begin{align}\label{convPrf:manDefEq}
		F_1(\tilde{u}) = & \Phi_1(\tilde{u})-c_1 = 0, \nonumber \\  &\vdots \nonumber \\ 	F_{2n-d}(\tilde{u}) = & \Phi_{2n-d}(\tilde{u})-c_{2n-d} = 0.
	\end{align}
	Since $\Phi$ is a diffeomorphism, its differential has full rank for all points in $\tilde{U}$. Hence, the matrix 
	\begin{equation}
		\begin{pmatrix}
			-&\nabla F_1 &-\\ & \vdots& \\ -&\nabla F_{2n-d}&-
		\end{pmatrix}=
	\begin{pmatrix}
		-&\nabla \Phi_1 &-\\ & \vdots& \\ -&\nabla\Phi_{2n-d}&-
	\end{pmatrix}
	\end{equation}
	has full rank for all $\tilde{u}\in \tilde{U}$. 
	
	Next, let $u=(\tilde{u}_1,\ldots,\tilde{u}_n)^T+i\cdot(\tilde{u}_{n+1},\ldots,\tilde{u}_{2n})^T$, and consider the set of equations 
	\begin{align}\label{convPrf:linConstraintsImpct}
		H_1(\tilde{u}) =&  \text{Re}\left\{\dprod{Y_1\cdot u}{x_i}\right\}=0 , \nonumber \\  \vdots& \nonumber \\  H_{d_G}(\tilde{u}) = & \text{Re}\left\{\dprod{Y_{d_G}\cdot u}{x_i}\right\} =0.
	\end{align}
	By a direct computation, we get that
	\begin{equation}\label{convPrf:Hmat}
		\begin{pmatrix}
			-&\nabla H_1&-\\
			&\vdots&\\
			-&\nabla H_{d_G}&-\\
		\end{pmatrix}=
	-\begin{pmatrix}
		-&\widetilde{Y_1\cdot x_i}&-\\
		&\vdots &\\
		-&\widetilde{Y_{d_G}\cdot x_i}&-
	\end{pmatrix},
	\end{equation}
	where $\widetilde{Y_k\cdot x_i}$ is the image of the map $\tilde{(\cdot)}$ applied to $Y_k\cdot x_i$. 
	 Now, we observe that by~\eqref{convPrf:linConstraintQuadForm}, we have that
	 \begin{equation}\label{tangVecGxi}
	 	\pDeriv{A(u)}{u_k}\cdot x_i = Y_k\cdot x_i,\quad k\in \left\{1,\ldots,d_G\right\}
	 \end{equation}
	 hence, the vectors $Y_1\cdot x_i, \ldots, Y_{d_G}\cdot x_i$ are the rows of the differential of the map~$\eqref{convPrf:expCoordinatesForG}$ at~$x_i$, which has full rank, since by assumption the map $\eqref{convPrf:expCoordinatesForG}$ is a diffemorphism. Thus, the vectors $Y_1\cdot x_i, \ldots, Y_{d_G}\cdot x_i$ are linearly independent. 	 
	 Since the map~$\tilde{(\cdot)}$ is an isometry, we infer that the vectors $\widetilde{Y_1\cdot x_i}, \ldots, \widetilde{Y_{d_G}\cdot x_i}$ are also linearly independent, whence we get that \eqref{convPrf:Hmat} has full rank. 

	Next, we observe that since~$A\mapsto Ax_i$ is a diffeomorphism onto~$G\cdot x_i$, by \eqref{tangVecGxi}, the vectors $Y_1\cdot x_i, \ldots, Y_{d_G}\cdot x_i$ reside in~$T_{x_i} (G\cdot x_i)\subset T_{x_i} \man$, the tangent space to~$G\cdot x_i$ at~$x_i$, and since $\tilde{(\cdot)}$ is a Riemannian isometry of $\man$ onto $\tilde{\man}$, we conclude that the vectors  $\widetilde{Y_1\cdot x_i}, \ldots, \widetilde{Y_{d_G}\cdot x_i}$  are tangent to~$\tilde{\man}$. On the other hand, the vectors $\nabla F_1,\ldots,\nabla F_{2n-d}$ are all perpendicular to the neighborhood $\tilde{U}$ of $\tilde{x}_i$, since it is defined as the level set~$F(\tilde{u})=0$, and therefore, they are perpendicular to all the vectors $\widetilde{Y_1\cdot x_i}, \ldots, \widetilde{Y_{d_G}\cdot x_i}$ tangent to $\tilde{\man}$ at~$\tilde{x}_i$. Hence, the $(2n-(d-d_G))\times (2n)$ matrix 
	\begin{equation}\label{convPrf:eqnUnion}
		\begin{pmatrix}
			-&\nabla F_1 &-\\ & \vdots& \\ -&\nabla F_{2n-d}&- \\ \\
				-&\nabla H_1&-\\
			&\vdots&\\
			-&\nabla H_{d_G}&-			
		\end{pmatrix}
	\end{equation}
	has full rank at~$\tilde{x}_i$. Thus, there exists a subset of $2n-(d-d_G)$ columns of \eqref{convPrf:eqnUnion} that form a $(2n-(d-d_G))\times (2n-(d-d_G))$ matrix $\tilde{D}_{\tilde{x_i}}$, which has a full rank. In particular, we have that~$\det\left(\tilde{D}_{\tilde{x_i}}\right)\neq 0$. 
	Lastly, by \eqref{convPrf:linConstraints}, the point $\tilde{x}_i$ is a solution of~\eqref{convPrf:linConstraintsImpct}, and by construction, also a solution of \eqref{convPrf:manDefEq}, and thus a solution of \eqref{convPrf:eqnUnion}.
	Hence, by the implicit function theorem, there exists an open subset $\tilde{V}\subset \tilde{U}$, and open subsets $\tilde{U}_1$ and $\tilde{U}_2$ such that $\tilde{V} = \tilde{U}_1\times \tilde{U}_2$, and coordinates $\tilde{u}_{i_1},\ldots, \tilde{u}_{i_{d-d_G}}\in \tilde{U}_1$, and smooth functions $g_1,\ldots,g_{2n-(d-d_G)}$ from $\tilde{U}_1$ onto $\tilde{U}_2$ such that 
	\begin{equation}\label{submanParametrization}
		\tilde{V} = \left\{(\tilde{u}_{i_1},\ldots, \tilde{u}_{i_{d-d_G}},g_1,\ldots,g_{2n-(d-d_G)})\; : \; \tilde{u}_{i_1},\ldots, \tilde{u}_{i_{d-d_G}}\in \tilde{U}_1 \right\}.
	\end{equation} 
	We can now redefine the set $\tilde{\mathcal{N}}$ in \eqref{convPrf:subManTagCharacter} as
	\begin{equation}\label{convPrf:subManTagCharacterMod}
		\tilde{\mathcal{N}} = \left\{ \tilde{z}\in \mathbb{R}^{2n}:	\text{Re}\left\{\dprod{Y_k\cdot z}{x_i}_{\mathbb{C}^{n}}\right\}=0, \quad k=1,\ldots,d_G, \quad z\in V \right\},
	\end{equation}
	where 
	\begin{equation}\label{Vset}
		V = \left\{x \; : \; \tilde{x}\in \tilde{V}\right\}.
	\end{equation}		
	By \eqref{submanParametrization}, we conclude that $\tilde{\subman}$ is a $(d-d_G)$-dimensional smooth submanifold in $\tilde{\man}$ of \eqref{manTilde}. 
	Now, we can redefine~$\subman$ in~\eqref{convPrf:subManCharacter} as 
	\begin{equation}
		\mathcal{N} = \left\{ z \; :\; \tilde{z}\in \tilde{\subman}\right\}.
	\end{equation}
	Moreover, we can take $\tilde{V}$ to be closed in $\mathbb{C}^N$ and small enough so that $V\subset \man'$, which guarantees that the problem~\eqref{convPrf: minimiaztion problem} has a unique solution for each~$x$ in~$\man'=G\cdot \subman$. Furthermore, since $\subman$ and $\man$ are isometric to~$\tilde{\subman}$ and~$\tilde{\man}$, respectively, we conclude that $\subman$ is a $(d-d_G)$-dimensional compact submanifold in $\man$. 
\end{proof}	

Next, we show how to integrate over $\mathcal{M}'$ using our $G$-invariant parametrization. 
Let $z(w)=z\left(w_1,\ldots, w_{d-d_G}\right)$ denote some coordinate chart on $\subman$ in \eqref{convPrf:subManCharacter}, and let~$A(u)= A(u_1,\ldots,u_{d_G})$ be the coordinate chart on $G$ in \eqref{convPrf: expMapForGnearI}.
The integral of a smooth function $h(x)$ over $\mathcal{M}'$ is given by the change of variables (see \cite{tuDiffGeom})
\begin{equation}\label{convPrf:changeOfVariables}
	\int_{\mathcal{M}'} h(x)dx = \int_{z\in\mathcal{N}}\int_{G}h(A\cdot z )dV(A\cdot z), 
\end{equation}
where we denote 
\begin{equation}
	V(A\cdot z) = \sqrt{\left| \det \left\{ g_{\mathcal{M}'}(A(u)\cdot z(w))\right\}\right|}
\end{equation}
and
\begin{equation}
	dV(A\cdot z) = \sqrt{\left| \det \left\{ g_{\mathcal{M}'}(A(u)\cdot z(w))\right\}\right|}dw_1\ldots dw_{d-d_G}du_1\ldots du_{d_G},
\end{equation}
is the volume form at $x=A\cdot z$, and $g_{\mathcal{M}'(x)}$ is the metric tensor on $\mathcal{M}'$ given by
\begin{equation}
	g_{\mathcal{M}'}(x) = \text{Re}\left\{ J^*_{\mathcal{M}'}(x)J_{\mathcal{M}'}(x)\right\},
\end{equation} 
and $J_{\mathcal{M}'}$ is the Jacobian change of variables matrix, given explicitly by 
\begin{equation}\label{convPrf: jacobMat}
	J_{\mathcal{M}'}(w,u) = \bigg( J_w \quad J_u\bigg), \quad J_w = \left( \pDeriv{x}{w_1} \cdots \pDeriv{x}{w_{d-d_G}}\right),\quad  J_u = \left(\pDeriv{x}{u_1} \cdots \pDeriv{x}{u_{d_G}}\right). 
\end{equation}

In the following section we prove Lemma \ref{convPrf:scndMomentApproxLemma}, which requires a careful asymptotic approximation of the second moment $\expect{\left(H_i\right)^2}$ in \eqref{convPrf:scndMomentApproxLemmaHi2} with respect to the uniform distribution over~$\man$. The proof employs the relationship between the Haar measure on $G$, and a certain measure induced by our $G$-invariant parametrization on orbits of the form $G\cdot z$, which we now define. 

First, we note that diffeomorphism $z\rightarrow A\cdot z$ admits an inverse map $\Phi:G\cdot z \rightarrow G$ given by 
\begin{equation}\label{convPrf:ChangeOfVar}
	\Phi(A\cdot z) = A, \quad A\cdot z\in G\cdot z, 
\end{equation} 
which induces a topology on $G\cdot z$ given by
\begin{equation}
	\mathcal{T}_{G\cdot z} = \left\{\Phi^{-1}(H)\;|\; H \text{ is a Borel measurable subset in }G\right\}.
\end{equation}
Next, consider the function $\mu_z$ over $\mathcal{T}_{G\cdot z}$  defined by 
\begin{equation}\label{convPrf:muDef}
	\mu_z(F) = \int_{u:A(u)\in\Phi(F)} 	d\mu_z(u),
\end{equation}
where
\begin{equation}
	d\mu_z(u) \coloneqq \sqrt{\left| \det(\text{Re}\{J_u^*\left(A(u)\cdot z\right)J_u\left(A(u)\cdot z\right)\}) \right|}du,
\end{equation}
and $J_u$ was defined in \eqref{convPrf: jacobMat}. The following lemma asserts that~$\mu_z$ is a measure over~$G\cdot z$, and characterizes its relationship to the Haar measure on~$G$. 
\begin{lemma}
	For every $z\in \subman$, the function $\mu_z$ is a measure over $G\cdot z$ with the topology $\mathcal{T}_{G\cdot z}$. Furthermore, define the pushforward of $\mu_z$ by the map $\Phi$ in \eqref{convPrf:ChangeOfVar}, as the function $\Phi_*(\mu_z)$ over the Borel $\sigma$-algebra of $G$ given by 
	\begin{equation}
		\Phi_*(\mu_z)(H) = \mu_z(\Phi^{-1}(H)), 
	\end{equation}
	for every Borel subset $H\subseteq G$. Then, with probability one we have that $\mu_z$ is a measure over $G\cdot z$. Furthermore, there exists a constant $\mu(z)>0$ such that 
	\begin{equation}\label{convPrf:changeOfMeasure}
		\Phi_*(\mu_z)(H) = \mu(z)\eta(H), \quad H\subseteq G, 
	\end{equation}
	where $\eta$ is the Haar measure over $G$. 
\end{lemma}
\begin{proof}
	To see that $\mu_z $ is a measure, first we note that since $\sqrt{\left| \det(\text{Re}\{J_u^*J_u\}) \right|}$ is non-negative then so is $\mu_z(\cdot)$. Furthermore, $\mu_z(\cdot)$ is bounded since for any $F\in \mathcal{T}_{G\cdot z}$ we have that
	\begin{equation}
		\mu_z(F) = \int_{u:A(u)\in\Phi(F)} 	d\mu_z(u)\leq \int_{u:A(u)\in G} 	d\mu_z(u)=\text{Vol}(G\cdot z)<\infty,
	\end{equation}
	where the last inequality is due to the fact that $G\cdot z$ is compact.  
	Thus, it only remains to show that $\mu_z(\cdot)$ is countably additive over $\mathcal{T}_{G\cdot z}$. Indeed, the map $\Phi$ being a homeomorphism preserves the topology of $G$ (see \cite{munkres}), and in particular, it holds that for any countable family of disjoint open sets $F_1,F_2,\ldots\in \mathcal{T}_{G\cdot z}$ we have 
	\begin{equation}\label{convPrf:disjointUnions}
		\Phi\left(\bigcup_{k=1}^\infty F_k\right) = \bigcup_{k=1}^\infty \Phi\left(F_k\right).
	\end{equation}
	In other words, the map $\Phi$ preserves disjoint unions. 
	Thus, by \eqref{convPrf:muDef} and \eqref{convPrf:disjointUnions} we have
	\begin{equation}
		\mu_z\left(\bigcup_{k=1}^\infty F_k\right) = \sum_{k=1}^{\infty}\mu_z(F_k).
	\end{equation}
	We conclude that~$\mu_z(\cdot)$ is a measure over~$G\cdot z$, the latter having the topology of~$G$ (induced by~$\Phi$).  
	
	Next, we show that~$\mu_z$ is left invariant under the action of~$G$, that is, for any measurable subset~$F\in \mathcal{T}_{G\cdot z}$, we have
	\begin{equation}\label{convPrf:homSpaceMeasureInvar}
		\mu_z(B\cdot F) = \mu_z(F), \quad B\in G. 
	\end{equation} 
	Indeed, by \eqref{convPrf: jacobMat} we have
	\begin{equation}
		J_u = \left[ J_{u_1}\;\cdots \; J_{u_{d_G}} \right] = \left[ \pDeriv{A}{u_1}\cdot z \;\cdots \;\pDeriv{A}{u_{d_G}} \cdot z \right],
	\end{equation}
	and thus
	\begin{equation}
		(J_u^*J_u)_{ij} = z^* \left(\pDeriv{A}{u_i}\right)^*\pDeriv{A}{u_j} z, \quad 1\leq i,j \leq d_G.
	\end{equation}
	Thus, for a fixed $B\in G$ we have
	\begin{equation}
		z^* \left(\pDeriv{\left(B\cdot A\right)}{u_i}\right)^*\pDeriv{\left(B\cdot A\right)}{u_j} z=z^* \left(\pDeriv{A}{u_i}\right)^*\pDeriv{A}{u_j} z.
	\end{equation}
	Now, the map $\Phi : G\cdot z \rightarrow G$
	induces a measure $\Phi_*(\mu_z)$ on $G$ via pushforward, defined explicitly by
	\begin{equation}\label{convPrf:pwMeasure}
		\Phi_*(\mu_z)(H) = \mu_z(\Phi^{-1}(H)),
	\end{equation}
	for every Borel measurable subset $H\subseteq G$. 
	Intuitively, the function $\Phi_*(\mu_z)$ measures the volume of a subset $H\subseteq G$ by first mapping~$H$ into the orbit $G\cdot z$, and then measuring the volume of the image $H\cdot z\subseteq G\cdot z$. 
	By \eqref{convPrf:homSpaceMeasureInvar}, for a fixed $A\in G$ and any open subset $H\subseteq G$ we have
	\begin{equation}
		\Phi_*(\mu_z)(AH)=\mu_z(\Phi^{-1}(AH))=\mu_z(AH\cdot z) = \mu_z (H\cdot z) = \Phi_*(\mu_z)(H),
	\end{equation}
	which shows that the measure $\Phi_*(\mu_z)$ is left-invariant. By Haar's theorem for compact groups there exists, up to multiplication by a positive scalar, a unique left invariant measure over $G$. It follows that for every $z\in \mathcal{N}$ there exists a positive scalar $\mu(z) \in \mathbb{R}$ such that
	\begin{equation}
		\mu(z)\eta(H) = \Phi_*(\mu_z)(H), \quad H\subseteq G,
	\end{equation} 
	which in turn implies that $\Phi_*(\mu_z)$ is related to the Haar measure by
	\begin{equation}\label{convPrf:changeOfMeasure}
		\eta(H) = \frac{\Phi_*(\mu_z)(H)}{\mu(z)}, \quad z\in\mathcal{N}.
	\end{equation}
	In particular, plugging $H=G$ into \eqref{convPrf:changeOfMeasure}, and using \eqref{secLieGroupAction:probMeasure} we get
	\begin{equation}
		\mu(z) = \frac{\Phi_*(\mu_z)(G)}{\eta(G)} =\mu_z(\Phi^{-1}(G))= \mu_z(G\cdot z),
	\end{equation}
	which shows that $\mu(z)$ is the volume of the orbit $G\cdot z$. By assumption, with probability one we have that $A\cdot z \neq z$ for all $A\in G$, and thus the map~$\Phi^{-1}(A) = A\cdot z$ is a diffeomorphism of $G$ onto $G\cdot z$. Hence, we conclude that $\mu_z(G\cdot z)>0$. 
\end{proof}

\subsection{Proof of Lemma \ref{convPrf:scndMomentApproxLemma}}\label{secConvBigLemmaPrf}
In this section, we evaluate the second order moment $\expect{\left(H_i\right)^2}$, which appears in the evaluation of \eqref{convPrf:JiExpression}. The evaluation of the remaining second order moments in $\eqref{convPrf:JiExpression}$ is done in a very similar fashion. 

Now, recall that in Section \ref{secGinvParam} we constructed a $G$-invariant parametrization of a certain neighborhood $\man'\subset\man$ of the data point $x_i$, such that $\norm{x-x_i}>\delta$ for all $x\notin \man'$. 
Thus, by \eqref{convPrf: groupConvolution} we have
\begin{equation}\label{convPrf:expNeglect}
	\expect{\left(H_i\right)^2} = \int_\man H_i^2(x)p(x)dx = \int_{\man'} H_i(x)p(x)dx+O\left(e^{-\delta^2/\epsilon}\right),
\end{equation}
where the second equality stems from the fact that $\man' = \man \cap B_{\delta}(x_i)$, and the integrand of $H_i(x)$ is a Gaussian of width $\epsilon$ centered at $x_i$. 
We point out that the exponentially small error term on the r.h.s of~\eqref{convPrf:expNeglect} is negligible with respect to the polynomial asymptotic error in \eqref{convPrf:scndMomentApproxLemmaHi2} that we are about to derive, and thus will be dropped in all subsequent analysis. 
Furthermore, for a fixed $x\in \man'$ there exist $z\in \subman$ and $B'\in G$ such that 
\begin{equation}
	H_i(x) = \int_G \expPow{x_i}{B\cdot B'z}f(B\cdot B'z)d\eta(B) = \int_G \expPow{x_i}{B\cdot z}f(B\cdot z)d\eta(B) = H_i(z),
\end{equation}
where in the second equality we used the change of variables $B=B\cdot B'$, and that the Haar measure $\eta$ on a compact group is right invariant. Therefore, continuing from \eqref{convPrf:expNeglect} and using \eqref{convPrf:changeOfVariables} we can write
\begin{align}\label{convPrf:secondMomentGInvariance}
	\expect{\left(H_i\right)^2} &= \int_\subman \int_G \left(H_i(B\cdot z)\right)^2 \frac{1}{\text{Vol}\left(\man\right)}V(B\cdot z)d\eta(B) dz  \\ \nonumber
	& = \int_\subman \left(H_i(z)\right)^2 p_\subman(z) dz, 
\end{align}
where we defined 
\begin{equation}
	p_\subman(z) = \frac{1}{\text{Vol}\left(\mathcal{M}\right)}\int_{G}V(B\cdot z)d\eta(B),
\end{equation}
where we used that $\text{Vol}(G)=1$.

Next, writing 
\begin{equation*}
	\left\lVert x_i - B\cdot z\right \rVert^2 = \left\lVert x_i - z+ z- B\cdot z\right \rVert^2= \norm{x_i-z}^2 + 2\text{Re} \left\{\dprod{ x_i-z}{z-B\cdot z}\right\} + \norm{z-B\cdot z}^2,	
\end{equation*}
and defining 
\begin{equation}\label{confPrf:DeltaDef}
	\delta_i(z,x) \coloneqq  2\text{Re} \left\{\dprod{ x_i-z}{z-x}\right\}, 
\end{equation}
we have
\begin{equation*}
	H_i(z) = e^{-\norm{x_i-z}^2/\epsilon}\int_G e^{-\norm{z-B\cdot z}^2/\epsilon} e^{-\delta_i(z,B\cdot z)/\epsilon} f(B\cdot z)d\eta(B). 
\end{equation*}
Taylor expanding the function $e^{-x}$ at $\delta_i(z,B\cdot z)/\epsilon$ we get,
\begin{equation*}
	e^{-\delta_i(z,B\cdot z)/\epsilon} = 1+O\left(\frac{\delta_i(z,B\cdot z)}{\epsilon}\right)
\end{equation*}
from which we have
\begin{align}\label{convPrf: HiPythagorDecomp}
	H_i(z) = e^{-\norm{x_i-z}^2/\epsilon}  \bigg[&\int_{G} e^{-\norm{z - B\cdot z}^2/\epsilon} f(B\cdot z) d\eta(B) \\ \nonumber
	&+O\left(\frac{1}{\epsilon} \int_{G} e^{-\norm{z - B\cdot z}^2/\epsilon} \delta_i(z,B\cdot z)f(B\cdot z)d\eta(B)\right)\bigg]\nonumber.
\end{align}
We now proceed to evaluate \eqref{convPrf: HiPythagorDecomp} term by term. 

For the first term, we have
\begin{align}\label{convPrf:firstTermEval}
	\int_{G} e^{-\norm{z - B\cdot z}^2/\epsilon} f(B\cdot z) d\eta(B) &= \frac{1}{\mu(z)}\int_{G} e^{-\norm{z-B\cdot z}^2/\epsilon}f(B\cdot z)d\Phi_*(\mu_z)(B)\nonumber\\ 
	&= \frac{1}{\mu(z)}\int_{G} e^{-\norm{z-\Phi^{-1}(B)}^2/\epsilon}f\left(\Phi^{-1}(B)\right)d\Phi_*(\mu_z)(B)\nonumber\\ 
	&= \frac{1}{\mu(z)}\int_{G\cdot z} e^{-\norm{z-x}^2/\epsilon}f(x)d\mu_z(x) \nonumber\\
	&= \frac{\left(\pi \epsilon\right)^{d_G/2}}{\mu(z)}(f(z)+O(\epsilon)),
\end{align}
where we used \eqref{convPrf:pwMeasure} and \eqref{convPrf:changeOfMeasure} in the first equality, the definition of the map $\Phi$ in \eqref{convPrf:ChangeOfVar} in the second equality, the change of variables theorem for pushforward measures (see Theorem 3.6.1. in \cite{Bogachev}) for $x=\Phi^{-1}(B)$ in the third one, and Theorem~\ref{convPrf: ClassicalConvRateThrm} in the last equality, where we note that $d\mu_z(x)$ is the Riemannian volume element of the $d_G$-dimensional manifold~$G\cdot z$. 

For the second term in \eqref{convPrf: HiPythagorDecomp}, using the same change of variables as in \eqref{convPrf:firstTermEval} we have 
\begin{align}
	&\frac{1}{\epsilon}\int_{G} e^{-\norm{z-B\cdot z}^2/\epsilon}\delta_i(z,B\cdot z)f(B\cdot z)d\eta(B) = \frac{1}{\epsilon\mu(z)}\int_{G\cdot z} e^{-\norm{z-x}^2}\delta_i(z,x)f(x)d\mu_z(x) \\ \nonumber
	&= \frac{\left(\pi \epsilon\right)^{d_G/2}}{\epsilon\mu(z)} \bigg[\delta_i(z,z)f(z)+\frac{\epsilon}{4}\bigg[ E(z)\delta_i(z,z)+ \Delta_{G\cdot z}\biggl\{\delta_i(z,x)f(x)\biggr\}\bigg|_{x=z} \bigg]+O(\epsilon^2)\bigg]\\ \nonumber
	&=  \frac{\left(\pi \epsilon\right)^{d_G/2}}{\epsilon\mu(z)} \bigg[\frac{\epsilon}{4} \Delta_{G\cdot z}\biggl\{\delta_i(z,x)f(x)\biggr\}\bigg|_{x=z}+O(\epsilon^2)\bigg],
\end{align}
where we used that $\delta_i(z,z)=0$ by \eqref{confPrf:DeltaDef}.  
Furthermore, we have
\begin{align}
	\Delta_{G\cdot z}&\biggl\{\delta_i(z,x)f(x)\biggr\}\big|_{x=z}= \\&= \Delta_{G\cdot z}f(x)\big|_{x=z}\cdot \delta_i(z,z)-2\dprod{\nabla_{G\cdot z}\delta_i(z,x)\big|_{x=z}}{\nabla_{G\cdot z}f(z)}+f(z)\cdot \Delta_{G\cdot z} \delta_i(z,x)\big|_{x=z} \\ \nonumber
	&= f(z)\Delta_{G\cdot z}\delta_i(z,x)\big|_{x=z},
\end{align}
where we have used the multivariate version of the formula for the second derivative of a product of functions (see Lemma 3.3 in \cite{rosenberg}), and that by \eqref{confPrf:DeltaDef}
\begin{align}\label{convPrf:gradDeltai}
	\nabla_{G\cdot z} \delta_i (z,x)\big|_{x=z} = -2\text{Re}\left\{\begin{pmatrix}
		|&\cdots & |\\
		\nabla_{G\cdot z}x_1 & \cdots &\nabla_{G\cdot z}x_n\\
		|&\dots & |
	\end{pmatrix}^* \cdot(x_i-z)\right\}
	=0,
\end{align}
where we used the fact that $x$ in \eqref{convPrf:gradDeltai} is a coordinate function on~$G\cdot z$, that is, the function $x(p)$ returns the coordinates in $\mathbb{C}^n$ of $p\in G\cdot z$, and thus the vectors~$\nabla_{G\cdot z} x_k|_{x=z}$ reside in the tangent space $T_{z} \left\{G\cdot z\right\}$ to $G\cdot z$ at $x=z$, and that, by construction, the vector $x_i-z$ is perpendicular to $T_{z} \left\{G\cdot z\right\}$ (see~\eqref{convPrf:optDeriv}). 
Continuing, we now define the function 
\begin{equation}
	q(z) \coloneqq \frac{f(z)}{4}\Delta_{G\cdot z}\delta_i(z,x)|_{x=z} = -\frac{f(z)}{2}\text{Re}\big\{ \dprod{x_i-z}{\Delta_{G\cdot z}x\big|_{x=z}} \big\},
\end{equation}
for all $z\in\mathcal{N}$, where the expression $\Delta_{G\cdot z}x\big|_{x=z}$ is the vector valued function whose entries are the Laplacians of each coordinate function of the parametrization of $G\cdot z$ via $x(B) = B\cdot z$, evaluated at $z$. 
By the Cauchy-Schwart inequality combined with the compactness of $G\cdot z$ we get
\begin{equation}\label{convPrf:qzProp}
	q(x_i) = 0, \quad q(z) = O(\norm{x_i-z}),
\end{equation}
which leads to 
\begin{equation}\label{convPrf:secondTermEval}
	\frac{1}{\epsilon}\int_{G} e^{-\norm{z-B\cdot z}^2/\epsilon}\delta_i(z,B\cdot z)f(B\cdot z)d\eta(B) = \frac{(\pi \epsilon)^{d_G/2}}{\epsilon \mu(z)}[\epsilon q(z) + O(\epsilon^2)] = \frac{(\pi \epsilon)^{d_G/2}}{\mu(z)}[q(z) + O(\epsilon)].
\end{equation}
Now, substituting \eqref{convPrf:firstTermEval} and \eqref{convPrf:secondTermEval} into \eqref{convPrf: HiPythagorDecomp} we have that
\begin{equation}
	H_i(z) = \frac{e^{-\norm{x_i-z}^2/\epsilon}}{\mu(z)}\left(\pi \epsilon\right)^{d_G/2}\bigg[f(z)+O(q(z))+O\left(\epsilon\right)\bigg].
\end{equation}
Thus, we get
\begin{equation}\label{HsqrdEval}
	\left(H_i(z)\right)^2 = \frac{e^{-2\norm{x_i-z}^2/\epsilon}}{\mu^2(z)}\left(\pi\epsilon\right)^{d_G}\big[f^2(z)+O(2f(z)q(z))+O\left(\norm{x_i-z}^2\right)+O\left(\epsilon\right)\big],
\end{equation}
after suppressing higher  order terms. Plugging \eqref{HsqrdEval} into \eqref{convPrf:secondMomentGInvariance}, we get
\begin{align}\label{convPrf:finalEval}
	&\expect{\left(H_i(x)\right)^2} = \int_{\subman}\left(H_i(z)\right)^2p_\subman(z)dx = \\ \nonumber
	&\left(\pi \epsilon\right)^{d_G}\int_\subman \frac{e^{-2\norm{x_i-z}^2/\epsilon}}{\mu^2(z)}\big[f^2(z)+O(2f(z)q(z))+O\left(\norm{x_i-z}^2\right)+O\left(\epsilon\right)\big]p_\subman(z) dz.
\end{align}
Applying Theorem \ref{convPrf: ClassicalConvRateThrm} to each term inside the integral \eqref{convPrf:finalEval}, we have that
\begin{equation}\label{convPrf:finalTerm1}
	\int_\subman e^{-2\norm{x_i-z}^2/\epsilon} \frac{f^2(z)p_\subman(z)}{\mu^2(z)}dz = \left(\pi\epsilon/2\right)^{(d-d_G)/2}\bigg[ \frac{f^2(x_i)p_\subman(x_i)}{\mu^2(x_i)}+O(\epsilon)\bigg],
\end{equation}
and that
\begin{align}\label{convPrf:finalTerm2}
	\int_\subman e^{-2\norm{x_i-z}^2/\epsilon} \frac{f(z)q(z)p_\subman}{\mu^2(z)}dz &= \left(\pi\epsilon/2\right)^{(d-d_G)/2}\bigg[ \frac{f(x_i)q(x_i)p_\subman(x_i)}{\mu^2(x_i)}+O(\epsilon)\bigg]\\ \nonumber 
	 &= (\pi\epsilon/2)^{\left(d-d_G\right)/2}\cdot O(\epsilon),
\end{align}
where we used that by \eqref{convPrf:qzProp} we have that $q(x_i) = 0$, and using that $\norm{x_i-z}^2$ vanishes at $z=x_i$ we get that
\begin{equation}\label{convPrf:finalTerm3}
	\int_\subman \frac{e^{-2\norm{x_i-z}^2/\epsilon}}{\mu^2(z)} O(\norm{x_i-z}^2)p_\subman(z)dz = \left(\pi\epsilon/2\right)^{(d-d_G)/2}\cdot O(\epsilon).
\end{equation} 
Finally, plugging \eqref{convPrf:finalTerm1}, \eqref{convPrf:finalTerm2} and \eqref{convPrf:finalTerm3} into \eqref{convPrf:finalEval}, we get that
\begin{align}
		\expect{\left(H_i(x)\right)^2} &= \frac{(\pi\epsilon)^{\left(d+d_G\right)/2}}{2^{(d-d_G)/2}}\bigg[ \frac{f^2(x_i)p_\subman(x_i)}{\mu^2(x_i)}+O(\epsilon)\bigg],
\end{align}
which finishes the proof of \eqref{convPrf:scndMomentApproxLemmaHi2} in Lemma \ref{convPrf:scndMomentApproxLemma}.

\section{Non-uniform sampling distribution}\label{secNonUniformDistProof}
First, let us compute the limiting operator resulting from assuming a non-uniform sampling distribution~$p(x)$ in the setting of Theorem~\ref{sec2:ConvThrmUnnormalized}, by repeating the analysis of the bias error at the beginning of~\ref{secConvPrf} under this assumption. Fixing an $\epsilon>0$, we compute the limit of~\eqref{eq:steerable graph laplacian fraction} as~$N\rightarrow\infty$. By~\eqref{convPrf:C1}-\eqref{convPrf:numerLim}, we have that the limit of~$C_{i,N}^1$ in~\eqref{convPrf:C1} evaluates as
\begin{align}
	\lim_{N\rightarrow\infty} C_{i,N}^1&=\int_{\man} H_i(x) p(x)d\omega(x)\nonumber \\
	 &= \int_{\mathcal{M}}\int_G \exp{\left\lbrace-{\left\Vert x_i - B\cdot x \right\Vert^2}{/\varepsilon}\right\rbrace}f(B\cdot x) p(x)d\omega(x)d\eta(A)
 \end{align}
Making the change of variables~$y=A\cdot x$, and using~\eqref{convPrf:U_ADef}-\eqref{convPrf:pushforwardInvarianceProp} we have that
\begin{align}\label{nonUniformPrf:C1}
	\lim_{N\rightarrow\infty} C_{i,N}^1&=
	\int_G\int_{\mathcal{M}} \exp{\left\lbrace-{\left\Vert x_i - y \right\Vert^2}{/\varepsilon}\right\rbrace}f(y)p(A^*\cdot y) dU_A^*(\omega)(y)d\eta(A) \nonumber \\
	&=\int_G\int_{\mathcal{M}} \exp{\left\lbrace-{\left\Vert x_i - y \right\Vert^2}{/\varepsilon}\right\rbrace}f(y)p(A^*\cdot y) d\omega(y)d\eta(A) \nonumber \\
	&=\int_{\mathcal{M}}\int_G \exp{\left\lbrace-{\left\Vert x_i - y \right\Vert^2}{/\varepsilon}\right\rbrace}f(y)p(A^*\cdot y) d\eta(A)d\omega(y) \nonumber \\
	&=	\int_{\mathcal{M}} \exp{\left\lbrace-{\left\Vert x_i - y \right\Vert^2}{/\varepsilon}\right\rbrace}f(y)\tilde{p}(y) d\omega(y), 
\end{align}
where we defined
\begin{equation}
	\tilde{p}(x) = \int_G p(A^*\cdot y)d\eta(A). 
\end{equation}
Similarly, we get that the limit of~$C_{i,N}^2$ in~\eqref{convPrf:C2} as~$N\rightarrow\infty$ is given by
\begin{align}\label{nonUniformPrf:C2}
		\lim_{N\rightarrow\infty}C_{i,N}^2 &= 	\lim_{N\rightarrow\infty} \frac{1}{N}\sum_{j=1}^N \int_G \exp{\left\lbrace-{\left\Vert x_i - B\cdot x_j\right\Vert^2}{/\varepsilon}\right\rbrace} \nonumber \\&=\int_{\mathcal{M}} \exp{\left\lbrace-{\left\Vert x_i - y \right\Vert^2}{/\varepsilon}\right\rbrace}\tilde{p}(y) d\omega(y) d\eta(B).
\end{align}
By~\eqref{nonUniformPrf:C1} and~\eqref{nonUniformPrf:C2}, we obtain that the  the limit of~\eqref{eq:steerable graph laplacian fraction} when~$N\rightarrow\infty$ is given by
\begin{align}
	\lim_{N\rightarrow\infty}\frac{4}{\varepsilon} \left\{\tilde{L}g\right\}(i,I) &=  \lim_{N\rightarrow\infty}\frac{4}{\varepsilon}\left[f(x_i) -  \frac{C_{i,N}^1}{C_{i,N}^2}\right] \nonumber \\
	&= \frac{4}{\varepsilon}\left[f(x_i) -  \frac{\int_{\mathcal{M}} \exp{\left\lbrace-{\left\Vert x_i - y \right\Vert^2}{/\varepsilon}\right\rbrace}f(y)\tilde{p}(y) d\omega(y)}{\int_{\mathcal{M}} \exp{\left\lbrace-{\left\Vert x_i - y \right\Vert^2}{/\varepsilon}\right\rbrace}\tilde{p}(y) d\omega(y) }\right]
\end{align}
By the results in~\cite{diffMaps}, we get that 
\begin{align}
	\lim_{\epsilon\rightarrow 0}\lim_{N\rightarrow\infty}\frac{4}{\varepsilon} \left\{\tilde{L}g\right\}(i,I) = \Delta_\man f(x_i)-2\frac{\dprod{\nabla_\man f(x_i)}{\nabla_\man \tilde{p}(x_i)}}{\tilde{p}(x_i)}.
\end{align}
This shows that in case that~$p(x)$ in non-uniform, the normalized~$G$-GL converges to an operator that is different from the Laplace-Beltrami operator~$\Delta_\man$, namely, the Fokker-Planck operator on~$\man$ which depends on the density~$\tilde{p}(x)$. 

Nevertheless, following \cite{steerMaps} and~\cite{diffMaps}, we now show that we can retrieve~$\Delta_\man$ by normalizing the kernel function~$W_{ij}(A,B)$ in~\eqref{GinvSec:affinityKernelDef}, as follows. Let us define for all $i,j\in \left\{1,\ldots,N\right\}$
\begin{equation}
	\bar{W}_{ij}(A,B) \coloneq \frac{W_{ij}(A,B)}{D_{ii}D_{jj}},
\end{equation}
and for all $i\in \left\{1,\ldots,N\right\}$ 
\begin{equation}
	\bar{D}_{ii} \coloneq \sum_{i=1}^N \int_G \bar{W}_{ij}(I,A)d\eta (A). 
\end{equation}
Then, we define the density-normalized $G$-invariant graph Laplacian~$\bar{L}$ as
\begin{equation}
	\bar{L} = I - \bar{D}^{-1}\bar{W}. 
\end{equation}
By repeating the computations in equations~$(190)-(191)$ in~\cite{steerMaps}, we obtain that 
\begin{align}
	\lim_{N\rightarrow\infty}\frac{4}{\varepsilon} \left\{\tilde{L}g\right\}(i,I) 
	= \frac{4}{\varepsilon}\left[f(x_i) -  \frac{\int_{\mathcal{M}} \exp{\left\lbrace-{\left\Vert x_i - y \right\Vert^2}{/\varepsilon}\right\rbrace}f(y)\hat{p}(y) d\omega(y)}{\int_{\mathcal{M}} \exp{\left\lbrace-{\left\Vert x_i - y \right\Vert^2}{/\varepsilon}\right\rbrace}\hat{p}(y) d\omega(y) }\right],
\end{align}
where
\begin{equation}
	\hat{p}(x) = \frac{\tilde{p}(x)}{\int_{\mathcal{M}} \exp{\left\lbrace-{\left\Vert x - y \right\Vert^2}{/\varepsilon}\right\rbrace}\tilde{p}(y) d\omega(y)}. 
\end{equation}
Then, by a derivation in~\cite{diffMaps} we obtain that
\begin{equation}
	\lim_{\epsilon\rightarrow \epsilon} \lim_{N\rightarrow \infty} \frac{4}{\epsilon}\left\{\bar{L}g\right\}(i,I) = \Delta_\man f(x_i).
\end{equation}

\section{Eigendecomposition of the normalized $G$-GL }\label{secEigDecompNGGL}	
We now restate Theorem \ref{GInvLapProp:Thrm1} for the operator $\tilde{L}$ in \eqref{GinvDef:normGLapDef}, the normalized version of the $G$-GL. The proof is obtained by repeating that of Theorem \ref{GInvLapProp:Thrm1}, with the matrices~$D^{(\ell)}-\hat{W}^{(\ell)}$ replaced by the matrix sequence 
\begin{equation}
	S^{(\ell)} = I-(D^{\ell})^{-1}\hat{W}^{\ell}, \quad \ell \in \I,
\end{equation}
of \eqref{GinvProp:normFourierMat}, with required changes made in equations \eqref{eigdecompProof:evProp1}-\eqref{eigdecompProof:evProp2}, and omitting the proof of orthogonality which doesn't hold in this case. 

\begin{theorem}\label{GInvLapProp:Thrm1Norm}
	For each $\ell\in \mathbb{N}$, let $D^{\ell}$ be the $Nd_\ell\times Nd_\ell$ block-diagonal matrix who's $i$-th block of size $d_\ell\times d_\ell$ on the diagonal is given by the product of the scalar $D_{ii}$ in \eqref{GinvDef:Ddef} with the $d_\ell\times d_\ell$ identity matrix. 
	Then, the normalized $G$-invariant graph Laplacian admits the following:
	\begin{enumerate}
		\item A sequence of non-negative eigenvalues $\{\tilde{\lambda}_{1,\ell},\ldots,\tilde{\lambda}_{Nd_\ell,\ell}\}_{\ell\in \I}$, where $\tilde{\lambda}_{n,\ell}$ is the $n$-th eigenvalue of the matrix $S^{(\ell)} = I-D^{-1}\hat{W}^{\ell}$. 
		\item  A sequence $\{\tilde{\Phi}_{\ell,1,1},\ldots,\tilde{\Phi}_{\ell,d_\ell,Nd_\ell}\}_{\ell\in \I}$ of eigenfunctions, which are complete in $\mathcal{H}$ and are given by 
		\begin{equation}\label{GInvLapProp:EigenFuncFormNorm}
			\tilde{\Phi}_{\ell,m,n}(\cdot,A) = 
			\begin{pmatrix}
				-& e^1(\tilde{v}_{n,\ell})&-\\
				& \vdots &\\
				-& e^N(\tilde{v}_{n,\ell})&-
			\end{pmatrix}\cdot
			U^{(\ell)}_{\cdot,m}(A^*),
		\end{equation}
		where $\tilde{v}_{n,\ell}$ is the eigenvector of $S^{(\ell)} = I-D^{-1}\hat{W}^{\ell}$ which corresponds to its eigenvalue~$\tilde{\lambda}_{n,\ell}$. For each $n\in \{1,\ldots, Nd_\ell\}$ and $\ell\in\I_G$, the eigenvectors $\{\tilde{\Phi}_{\ell,-\ell,n},\ldots,\tilde{\Phi}_{\ell,\ell,n}\}$ correspond to the eigenvalue $\tilde{\lambda}_{n,\ell}$ of the normalized $G$-invariant graph Laplacian. 
	\end{enumerate}
\end{theorem}

\bibliographystyle{plain}
\bibliography{GGLmanuscript}

\end{document}